\documentclass[twoside,11pt]{article}

\usepackage{blindtext}

\usepackage{amsmath}
\usepackage{amssymb}
\usepackage{amsthm}  

\usepackage{thmtools}  

\usepackage[colorlinks=false,allbordercolors={1 1 1}]{hyperref}  
\usepackage[capitalise,nameinlink]{cleveref}   

\hypersetup{hidelinks}  

\usepackage[utf8]{inputenc} 
\usepackage[T1]{fontenc}    
\usepackage{url}            
\usepackage{booktabs}       
\usepackage{amsfonts}       
\usepackage{nicefrac}       
\usepackage{microtype}      
\usepackage{lipsum}
\usepackage{xcolor}
\usepackage{siunitx}
\usepackage{comment}
\usepackage{algpseudocode}
\usepackage{algorithm}

\newcommand{\Eqref}[1]{Equation~\eqref{#1}}

\usepackage[utf8]{inputenc}
\usepackage[T1]{fontenc}
\usepackage[USenglish]{babel}

\usepackage{amsmath}
\usepackage{amssymb}
\usepackage{amsthm}
\usepackage{mathtools}

\usepackage{comment}
\usepackage[shortlabels]{enumitem}
\usepackage{csquotes}
\usepackage{tikz}
\usepackage{tikz-qtree}
\usetikzlibrary{trees,positioning,shapes}
\usetikzlibrary{arrows.meta}
\usepackage{pgfplots}
\usepackage[tableposition=above]{caption}
\usepackage{subcaption}
\usepackage{longtable}

\usepackage{hyperref}
\usepackage[capitalise,nameinlink]{cleveref}

\usepackage{chngcntr}

\usepackage[nohyperref]{jmlr2e}

\usepackage{dsfont}  

\usepackage{pifont}

\usepackage[retainorgcmds]{IEEEtrantools}

\pgfplotsset{compat=1.11}

\newlength\Origarrayrulewidth

\newenvironment{appendixenv}{%
\clearpage

\begin{appendices}

\listofappendices

\newpage

\numberwithin{theorem}{section}
\numberwithin{lemma}{section}
\numberwithin{proposition}{section}
\numberwithin{corollary}{section}
\numberwithin{definition}{section}
\numberwithin{remark}{section}
\numberwithin{example}{section}
\counterwithin{figure}{section}
\counterwithin{table}{section}
\counterwithin{algorithm}{section}}{\end{appendices}}

\providecommand{\BlackBox}{\rule{1.5ex}{1.5ex}}  

\renewenvironment{proof}[1][\proofname]{\par
  \pushQED{\qed}%
  \normalfont\noindent{\bf #1\ }
}{%
  \popQED
}
\providecommand{\proofname}{Proof}

\newcommand{\ifnotinthesis}[1]{#1}

\newcommand{\bbC}{\mathbb{C}}

\newcommand{\bbE}{\mathbb{E}}

\newcommand{\bbN}{\mathbb{N}}

\newcommand{\bbR}{\mathbb{R}}

\newcommand{\bbone}{\mathds{1}}

\newcommand{\calA}{\mathcal{A}}

\newcommand{\calC}{\mathcal{C}}

\newcommand{\calF}{\mathcal{F}}
\newcommand{\calG}{\mathcal{G}}

\newcommand{\calK}{\mathcal{K}}

\newcommand{\calM}{\mathcal{M}}

\newcommand{\calP}{\mathcal{P}}
\newcommand{\calQ}{\mathcal{Q}}

\newcommand{\calU}{\mathcal{U}}
\newcommand{\calV}{\mathcal{V}}

\newcommand{\calX}{\mathcal{X}}

\newcommand{\calZ}{\mathcal{Z}}

\providecommand{\eps}{\varepsilon}

\newcommand{\diff}{\,\mathrm{d}}
\newcommand{\quot}[1]{\enquote{#1}}

\newcommand{\equalDef}{\coloneqq}
\newcommand{\defEqual}{\eqqcolon}

\DeclareMathOperator{\tr}{tr}

\DeclareMathOperator{\Span}{Span}
\DeclareMathOperator{\Var}{Var}

\DeclareMathOperator{\diam}{diam}

\DeclareMathOperator{\esssup}{ess\,sup}
\DeclareMathOperator{\conv}{conv}

\DeclareMathOperator*{\argmax}{argmax}

\DeclarePairedDelimiterX{\kldivx}[2]{(}{)}{%
  #1\;\delimsize\|\;#2%
}

\newcommand{\DKL}{D_{\mathrm{KL}}\kldivx}
\newcommand{\Dsuplog}{D_{\operatorname{sup-log}}}
\newcommand{\DTV}{D_{\mathrm{TV}}}
\newcommand{\DW}{W_1}
\newcommand{\Denergy}{D_{\mathrm{energy}}}

\newcommand{\Lopt}{L^{\mathrm{OPT}}}
\newcommand{\Dopt}{D_{\mathrm{KL}}^{\mathrm{OPT}}\kldivx}
\newcommand{\Doptsingle}{D_{\mathrm{KL}}^{\mathrm{OPT}}}

\newcommand{\sC}{{}^*C}
\newcommand{\ssad}{{}^*\sigma^{\mathrm{ad}}}
\newcommand{\Aad}{\calA^{\mathrm{ad}}}
\newcommand{\Aadstoch}{\calA^{\operatorname{ad-stoch}}}
\newcommand{\ead}{e^{\mathrm{ad}}}
\newcommand{\eadstoch}{e^{\operatorname{ad-stoch}}}

\newcommand{\Ssamp}{S_{\mathrm{samp}}}
\newcommand{\Sapp}{S_{\mathrm{app}}}
\newcommand{\Sopts}{S_{\mathrm{opt}^*}}
\newcommand{\Sint}{S_{\mathrm{int}}}

\newcommand{\Dinfty}{D_\infty}
\newcommand{\Dabs}{D_{\mathrm{abs}}}

\newcommand{\Vlin}{\calV_{\mathrm{lin}}}

\newcommand{\ovl}{\overline}

\allowdisplaybreaks

\newcolumntype{C}[1]{>{\centering\arraybackslash}p{#1}}

\newcommand{\assign}{\leftarrow}

\setlist[enumerate]{nosep}
\setlist[itemize]{nosep}

\usepackage[page]{appendix}

\renewcommand{\appendixtocname}{Appendix Contents.}

\makeatletter
\let\oldappendix\appendices

\renewcommand{\appendices}{%
  \clearpage
  \renewcommand{\thesection}{\Roman{section}}
  \let\tf@toc\tf@app
  \addtocontents{app}{\protect\setcounter{tocdepth}{2}}
  \immediate\write\@auxout{%
    \string\let\string\tf@toc\string\tf@app^^J
  }
  \oldappendix
}%

\newcommand{\listofappendices}{%
  \begingroup
  \renewcommand{\contentsname}{\appendixtocname}
  \let\@oldstarttoc\@starttoc
  \def\@starttoc##1{\@oldstarttoc{app}}
  \tableofcontents
  \endgroup
}

\makeatother

\makeatletter
\newcommand\ackname{Acknowledgements}
\if@titlepage
   \newenvironment{acknowledgements}{%
       \titlepage
       \null\vfil
       \@beginparpenalty\@lowpenalty
       \begin{center}%
         \bfseries \ackname
         \@endparpenalty\@M
       \end{center}}%
      {\par\vfil\null\endtitlepage}
\else
   
\fi
\makeatother

\usepackage{lastpage}
\jmlrheading{26}{2025}{1-\pageref{LastPage}}{9/24; Revised 4/25}{11/25}{24-1494}{David Holzmüller and Francis Bach}

\ShortHeadings{Convergence Rates for Non-Log-Concave Sampling}{Holzmüller and Bach}
\firstpageno{1}

\begin{document}

\title{Convergence Rates for Non-Log-Concave Sampling and Log-Partition Estimation}

\author{\name David Holzmüller\thanks{Work done partially while at the University of Stuttgart.} \email david dot holzmuller at inria.fr \\ 
       \name Francis Bach \email francis.bach@inria.fr \\
       \addr INRIA \\
  Ecole Normale Supérieure \\
  PSL Research University}

\editor{Jianfeng Lu}

\maketitle

\begin{abstract}%
Sampling from Gibbs distributions and computing their log-partition function are fundamental tasks in statistics, machine learning, and statistical physics. While efficient algorithms are known for log-concave densities, the worst-case non-log-concave setting necessarily suffers from the curse of dimensionality. 
For many numerical problems, the curse of dimensionality can be alleviated when the target function is smooth, allowing the exponent in the rate to improve linearly with the number of available derivatives.
Recently, it has been shown that similarly fast convergence rates can be achieved by efficient optimization algorithms. Since optimization can be seen as the low-temperature limit of sampling from Gibbs distributions, we pose the question of whether similarly fast convergence rates can be achieved for non-log-concave sampling.
We first study the information-based complexity of the sampling and log-partition estimation problems and show that the optimal rates for sampling and log-partition computation are sometimes equal and sometimes faster than for optimization.
We then analyze various polynomial-time sampling algorithms, including an extension of a recent promising optimization approach, and find that they sometimes exhibit interesting behavior but no near-optimal rates. Our results also give further insights into the relation between sampling, log-partition, and optimization problems.
\end{abstract}


\section{Introduction} \label{sec:sampling:introduction}

The tasks of sampling from a Gibbs distribution with density $p(x) \propto \exp(-V(x)/\eps)$ and computing the corresponding normalization constant are important problems in many computational fields such as machine learning, (Bayesian) statistics, and statistical physics. 
Specifically, we are interested in the following setting:

\begin{definition}[Sampling and log-partition problems] \label{def:problem}
In this paper, we consider distributions on the unit cube $\calX \equalDef [0, 1]^d$, $d \geq 1$. While many of our results could be generalized to other domains, the unit cube has convenient properties that allow us to prove all results on the same domain: It is compact, has unit volume, does not have too sharp corners, there are well-studied approximation results, and it allows investigating algorithms for periodic functions.

Let $f: \calX \to \bbR$ be bounded and measurable.
The \emph{sampling problem} is to draw samples from the distribution $P_f$ on $\calX$ with density
\begin{IEEEeqnarray*}{+rCl+x*}
p_f(x) \equalDef \frac{\exp(f(x))}{Z_f}~,
\end{IEEEeqnarray*}
where $Z_f \equalDef \int_{\calX} \exp(f(x)) \diff x$ is the normalization constant or \emph{partition function}. The \emph{log-partition problem} is to compute the \emph{log-partition function}
\begin{IEEEeqnarray*}{+rCl+x*}
L_f & \equalDef & \log Z_f = \log\left(\int_{\calX} \exp(f(x)) \diff x\right)~. %
\end{IEEEeqnarray*}
\end{definition}

Distributions of the form $P_f$ are known as \emph{Gibbs distributions}, \emph{Gibbs measures}, or \emph{Boltzmann distributions}. They arise, for example, in statistical physics with $f(x) = -V(x)/\eps$, where $V(x)$ denotes the (potential) energy of state $x$, and $\eps$ is (proportional to) the temperature of the system. (Technically, $\eps = k_B T$, where $k_B$ is Boltzmann's constant and $T$ is the temperature.) Instead of the temperature, sometimes the inverse temperature (or coldness / thermodynamic beta) $\beta = 1/\eps$ is used. The log-partition problem is also related to the computation of the free energy $-\eps L_{-V/\eps}$. In energy-based ML models, $f$ could be learned. In a Bayesian statistical or ML model with parameters $\theta$ and data $D$, we could set $f(\theta) \equalDef \log p(D|\theta) + \log p(\theta)$ to sample from the posterior distribution $p_f(\theta) = p(\theta | D)$ or compute the log-evidence $L_f = \log p(D) = \log \left(\int_{\calX} p(D|\theta) p(\theta) \diff \theta\right)$. In some contexts, $L_f$ is also called the log-marginal likelihood, which is useful for model selection \citep{robert_bayesian_2007}.

While exact sampling and log-partition computation are possible for some simple functions $f$ like linear functions, we can usually only expect algorithms to obtain approximations within a limited runtime. Therefore, we are interested in how fast this approximation converges to the true $P_f$ or $L_f$ in terms of the number $n$ of times that the algorithm is allowed to evaluate $f$. To study this, we need to make some assumptions on $f$. While efficient sampling algorithms for suitable classes of concave $f$ are known, at least with access to gradients of $f$ \citep{dwivedi_log-concave_2018, mangoubi_dimensionally_2018, chewi_optimal_2021, altschuler_resolving_2023}, we are interested in larger classes of non-concave functions, which are defined in the following:

\begin{definition}[Function spaces and other notations] \label{def:notation}
For measurable functions $f: \calX \to \bbR$, we use the notation $\|f\|_\infty \equalDef \esssup_{x \in \calX} |f(x)|$, where the essential supremum is w.r.t.\ the Lebesgue measure on $\calX$. We define function spaces of $m$-times continuously differentiable functions\footnote{It would also be possible to replace $C^m(\calX)$ with the slightly larger Sobolev space $W^{m, \infty}(\calX)$.} whose derivatives are bounded by some constant $B \geq 0$:
\begin{IEEEeqnarray*}{+rCl+x*}
\calF_{d, m, B} & \equalDef & \left\{f \in C^m(\calX), \|f\|_{C^m} \leq B\right\}, \qquad \|f\|_{C^m} \equalDef \sup_{\alpha \in \bbN_0^d: |\alpha|_1 \leq m} \|\partial_\alpha f\|_\infty~.
\end{IEEEeqnarray*}
Here, we use the notation $|\alpha|_1 \equalDef \alpha_1 + \cdots + \alpha_d$ and $\partial_\alpha f = \frac{\partial^{|\alpha|_1} f}{\partial x_1^{\alpha_1} \cdots \partial x_d^{\alpha_d}}$. Moreover, if $f$ is Lipschitz, we denote its minimal Lipschitz constant by $|f|_1$. If $f$ is bounded, we denote its maximum by $M_f$. We define $\bar f \equalDef f - L_f$, such that $L_{\bar f} = 0$ and $P_{\bar f} = P_f$. Finally, we denote the uniform distribution on $\calX$ by $\calU(\calX)$.
\end{definition}

We study the worst-case error of algorithms over the function class $\calF_{d, m, B}$, which is formally defined in \Cref{sec:ibc}. This error depends on the variables $(B, n, d, m)$. We study the asymptotic behavior in terms of $n$ and $B$ while ignoring, for simplicity, constants that depend only on $m$ and $d$. Depending on the definition of the norm, such constants are often necessarily exponential in $d$ and represent the part of the curse of dimensionality that cannot be overcome in this setting \citep{novak_approximation_2009}.
For example, typical convergence rates for function approximation are of the form $O_{m, d}(Bn^{-m/d})$, which we sometimes also write as $O_{m, d}(\|f\|_{C^m} n^{-m/d})$ \citep{novak_deterministic_1988, wendland_scattered_2004}. As we will see later, the dependence on $B$ is not always linear, and tracking it is important since the function $f$ appears inside an exponential. \emph{When using asymptotic notation like $O_{m, d}$, we mean that the corresponding inequality should hold for \emph{all} values of $n \in \bbN_{\geq 1}$ and $B > 0$, not only large enough values.}\footnote{
Specifically, we use the notation $g(B, n, d, m) \leq O_{m, d}(h(B, n, d, m))$ to mean
\begin{IEEEeqnarray*}{+rCl+x*}
\forall d, m \in \bbN_{\geq 1} \exists C_{m, d} > 0 \forall B > 0, n \in \bbN_{\geq 1}: g(B, n, d, m) \leq C_{m, d}h(B, n, d, m)~,
\end{IEEEeqnarray*}
and we similarly write $g \geq \Omega_{m, d}(h)$ for $h \leq O_{m, d}(g)$, as well as $g = \Theta_{m, d}(h)$ for $g \leq O_{m, d}(h)$ and $h \leq O_{m, d}(g)$.}

We express bounds on the error $E$ achieved for $n$ function evaluations, such as $E = O_{m, d}(Bn^{-m/d})$. Some authors prefer to express rates in terms of the number of function evaluations needed to reach an error $E$ or lower, which would then, for example, be $n = O_{m, d}((B/E)^{d/m})$.

Sometimes, we explicitly include a temperature $\varepsilon > 0$ and formulate our theorems in terms of $f/\eps$ instead of $f$. It is well-known that in the limit of low temperatures ($\varepsilon \searrow 0$), sampling becomes essentially equivalent to optimization. Here, we give a quantitative version of this statement:

\begin{restatable}[Optimization limit]{lemma}{LemLipMaxBound} \label{lemma:lipschitz_maximization_bound}
Let $f: \calX \to \bbR$ be Lipschitz-continuous with Lipschitz constant $|f|_1 < \infty$. 
Then, for any temperature $\varepsilon > 0$, the maximum $M_f = \max_{x \in \calX} f(x)$ satisfies
\begin{IEEEeqnarray*}{+rCl+x*}
|M_f - \eps L_{f/\eps}| \leq \varepsilon d \log(1+3d^{-1/2}|f|_1/\varepsilon) \longrightarrow 0\quad \text{for }\varepsilon \searrow 0~. \IEEEyesnumber \label{eq:opt_error}
\end{IEEEeqnarray*}
Moreover, for any bounded and measurable $f: \calX \to \bbR$ and any $\delta \in (0, 1]$, we have
\begin{IEEEeqnarray*}{+rCl+x*}
P_{f/\eps}(\{x \in \calX \mid f(x) < \eps L_{f/\eps} - \eps \log(1/\delta)\}) \leq \delta~.
\end{IEEEeqnarray*}
\end{restatable}

\Cref{lemma:lipschitz_maximization_bound} is proven in \Cref{sec:appendix:introduction} and can be related to our function classes by using $|f|_1 \leq d^{1/2} \|f\|_{C^1}$, cf.\ \Cref{lemma:lipschitz_constant}. Note that such a convergence result does not hold for general bounded functions, as can be seen for the characteristic function $f = \bbone_{\{0\}}$, which satisfies $M_f = 1$ but $\eps L_{f/\eps} = 0$ for all $\varepsilon > 0$. \Eqref{eq:opt_error} is related to Corollary 1 of \cite{ma_sampling_2019}, which shows that to achieve an optimization error of $E$ through sampling, it is necessary that $1/\eps = \widetilde \Omega(d/E)$.

\cite{hwang_laplaces_1980} uses Laplace's method to show that under certain assumptions on the Hessians at the maximizers, $P_{f/\eps}$ converges weakly to a distribution on the maximizers. In contrast to optimization, where it typically does not matter which global maximum is found, the low-temperature limit of sampling often yields a unique distribution on the maximizers. \cite{talwar_computational_2019} uses this to show that optimization can be easier than sampling for very particular classes of functions. It is known that the optimal worst-case convergence rate for optimization on $\calF_{d, m, B}$ is the same as for approximation, i.e., $O_{m, d}(Bn^{-m/d})$ \citep[see][and references therein]{novak_deterministic_1988}. Recently, \cite{rudi_finding_2025} and \cite{woodworth_non-convex_2022} have shown polynomial-time optimization algorithms (in $n$ and $d$) achieving rates close to the optimal rate under relatively mild additional assumptions. 

In the high-temperature limit $\varepsilon \to \infty$, $P_{f/\eps}$ converges to a uniform distribution, and \cite{ma_sampling_2019} showed that the Metropolis-adjusted Langevin algorithm (MALA) can achieve exponentially fast convergence rates in $n$, although with exponential dependence on the Lipschitz constant of $\nabla f$. While MALA theoretically does not fit in our framework since it uses gradient information of $f$, this could be emulated using numerical differentiation. We show in \Cref{sec:rs_uniform} that in our setting, if $B$ is known, a fixed-budget version of rejection sampling can achieve similar rates \citep[see also][]{talwar_computational_2019}.

Since we know that polynomial-time algorithms with fast convergence rates are possible for the high-temperature case and for optimization, which is essentially the low-temperature limit, this poses the question of whether we can find such algorithms for the general sampling and log-partition problems. Ideally, such an algorithm should have the following properties:
\begin{itemize}
\item A convergence rate close to the optimal convergence rate $O_{m, d}(\|f\|_{C^m} n^{-m/d})$ for approximation, at least up to $m \geq \Omega(d)$, such that the exponent in the rate does not approach zero for large $d$,
\item Polynomial runtime $O(n^k)$ for some $k$ independent of $d$ and $m$. Especially, the runtime should be polynomial in $d$ and $m$. Moreover, the runtime should not depend on $\|f\|$. 
\item Adaptivity: The algorithm should achieve these rates without knowing $m$ and $\|f\|$. This is not investigated here, and for sampling and log-partition estimation, $m$ can often be known.
\end{itemize}

Regarding the implications of achieving near-optimal rates in polynomial time, consider the following example:

\begin{example} \label{ex:bayes_runtime}
Consider a Bayesian model where the data set $D = (D_1, \hdots, D_N)$ consists of $N$ observed samples that are assumed to be drawn in an i.i.d.\ fashion. Then, we can again model
\begin{IEEEeqnarray*}{+rCl+x*}
f(\theta) \equalDef \log p(\theta, D) = \log p(D_1 \mid \theta) + \cdots + \log p(D_N \mid \theta) + \log p(\theta)~,
\end{IEEEeqnarray*}
which is a sum of $N+1$ functions. Hence, we would expect that $\|f\|_{C^m}$ scales like $\Theta_{m, d}(N)$. 
\begin{enumerate}[(a)]
\item Suppose that we have a log-partition method with rate $\Theta_{m, d}(\|f\|_{C^m} n^{-m/d})$ and polynomial runtime $O_{m, d}(n^k)$. To achieve an error of $O(1)$ for the log-evidence $L_f$, this method would need $n = \Theta_{m, d}(N^{d/m})$ function evaluations and hence a runtime of $O_{m, d}(N^{kd/m})$. If $m = d$, the exponent $kd/m$ is independent of the dimension $d$, alleviating the curse of dimensionality except for the constants.
\item Now, suppose instead that the runtime is of the form $O_{m, d}(n^m)$ or the rates are of the form $O_{m, d}(\|f\|_{C^m} n^{-1/d})$ or $O_{m, d}(\|f\|_{C^m}^m n^{-m/d})$. In each case, to achieve an error of $O(1)$ for the log-evidence $L_f$, the resulting runtime would be polynomial in $N$, but the exponent would be proportional to $d$, meaning that reaching an acceptable accuracy is expensive for large $d$. %
\end{enumerate}
\end{example}

\subsection{Contribution}

This paper provides a theoretical exploration of the possible convergence rates for non-log-concave sampling and log-partition algorithms with smooth log-densities. To this end, we prove many results in the common setting of Definitions \ref{def:problem} and \ref{def:notation}. Specifically:
\begin{enumerate}[(1)]
\item We analyze the information-based complexity of the sampling and log-partition problems, i.e., the worst-case optimal rates without computational constraints, in \Cref{sec:ibc}. For algorithms that evaluate $f$ at a deterministic set of points, we show that the optimal rate for the log-partition problem is $\Theta_{m, d}(Bn^{-m/d})$, i.e., the same as for approximation. For the bounded total variation and 1-Wasserstein metrics, the optimal rate for sampling is $\Theta_{m, d}(\min\{1, Bn^{-m/d}\})$. For algorithms that are allowed to evaluate $f$ at a stochastic set of points, we identify two different regimes: An \emph{\quot{optimization regime}} when $Bn^{-m/d} \gg 1$, where the log-partition problem is close to optimization and the same optimal rates hold, and a \emph{\quot{quadrature regime}} when $Bn^{-m/d} \ll 1$, where the log-partition problem behaves more like classical quadrature and faster rates are possible. %
\item We show reductions between different problems. For example, we discuss variants of existing methods to employ log-partition algorithms for sampling and vice versa, and analyze the resulting guarantees for the rates. We also discuss how approximate sampling algorithms can be employed for optimization. Moreover, we show how function approximation yields reductions between different runtime complexities, convergence rates, and from stochastic to deterministic evaluation points.
\item We analyze bounds on the convergence rates for different algorithms. For example, we show that it is possible to achieve the rate $O_{m, d}(Bn^{-m/d})$, but with runtime $O(n^m)$, which is polynomial but still involves the curse of dimensionality since we need $m = \Theta(d)$ to beat the curse of dimensionality in the convergence rate. We show that other simple and efficient algorithms also fail to achieve the optimal rates in different ways, sometimes with multi-regime behavior. Finally, we study an approach toward the log-partition problem by \cite{bach_sum--squares_2025}, whose optimization limit has been used by \cite{woodworth_non-convex_2022} to obtain near-optimal optimization rates in polynomial time. We show that all versions of this approach necessarily fail to exceed the rate $O_{m, d}(Bn^{-2/d})$ in an intermediate temperature regime $\varepsilon \sim n^{-2/d}$ (corresponding to $B \sim n^{2/d}$).
\end{enumerate}

In conclusion, the general non-log-concave sampling and log-partition problems suffer from the curse of dimensionality (\Cref{sec:ibc}). In principle, this can be alleviated to some extent if the log-density $f$ is smooth. Optimal convergence rates can be achieved (at least in some regimes) using surrogates (\Cref{ex:tradeoff_more_points}), but it is unclear if they can be achieved with fixed-order polynomial runtime complexity. Convergence rates of algorithms may transition between different regimes, and given enough points, sampling algorithms using stochastic points can exhibit superexponentially fast convergence.

\subsection{Related Work}

The analysis of sampling algorithms has received considerable attention in recent years. In the case where $p_f$ is (strongly) log-concave, that is, if $f$ is (strongly) concave, convergence rates of Markov chain Monte Carlo (MCMC) sampling algorithms have been studied extensively. For example, good convergence rates in terms of the dimension~$d$ have been established for versions of the Langevin algorithm \citep{chewi_optimal_2021, altschuler_resolving_2023} and Hamiltonian Monte Carlo \citep{mangoubi_dimensionally_2018}. \cite{chewi_query_2022} establish an algorithm with optimal convergence rate for the case $d=1$, while not much is known about algorithm-independent lower bounds in other cases.

For sampling from more general non-log-concave distributions, convergence rates have been established for versions of the Langevin algorithm.
\cite{bou-rabee_nonasymptotic_2013} showed an essentially geometric convergence result in TV distance for a class of non-log-concave Gibbs distributions, but without clear dependence of the constants on $f$. \cite{mangoubi_nonconvex_2019} and \cite{zou_faster_2021} prove convergence rates that are polynomial in $d$ but additionally depend on properties of $f$ through the Cheeger constant. Other assumptions on $f$ like log-Sobolev or Poincaré inequalities also allow fast convergence rates if the involved constants are not too small \citep{vempala_rapid_2019, ma_sampling_2019}. The analysis of \cite{ma_sampling_2019} and \cite{cheng_sharp_2018} is closer to our setting, and their convergence rate is polynomial in $d$ as well, but their rate exhibits an exponential dependence on the Lipschitz constant of $\nabla f$ and the radius of the domain where $f$ is non-log-convex. \cite{bou-rabee_coupling_2020} obtain similar results for Hamiltonian Monte Carlo.
\cite{balasubramanian_towards_2022} show that even for non-log-concave distributions, averaged Langevin Monte Carlo converges quickly to a distribution with low relative Fisher information to the target distribution, although this does not imply that the distribution is close to the target distribution with respect to other measures, such as the total variation distance. \cite{chewi_fisher_2023} prove corresponding lower bounds. \cite{woodard_sufficient_2009} show that the mixing time of parallel and simulated tempering for certain distributions can scale exponentially with $d$, but in a setting different from ours. \cite{achddou_minimax_2019} propose and analyze an adaptive rejection sampling algorithm using a piecewise constant approximation of the density. Their setting is significantly different from ours as well, and they only consider functions of low (Hölder) smoothness and regimes with large $n$. \cite{marteau-ferey_sampling_2022} propose an approximation-based sampling algorithm with a rate similar to $O_{m, d, B}(n^{-m/d})$ but without analyzing the dependence on $B$.

Another related line of work studies the relation of sampling to optimization. Through their analysis of Langevin algorithms in the non-log-concave setting, \cite{ma_sampling_2019} show that there are settings where sampling is easier than optimization. \cite{talwar_computational_2019} provides a simpler argument and shows that the converse can also occur for special function classes. The relation between sampling and optimization is also exploited in simulated annealing \citep{kirkpatrick_optimization_1983}. A different connection between sampling and optimization stems from \cite{jordan_variational_1998}, who showed that Langevin-type sampling can be interpreted as a gradient flow over distributions for the Wasserstein metric. For an overview of connections between sampling and optimization, we also refer to \cite{cheng_interplay_2020}.

The log-partition problem is often addressed through sampling algorithms, for example via thermodynamic integration \citep{kirkwood_statistical_1935}. For an overview of thermodynamic integration and other methods for the log-partition problem, we refer to \cite{gelman_simulating_1998} and \cite{friel_estimating_2012}. \cite{ge_estimating_2020} analyze an annealing algorithm combined with multilevel Monte Carlo sampling for the log-partition problem in the log-concave setting, and also give an information-based lower bound on the achievable convergence rate. Another popular approach is the Laplace approximation \citep{laplace_memoire_1774}, however, its log-partition function does not converge to the true log-partition function as $n \to \infty$. Well-tempered metadynamics \citep{barducci_well-tempered_2008} is a popular approach towards the log-partition problem in molecular dynamics simulations, although it relies on a well-chosen low-dimensional collective variable representation. Recently, \cite{marteau-ferey_sampling_2022} suggested an approach that performs sampling via estimating the log-partition function. \cite{bach_information_2023} and \cite{bach_sum--squares_2025} suggest further approaches toward solving the log-partition problem.

To analyze possible convergence rates for the sampling and log-partition problem without computational constraints, we use the framework of information-based complexity. Here, we refer to \cite{novak_deterministic_1988} and \cite{traub_information-based_2003} for an overview of this topic. In particular, our work is motivated by \cite{rudi_finding_2025} and \cite{woodworth_non-convex_2022}, who demonstrated that for optimization, convergence rates close to the optimal rates from information-based complexity can be achieved in polynomial time.

The rest of our paper is organized as follows: In \Cref{sec:ibc}, we study upper and lower bounds for the information-based complexity of different variants of the sampling and log-partition problems. In \Cref{sec:relations}, we study relations and reductions between different variants of the sampling, log-partition, and optimization problems. We then investigate convergence rates of different algorithms in \Cref{sec:algorithms}. We compare some of these algorithms experimentally in \Cref{sec:sampling:experiments} before concluding in \Cref{sec:sampling:conclusion}. All proofs are provided in the appendix, which is structured analogously to the main part of this paper.

\section{Information-based Complexity} \label{sec:ibc}

In this section, we look at the log-partition and sampling problems from the viewpoint of (worst-case) \emph{information-based complexity}, where one is interested in what is possible if one is not constrained computationally but only by the number~$n$ of function evaluations of the unknown function~$f$. We adopt the general setting of \cite{novak_deterministic_1988}, where one is given a function space $\calF$ (such as $\calF_{d, m, B}$) of functions $f: \calX \to \bbR$ and wishes to approximate a map $S: \calF \to \calM$, with the approximation error on $\calM$ measured by a metric $D$. For example, the following problems are considered by \cite{novak_deterministic_1988}:
\begin{itemize}
\item \emph{Approximation}: $S_{\mathrm{app}}(f) \equalDef f$ and $\Dinfty(f, g) \equalDef \|f - g\|_\infty$.
\item \emph{Optimization}: $S_{\mathrm{opt}^*}(f) \equalDef \sup_{x \in \calX} f(x)$ and $\Dabs(a, b) \equalDef |a - b|$.
\item \emph{Integration}: $S_{\mathrm{int}}(f) \equalDef \int_{\calX} f(x) \diff x$ and $\Dabs(a, b) = |a - b|$.
\end{itemize}
We can define our sampling and log-partition problems in this context as follows:
\begin{itemize}
\item \emph{Log-partition}: $S_L(f) = L_f$ and $\Dabs(a, b) = |a - b|$.
\item \emph{Sampling}: While a sampling algorithm produces samples, we do not want to compare errors of individual samples but the error of the distribution of the samples. Therefore, we set $S_{\mathrm{samp}}(f) \equalDef P_f$. For $D(P, Q)$, we can use different metrics or divergences on probability distributions, which will be discussed in \Cref{sec:ibc:deterministic_points}.
\end{itemize}

\subsection{Deterministic Evaluation Points} \label{sec:ibc:deterministic_points}

To consider minimax optimal convergence rates, we still need to define a space $\calA$ of admissible maps $\tilde S: \calF \to \calM$. Here, we will first consider maps that evaluate functions in a deterministic set of points, before considering stochastic points in \Cref{sec:ibc:stochastic_points}. For example, we define
\begin{IEEEeqnarray*}{+rCl+x*}
\calA_n \equalDef \left\{\tilde S = \phi \circ N \mid N(f) = (f(x_1), \hdots, f(x_n))\text{ for some } x_1, \hdots, x_n \in \calX\right\}~,
\end{IEEEeqnarray*}
the set of maps that only evaluate $f$ in $n$ deterministic and non-adaptive points. We can also allow adaptive points by defining
\begin{IEEEeqnarray*}{+rCl+x*}
\Aad_n \equalDef \left\{\tilde S = \phi \circ N \mid N(f) = (f(x_1), f(x_2(f(x_1))), \hdots, f(x_n(f(x_1), \hdots, f(x_{n-1}))))\right\}~,
\end{IEEEeqnarray*}
where evaluation points may be chosen depending on previous function values. We are interested in the (non-adaptive/adaptive) minimax optimal error
\begin{IEEEeqnarray*}{+rClrCl+x*}
e_n(\calF, S, D) & \equalDef & \inf_{\tilde S \in \calA_n} \sup_{f \in \calF} D(S(f), \tilde S(f)), \qquad & \ead_n(\calF, S, D) & \equalDef & \inf_{\tilde S \in \calA_n^{\mathrm{ad}}} \sup_{f \in \calF} D(S(f), \tilde S(f))~.
\end{IEEEeqnarray*}
The sets $\calA_n$ and $\Aad_n$ can be interpreted as classes of \quot{black-box algorithms} that are only constrained in the number of evaluations of $f$ but not in terms of computational efficiency or computability. The minimax-optimal errors $e_n$ and $\ead_n$ thus give lower bounds to what can be achieved by computationally efficient algorithms.

For the case of sampling, maps $\tilde S \in \calA_n$ (or $\Aad_n$) produce distributions based on $n$ evaluations of a function~$f$. They correspond to idealized sampling algorithms in the following sense: We consider an idealized sampling algorithm to take some source of randomness $\omega$ sampled from a distribution $P_\Omega$ independent of $f$, and then output a random sample $X_f(\omega) = \tilde \phi(N(f), \omega)$. For example, $\omega$ could be a sequence of i.i.d. random variables from the uniform distribution $\calU([0, 1])$ on the interval $[0, 1]$. The maps $\tilde S \in \calA_n$ (or $\Aad_n$) then correspond to the distributions produced by such sampling algorithms, i.e.,
\begin{IEEEeqnarray*}{+rCl+x*}
\tilde S(f) = \text{ distribution of $X_f(\omega)$ for $\omega \sim P_\Omega$}.
\end{IEEEeqnarray*}

The following theorem, which is proven in \Cref{sec:appendix:deterministic_points}, adapts known results on minimax optimal rates to our considered function spaces. 

\begin{restatable}[adapted from \cite{novak_deterministic_1988}]{theorem}{thmNovak} \label{thm:novak}
We have
\begin{IEEEeqnarray*}{+rClrCl+x*}
e_n(\calF_{d, m, B}, \Sapp, \Dinfty) & = & \Theta_{m, d}(Bn^{-m/d}), \qquad & \ead_n(\calF_{d, m, B}, \Sapp, \Dinfty) & = & \Theta_{m, d}(Bn^{-m/d}), \\
e_n(\calF_{d, m, B}, \Sopts, \Dabs) & = & \Theta_{m, d}(Bn^{-m/d}), \qquad & \ead_n(\calF_{d, m, B}, \Sopts, \Dabs) & = & \Theta_{m, d}(Bn^{-m/d}), \\
e_n(\calF_{d, m, B}, \Sint, \Dabs) & = & \Theta_{m, d}(Bn^{-m/d}), \qquad & \ead_n(\calF_{d, m, B}, \Sint, \Dabs) & = & \Theta_{m, d}(Bn^{-m/d})~.
\end{IEEEeqnarray*}
\end{restatable}

\cite{novak_deterministic_1988} states these results in a form like $e_n(\calF_{d, m, 1}, \Sapp, \Dinfty) = \Theta_{m, d}(n^{-m/d})$. This implies the rates for general $B \geq 0$ in the theorem above since $\Sapp$ and $\Dinfty$ are positively homogeneous, which leads to $\Dinfty(\Sapp(Bf), Bg) = B\Dinfty(\Sapp(f), g)$. The same holds for optimization and integration, but not for log-partition estimation and sampling. Hence, for our considered problems, it is important to explicitly study the dependence on $B$, since it is not necessarily linear. The optimal rates for approximation can be achieved, for example, using piecewise polynomial interpolation, local polynomial reproductions, or moving least squares \citep{wendland_scattered_2004}, see also \Cref{thm:mls}. The optimal rates for optimization and integration can be achieved by optimizing or integrating a corresponding approximation.

We use the following distance measures for probability distributions $P, Q$ on $\calX$:
\begin{itemize}
\item The sup-log distance $\Dsuplog(P, Q) \equalDef \left\|\log\left(\frac{\diff P}{\diff Q}\right)\right\|_\infty$, where $\|\cdot\|_\infty$ is taken over $\calX$, and $\Dsuplog(P, Q) = \infty$ whenever $P$ and $Q$ are not both absolutely continuous with respect to each other.
The sup-log distance is a symmetrized version of the max-divergence $D_\infty(P||Q)$, which is used in differential privacy \citep{dwork_boosting_2010} and is a special case of Rényi divergences for $\alpha = \infty$ \citep[cf.][]{van_erven_renyi_2014}. The sup-log distance is particularly well-suited to our setting, thanks to its relation to uniform approximation. 
\item The total variation distance $\DTV(P, Q) \equalDef \sup_{A \subseteq \calX \text{ measurable}} |P(A) - Q(A)|$.
\item The 1-Wasserstein distance $\DW(P, Q) \equalDef \inf_{X \sim P, Y \sim Q} \bbE \|X - Y\|_2$, also known as Kantorovich–Rubinstein or earth mover distance.
\end{itemize}
The sup-log, total variation, and 1-Wasserstein distances are metrics.\footnote{For the sup-log distance, the triangle inequality follows from $\log\left(\frac{\diff P}{\diff R}\right) = \log\left(\frac{\diff P}{\diff Q}\right) + \log\left(\frac{\diff Q}{\diff R}\right)$.} We first show that these quantities can be bounded in terms of the approximation error:

\begin{restatable}[Upper bounds via approximation]{proposition}{propAppUpperBounds} \label{prop:app_bounds}
For bounded and measurable $f, g: \calX \to \bbR$, we have
\begin{enumerate}[(a)]
\item $|L_f - L_g| \leq \|f - g\|_\infty$.
\item $d^{-1/2}\DW(P_f, P_g) \leq \DTV(P_f, P_g) \leq \Dsuplog(P_f, P_g) \leq 2\|f - g\|_\infty$.
\end{enumerate}
\end{restatable}
\Cref{prop:app_bounds} is proven in \Cref{sec:appendix:deterministic_points}. For the KL divergence, which we will not study further, we can leverage the results of \Cref{prop:app_bounds} by using the trivial bound $\DKL{P}{Q} \leq \Dsuplog(P, Q)$ as well as the inequality $\DKL{P}{Q} \leq \Dsuplog(P, Q)(e^{\Dsuplog(P, Q)} - 1)$ from Lemma III.2 of \cite{dwork_boosting_2010}.

\Cref{thm:novak} and \Cref{prop:app_bounds} lead to upper bounds on the minimax optimal rates. Combined with the trivial upper bound $\DTV(P, Q) \leq 1$, these are optimal for the deterministic point setting: 

\begin{restatable}[Information-based complexity of sampling and log-partition with deterministic evaluation points]{theorem}{thmDeterministicPointsRates} \label{thm:deterministic_points_rates}
We have
\begin{IEEEeqnarray*}{+rCl+x*}
e_n(\calF_{d, m, B}, S_L, \Dabs) & = & \Theta_{m, d}(Bn^{-m/d}), \\
e_n(\calF_{d, m, B}, \Ssamp, \Dsuplog) & = & \Theta_{m, d}(Bn^{-m/d}), \\
e_n(\calF_{d, m, B}, \Ssamp, \DTV) & = & \Theta_{m, d}(\min\{1, Bn^{-m/d}\}), \\
e_n(\calF_{d, m, B}, \Ssamp, \DW) & = & \Theta_{m, d}(\min\{1, Bn^{-m/d}\}),
\end{IEEEeqnarray*}
and the same rates hold for adaptive points.
\end{restatable}

\Cref{thm:deterministic_points_rates} is proven in \Cref{sec:appendix:deterministic_points}. The minimax optimal rates for optimization can be related to those for approximation on a very general class of function spaces \citep{novak_deterministic_1988}.
For sampling, such a general relationship does not hold: For example, the set $\calF \equalDef \{f : \calX \to \bbR \mid \|f\|_\infty \leq 1, \{x \mid f(x) \neq 0\} \text{ is finite}\}$ satisfies $e_n(\calF, \Sapp, \Dinfty) = 1$ for all $n \in \bbN$, but all functions $f \in \calF$ have the same distribution and the same log-partition function.
However, our proofs for the lower bounds in \Cref{thm:deterministic_points_rates} follow the general idea that underlies many lower bounds for Sobolev-type functions: place bumps with small support in regions that the algorithm does not query.

\subsection{Stochastic Evaluation Points} \label{sec:ibc:stochastic_points}

We also want to consider methods that are allowed to choose the points $x_i$ stochastically, such as Monte-Carlo type methods \citep{metropolis_monte_1949, brooks_handbook_2011}. For the log-partition problem, we again follow \cite{novak_deterministic_1988} and define the set $\sC(\Aad_n)$ of random variables $\tilde S: \Omega \to \Aad_n$ with given base distribution $P_\Omega$ with associated minimax optimal error
\begin{IEEEeqnarray*}{+rCl+x*}
\ssad_n(\calF, S, D) & \equalDef & \inf_{(\tilde S, P_\Omega) \in \sC(\Aad_n)} \sup_{f \in \calF} \bbE_{\omega \sim P_\Omega} D(S(f), \tilde S(\omega)(f))~. \IEEEyesnumber \label{eq:stoch_minmax}
\end{IEEEeqnarray*}
\cite{novak_deterministic_1988} defines further variants, for example, with $L_2(P_\Omega)$ instead of $L_1(P_\Omega)$ convergence or more limited stochastic resources, which we will not discuss here for simplicity.

When applying this definition to sampling, a map $\tilde S$ would output a \emph{random} distribution. However, the random samples produced by a sampling algorithm typically still follow a \emph{fixed} distribution, regardless of whether the function $f$ is evaluated in deterministically or randomly chosen points. Hence, the model in \Eqref{eq:stoch_minmax} is inadequate for sampling. Instead, we consider again idealized sampling algorithms using some randomness $\omega \sim \Omega$, but this time, we allow the function to be evaluated in randomly and adaptively chosen points, by considering random samples of the form $X_f(\omega) = \tilde \phi(N(f, \omega), \omega)$. We then denote the corresponding map from $f$ to $P_{X_f}$ by $\tilde S$ and define the set $\Aadstoch_n$ of all $\tilde S$ that can be realized in this fashion using $n$ function evaluations. We then define
\begin{IEEEeqnarray*}{+rCl+x*}
\eadstoch_n(\calF, \Ssamp, D) & \equalDef & \inf_{\tilde S \in \Aadstoch_n} \sup_{f \in \calF} D(\Ssamp(f), \tilde S(f))~.
\end{IEEEeqnarray*}
Unlike the deterministic points setting, the stochastic points setting potentially requires evaluating $f$ at $n$ different points for every generated sample. This has the unintuitive consequence that for a map $\tilde S \in \Aadstoch_n$, the distribution $\tilde S(f)$ typically depends on the values of $f$ at infinitely many points, but a sample from $\tilde S(f)$ can be drawn by only evaluating $f$ at $n$ (stochastic) points.

Again, results for approximation, optimization, and integration are known and can be adapted to our function classes:
\begin{restatable}[adapted from \cite{novak_deterministic_1988}]{theorem}{thmNovakStoch} \label{thm:novak_stoch}
We have
\begin{IEEEeqnarray*}{+rCl+x*}
\ssad_n(\calF_{d, m, B}, \Sapp, \Dinfty) & = & \Theta_{m, d}(Bn^{-m/d}), \\
\ssad_n(\calF_{d, m, B}, \Sopts, \Dabs) & = & \Theta_{m, d}(Bn^{-m/d}), \\
\ssad_n(\calF_{d, m, B}, \Sint, \Dabs) & = & \Theta_{m, d}(Bn^{-1/2-m/d})~.
\end{IEEEeqnarray*}
\end{restatable}
For a proof sketch, we refer to \Cref{sec:appendix:stochastic_points}. The faster rate for integration can be achieved by spending half of the $n$ points for approximating $f$ with $g$ and spending the other half of the points on Monte Carlo quadrature to estimate the error \citep{novak_deterministic_1988}
\begin{IEEEeqnarray*}{+rCl+x*}
\int f(x) \diff x - \int g(x) \diff x = \bbE_{x \sim \calU(\calX)}[f(x)-g(x)]~.
\end{IEEEeqnarray*}
For a more practical algorithm, we refer to \cite{chopin_higher-order_2024}. For log-partition estimation, we can similarly use an importance sampling formulation
\begin{IEEEeqnarray*}{+rCl+x*}
L_f - L_g & = & \log\left(\bbE_{x \sim P_g} \left[\exp(f(x) - g(x))\right]\right)~.
\end{IEEEeqnarray*}

\begin{restatable}[Upper bound for stochastic log-partition]{theorem}{thmUpperStochLog} \label{thm:upper_stoch_log}
There exists a constant $C_{m, d} > 0$ depending only on $m$ and $d$ such that
\begin{IEEEeqnarray*}{+rCl+x*}
\ssad_n(\calF_{d, m, B}, S_L, \Dabs) \leq O_{m, d}\left(\min\left\{Bn^{-m/d}, \exp(C_{m, d}Bn^{-m/d}) B n^{-1/2-m/d})\right\}\right)~.
\end{IEEEeqnarray*}
\end{restatable}

The upper bound above, which is proven in \Cref{sec:appendix:stochastic_points}, exhibits a fast transition between the rates $n^{-m/d}$ and $n^{-1/2-m/d}$. This is necessary, as we can exploit the relation of the log-partition problem to optimization to show that the rate $n^{-m/d}$ is optimal in an optimization regime where $Bn^{-m/d} \gg 1$:

\begin{restatable}[Lower bound for stochastic log-partition]{proposition}{propLowerStochLog} \label{prop:lower_stoch_log}
For $m \geq 1$, we have
\begin{IEEEeqnarray*}{+rCl+x*}
\ssad_n(\calF_{d, m, B}, S_L, \Dabs) \geq \Omega_{m, d}(Bn^{-m/d}) - d\log(1+3B)~.
\end{IEEEeqnarray*}
\end{restatable}

\Cref{prop:lower_stoch_log} is proven in \Cref{sec:appendix:stochastic_points}. We leave a lower bound outside of the optimization regime as an open problem; however, we conjecture that the rate $Bn^{-1/2-m/d}$ from the upper bound in \Cref{thm:upper_stoch_log} cannot be improved. For a certain class of strongly concave $f$ with Lipschitz gradient, Theorem 5.1 by \cite{ge_estimating_2020} contains a lower bound which, in our setting, could be roughly expressed as $\Omega_{d, B}(n^{-1/(2-c/d)})$ for some constant $c$. A simple Taylor expansion shows $-1/(2-c/d) \leq -1/2-(c/4)/d$, hence this rate is compatible with our upper bound for $m=2$ if $c \geq 8$.

\begin{algorithm}[htb]
\caption{Rejection sampling with proposal distribution $P_g$ limited to $n$ function evaluations.} \label{alg:rejection_sampling}
\begin{algorithmic}
\Function{RejectionSampling}{$f$, $g$, number of steps $n$}
	\For{$i$ from $1$ to $n$}
		\State Sample $x \sim P_g$ and $u \sim \calU([0, 1])$
		\State Return $x$ if $ue^{g(x)} \leq e^{f(x)}$
	\EndFor
	\State \Return Sample from $P_g$
\EndFunction
\end{algorithmic}
\end{algorithm}

To achieve better rates for sampling in the stochastic points setting, we combine approximation with a budget-limited version of rejection sampling defined in \Cref{alg:rejection_sampling}. If $g$ is shifted appropriately such that it upper-bounds $f$, we obtain the following convergence rate bound: 

\begin{restatable}[General rejection sampling bound]{lemma}{lemRejectionSampling} \label{lemma:rejection_sampling}
Suppose that $f, g: \calX \to \bbR$ are bounded and measurable with $f(x) \leq g(x)$ for all $x \in \calX$. In this case, the distribution $\tilde P_f$ of $\textsc{RejectionSampling}(f, g, n)$ satisfies
\begin{IEEEeqnarray*}{+rCl+x*}
\tilde P_f & = & (1 - p_R)P_f + p_RP_g \IEEEyesnumber \label{eq:rejection_distribution} \\
\Dsuplog(P_f, \tilde P_f) & \leq & \min\left\{\Dsuplog(P_f, P_g), p_R(\exp(\Dsuplog(P_f, P_g)) - 1)\right\} \\
\DTV(P_f, \tilde P_f) & = & p_R \DTV(P_f, P_g) \\
\DW(P_f, \tilde P_f) & = & p_R \DW(P_f, P_g)~,
\end{IEEEeqnarray*}
where $p_R = (1 - Z_f/Z_g)^n \leq \exp(-nZ_f/Z_g)$ is the probability of overall rejection.
\end{restatable}

The proof can be found in \Cref{sec:appendix:stochastic_points}. Due to the early stopping after $n$ rejections, rejection sampling may significantly oversample regions where $p_f$ is very small. Since $\Dsuplog$ is very sensitive to this behavior, the corresponding bound is worse than for $\DTV$ and $\DW$.

By using half of the $n$ points to create an approximation $g$ and then using a shifted version of $g$ for rejection sampling with the other half of the $n$ points, we obtain the following upper bound on the minimax optimal error:

\begin{restatable}[Upper bound for sampling with stochastic evaluation points]{theorem}{thmUpperStochSampling} \label{thm:upper_stoch_sampling}
There exists a constant $C_{m, d} > 0$ such that
\begin{IEEEeqnarray*}{+rCl+x*}
\eadstoch_n(\calF_{d, m, B}, \Ssamp, \Dsuplog) \leq \begin{cases}
O_{m, d}(B n^{-m/d}) &, C_{m, d} B n^{-m/d} > 1 \\
O_{m, d}((C_{m, d} B n^{-m/d})^{n/2+1}) &, C_{m, d} B n^{-m/d} \leq 1~.
\end{cases}
\end{IEEEeqnarray*}
\end{restatable}
\Cref{thm:upper_stoch_sampling} is proven in \Cref{sec:appendix:stochastic_points}. Combinations of approximation and rejection sampling have also been used, for example, by \cite{achddou_minimax_2019} and \cite{chewi_query_2022}. For $C_{m, d}Bn^{-m/d} \leq 1$, the upper bound above decays faster than exponential in $n$. The bound is not tight, as the exponent $n/2+1$ can at least be improved close to $n$ at the cost of increasing the constant $C_{m, d}$. However, for the optimization regime, the bound is tight: 

\begin{restatable}[Lower bound for sampling with stochastic evaluation points]{theorem}{thmLowerStochSampling} \label{thm:lower_stoch_sampling}
There exists a constant $c_{m, d} > 0$ such that for $B > 0$ and $n \in \bbN$ with $Bn^{-m/d} \geq c_{m, d}(1 + \log(n))$, we have
\begin{IEEEeqnarray*}{+rCl+x*}
\eadstoch_n(\calF_{d, m, B}, \Ssamp, \Dsuplog) & \geq & \Omega_{m, d}(Bn^{-m/d}) \\
\eadstoch_n(\calF_{d, m, B}, \Ssamp, \DTV) & \geq & \Omega_{m, d}(1) \\
\eadstoch_n(\calF_{d, m, B}, \Ssamp, \DW) & \geq & \Omega_{m, d}(1)~.
\end{IEEEeqnarray*}
\end{restatable}

The proof of \Cref{thm:lower_stoch_sampling} in \Cref{sec:appendix:stochastic_points} uses the classical approach of hiding a bump, although explicitly exploiting the relation to optimization via \Cref{prop:opt_by_apx_sampling} might also work.
Proving lower bounds for sampling with stochastic points outside of the optimization regime seems difficult. Indeed, when restricting the function class a bit further, we can even achieve zero error:
\begin{restatable}{proposition}{propExactSampling}
Let $\calF \equalDef \{f \in C(\calX) \mid \|f\|_\infty \leq \log(3/2), L_f = 0\}$. Then, for all $n \geq 1$,
\begin{IEEEeqnarray*}{+rCl+x*}
\eadstoch_n(\calF, \Ssamp, \Dsuplog) = 0~.
\end{IEEEeqnarray*}
\end{restatable}
The proof idea, executed in \Cref{sec:appendix:stochastic_points}, is to use $\textsc{RejectionSampling}(\tilde f, g, 1)$, where $\tilde f(x) \equalDef \log(2\exp(f(x)) - 1)$ and $g(x) = \log(2)$ are constructed such that the resulting distribution is exactly $P_f$. The assumption that $L_f$ is known is necessary to exactly control the acceptance probability in the rejection sampling step.

\Cref{table:ibc_rates} summarizes the obtained convergence rates.

\begin{table}
\centering
\begin{tabular}{lcc}
\toprule %
& Optimization regime  & Sampling regime \\
& $Bn^{-m/d} \geq C_{m,d}\log(1+B)$ & $Bn^{-m/d} \leq c_{m, d}$ \\
\midrule
Log-partition (det.) & $\Theta_{m, d}(Bn^{-m/d})$ & $\Theta_{m, d}(Bn^{-m/d})$ \\
Log-partition (stoch.) & $\Theta_{m, d}(Bn^{-m/d})$ & $O_{m, d}(Bn^{-m/d-1/2})$ \\
\midrule
Sampling (det., $\Dsuplog$) & $\Theta_{m, d}(Bn^{-m/d})$ & $\Theta_{m, d}(Bn^{-m/d})$ \\
Sampling (stoch., $\Dsuplog$) & $\Theta_{m, d}(Bn^{-m/d})$ & $O_{m, d}((C_{m, d} B n^{-m/d})^{n/2+1})$ \\
Sampling (det., $\DTV$ or $\DW$) & $\Theta_{m, d}(1)$ & $\Theta_{m, d}(Bn^{-m/d})$ \\
Sampling (stoch., $\DTV$ or $\DW$)\hspace{-0.3cm} & $\Theta_{m, d}(1)$ & $O_{m, d}((C_{m, d} B n^{-m/d})^{n/2+1})$ \\
\bottomrule
\end{tabular}
\vspace{0.05in}
\caption{Obtained information-based complexity convergence rates for the log-partition and sampling problems with deterministic or stochastic evaluation points. Rates are taken from \Cref{thm:deterministic_points_rates}, \Cref{thm:upper_stoch_log}, \Cref{prop:lower_stoch_log}, \Cref{thm:upper_stoch_sampling}, and \Cref{thm:lower_stoch_sampling}.} \label{table:ibc_rates}
\end{table}

\section{Relations Between Different Problems} \label{sec:relations}

In this section, we study how different problems, such as sampling, log-partition estimation, and optimization, are related. In particular, we consider reductions between algorithms, their runtime complexities, and their convergence rates. Again, certain bounds can be established via the connection to function approximation. For this, we need an efficient approximation method that achieves optimal convergence rates while producing a smooth approximation. This is possible using the moving least squares method \citep{lancaster_surfaces_1981}, 
which produces an approximant $g(x) = g_x(x)$, where $g_x$ is a local polynomial regression function fitted using a smooth local weight function $w(x_i, x)$.
The following theorem shows that the moving least squares method achieves the desired properties:

\begin{restatable}[adapted from \cite{li_error_2016} and \cite{mirzaei_analysis_2015}]{theorem}{thmMls} \label{thm:mls}
Let $m, d \in \bbN_{\geq 1}$. Using the moving least squares method, it is possible to construct an approximation $f_n$ of $f \in C^m(\calX)$ using $n$ deterministic non-adaptive function evaluations such that
\begin{enumerate}[(a)]
\item $\|f - f_n\|_{C^k} \leq O_{m, d}(\|f\|_{C^m} n^{-(m-k)/d})$ for $k \in \{0, 1, \hdots, m\}$,
\item the runtime for pre-computations for $f_n$ (done once before evaluation) is $O_{m, d}(n)$, and
\item the runtime of evaluating $f_n$ at a point $x \in \calX$ is $O_{m, d}(1)$.
\end{enumerate}
\end{restatable}
We prove \Cref{thm:mls} in \Cref{sec:appendix:relations}.

\subsection{Runtime-Accuracy Trade-off}

When investigating sampling and log-partition algorithms, we study their convergence rate and their runtime complexity both in terms of the number $n$ of required function evaluations. Here, we show that these two quantities can be traded off against each other to some extent. Improving the computational complexity at the cost of worse convergence rates is easy by increasing $n$ without using the additional function values:

\begin{example}[Trading convergence rates for better runtime complexity] \label{ex:tradeoff_less_points}
Suppose that we have an algorithm $A$ for the sampling or log-partition problems with convergence rate $\Theta_{m, d}(\|f\|_{C^m}n^{-\alpha_{m, d}})$ and runtime $\Theta_{m, d}(n^{\beta_{m, d}})$. We can then evaluate $f$ in $n$ points, but only use $N \leq n$ of these points for $A$. If $N = \Theta_{m, d}(n^\gamma)$, $\gamma \in (0, 1]$, we obtain a (slower) convergence rate of $\Theta_{m, d}(\|f\|_{C^m}N^{-\alpha_{m, d}}) = \Theta_{m, d}(\|f\|_{C^m}n^{-\gamma \alpha_{m, d}})$ and a (faster) runtime of $\Theta_{m, d}(n + N^{\beta_{m, d}}) = \Theta_{m, d}(n^{\max\{1, \gamma \beta_{m, d}\}})$. A similar construction could be used to move constants $C_{m, d} \geq 1$ or potential factors $\|f\|_{C^m}^k \geq 1$ from the runtime to the convergence rate. 
\end{example}

Of course, the construction in \Cref{ex:tradeoff_less_points} does not improve the runtime needed to reach a desired error level, but it shows that some combinations of runtime complexity and convergence rates are not better than others. To trade runtime complexity for better convergence rates, an analogous construction is not possible, since it would need to use $N > n$ function evaluations, which would contradict the definition of $n$. However, we can instead use $N$ evaluations of an approximant created using $n$ function evaluations:

\begin{example}[Trading runtime complexity for better convergence rates] \label{ex:tradeoff_more_points}
Suppose again that we have an algorithm $A$ for the sampling or log-partition problems with convergence rate $\Theta_{m, d}(\|f\|_{C^m}n^{-\alpha_{m, d}})$ and runtime $\Theta_{m, d}(n^{\beta_{m, d}})$. We consider an algorithm resulting from the following construction:
\begin{enumerate}[(1)]
\item Use an approximation algorithm as in \Cref{thm:mls} to create an approximation $f_n$ of $f$ using $n$ (deterministic) function evaluations.
\item Run algorithm $A$ on $N = \Theta_{m, d}(n^\gamma)$ function evaluations of $f_n$, $\gamma \in (0, \infty)$. 
\end{enumerate}
By \Cref{thm:mls}, we have $\|f_n - f\|_\infty \leq O_{m, d}(\|f\|_{C^m}n^{-m/d})$, and by \Cref{prop:app_bounds}, this rate also applies to the considered distances of $L_{f_n}$ to $L_f$ or $P_{f_n}$ to $P_f$. By the triangle inequality, the resulting algorithm has a convergence rate of
\begin{IEEEeqnarray*}{+rCl+x*}
O_{m, d}(\|f\|_{C^m} n^{-m/d} + \|f_n\|_{C^m} N^{-\alpha_{m, d}}) = O_{m, d}(\|f\|_{C^m} n^{-\min\{m/d, \gamma \alpha_{m, d}\}})~,
\end{IEEEeqnarray*}
where we used $\|f_n\|_{C^m} \leq O_{m, d}(\|f\|_{C^m})$ due to \Cref{thm:mls} (a) with $k=m$. The runtime complexity of this algorithm is
\begin{IEEEeqnarray*}{+rCl+x*}
&& O_{m, d}(n^{\gamma \beta_{m, d}})~. %
\end{IEEEeqnarray*}
\end{example}

While the construction in \Cref{ex:tradeoff_more_points} also does not improve the runtime complexity needed to reach a desired error level, it can still be useful if evaluations of the approximant (or surrogate model) $f_n$ are much cheaper than evaluations of $f$. This principle is used, for example, in computational chemistry, where expensive direct simulations $f$ are approximated with machine-learned interatomic potentials $f_n$ \citep{deringer_machine_2019}.

\subsection{Relation between Stochastic and Deterministic Evaluation Points}

When the construction in \Cref{ex:tradeoff_more_points} is applied to a sampling algorithm with stochastic evaluation points, it yields a sampling algorithm with deterministic evaluation points. This can be advantageous since the latter only needs $n$ function evaluations to draw an arbitrary number of samples, while the former may require $n$ new function evaluations for every drawn sample. On the other hand, this construction limits the convergence rate of the sampling algorithm to $\Omega_{m, d}(Bn^{-m/d})$, a rate which can be improved by sampling algorithms with stochastic evaluation points outside of the optimization regime (cf.\ \Cref{thm:upper_stoch_sampling}). 

Applying the construction in \Cref{ex:tradeoff_more_points} to a log-partition algorithm with stochastic evaluation points yields a stochastic log-partition algorithm with deterministic evaluation points. We did not consider such algorithms separately in \Cref{sec:ibc:stochastic_points}. However, such an algorithm is never better than its median or expected output, which is a deterministic log-partition method with deterministic evaluation points. Hence, it follows from \Cref{thm:deterministic_points_rates} that the convergence rate of the construction in \Cref{ex:tradeoff_more_points} is limited to $\Omega_{m, d}(Bn^{-m/d})$, and this rate can be improved by log-partition algorithms with stochastic evaluation points outside of the optimization regime (cf.\ \Cref{thm:upper_stoch_log}).

\subsection{Relation Between Sampling and Log-partition Estimation} \label{sec:relations:sampling_log}

A natural question is whether efficient sampling algorithms can be used to obtain efficient log-partition estimators and vice versa. We study both of these directions in the following. In fact, sampling algorithms are frequently employed for log-partition estimation in computational statistical physics and other fields \citep{frenkel_understanding_2001, friel_estimating_2012}. One method to achieve this is thermodynamic integration \citep{kirkwood_statistical_1935}, of which we present a particularly simple version here. By integrating the derivative of $L(\beta) \equalDef L_{\beta f}$, it is possible to derive the following formula \citep{gelman_simulating_1998, friel_estimating_2012}:
\begin{IEEEeqnarray*}{+rCl+x*}
L_f & = & \int_0^1 \bbE_{x \sim P_{\beta f}}[f(x)] \diff \beta = \bbE_{\beta \sim \calU([0, 1])} \bbE_{x \sim P_{\beta f}} f(x)~.
\end{IEEEeqnarray*}
Thermodynamic integration can be used more generally to estimate a difference $L_f - L_g$ by integrating along a path between $f$ and $g$. In practice, the inner expectation is typically evaluated by Monte Carlo methods using sampling algorithms to sample from $P_{\beta f}$, while the outer integral is typically approximated with a suitable (deterministic) quadrature rule. For convenience of analysis, we will consider the case where both expectations are approximated using Monte Carlo quadrature: 

\begin{restatable}[Convergence of thermodynamic integration]{theorem}{thmThermodynamicIntegration} \label{thm:thermodynamic_integration}
Given $N \in \bbN_{\geq 1}$ and a sampling algorithm producing samples from approximate distributions $\tilde P_{\beta f}$, consider the following algorithm:
\begin{itemize}
\item Sample $\beta_1, \hdots, \beta_N \sim \calU([0, 1])$ independently.
\item Draw $X_i \sim \tilde P_{\beta_i f}$ independently.
\item Output $\tilde L_f \equalDef \frac{1}{N} \sum_{i=1}^N f(X_i)$.
\end{itemize}
Then, for $\delta > 0$, we have
\begin{IEEEeqnarray*}{+rCl+x*}
|L_f - \tilde L_f| \leq |L_f - \bbE \tilde L_f| + 2\|f\|_\infty \sqrt{\frac{\log(2/\delta)}{2N}}
\end{IEEEeqnarray*}
with probability $\geq 1-\delta$, where
\begin{IEEEeqnarray*}{+rCl+x*}
|L_f - \bbE \tilde L_f| & \leq & 2 \|f\|_\infty \sup_{\beta \in [0, 1]} \DTV(P_{\beta f}, \tilde P_{\beta f}), \\
|L_f - \bbE \tilde L_f| & \leq & |f|_1 \sup_{\beta \in [0, 1]} \DW(P_{\beta f}, \tilde P_{\beta f}).
\end{IEEEeqnarray*}
\end{restatable}

\Cref{thm:thermodynamic_integration} is proven in \Cref{sec:appendix:relations:sampling_logpartition}. In the upper bounds above, we obtain additional factors $\|f\|_\infty$ or $|f|_1$, which deteriorate the convergence rate. While it appears that these factors are in general necessary for the TV and 1-Wasserstein distances, we explain in \Cref{rem:ti_suplog} that better bounds in terms of $\Dsuplog$ seem plausible but appear to be more difficult to prove. When considering the runtime complexity and convergence rate of the construction in \Cref{thm:thermodynamic_integration}, it is important to set them in relation to the total number $n$ of function evaluations used. For example, if sampling from $P_{\beta_i f}$ uses $\tilde n$ function evaluations, then in general $n = (\tilde n + 1)N$. If the employed sampling algorithm is non-adaptive with deterministic evaluation points, we only need $n = \tilde n + N$ function evaluations. Still, due to the Monte Carlo nature of thermodynamic integration, the convergence rate is at least limited to $\Omega_{m, d, f}(n^{-1/2})$, which is not optimal as we showed in \Cref{thm:upper_stoch_log}. Of course, thermodynamic integration can be performed on top of an approximation of $f$ instead, similar to \Cref{ex:tradeoff_more_points}.

\begin{algorithm}[htb]
\caption{Bisection sampling algorithm using a log-partition algorithm $\tilde L$.} \label{alg:bisection_sampling}
\begin{algorithmic}
\Function{BisectionSampling}{Function $f: \calX \to \bbR^d$, Log-partition algorithm $\tilde L$, Number $M \in \bbN_0$ of bisection steps per dimension}
	\State For a hyperrectangle $\calZ = \bigtimes_{i=1}^d [z_i, z_i + h_i]$, define $f_{\calZ}: \calX \to \bbR$ by $f_{\calZ}(x) \equalDef f(z_1 + h_1 x_1, \hdots, z_d + h_d x_d)$
	\State $\calZ \assign \calX$
	\For{$i$ from $1$ to $M$}
		\For{$j$ from $1$ to $d$}
			\State Split $\calZ$ along dimension $j$ into two equal-sized hyperrectangles $\calZ_1$ and $\calZ_2$
			\State Compute $p_1 \equalDef \sigma(\tilde L_{f_{\calZ_1}} - \tilde L_{f_{\calZ_2}})$, where $\sigma(u) = (1 + \exp(-u))^{-1}$ is the sigmoid function
			\State Sample $k=1$ with probability $p_1$ and $k=2$ otherwise
			\State $\calZ \assign \calZ_k$
		\EndFor
	\EndFor
	\State $\xi \assign$ sample from the uniform distribution $\calU(\calZ)$
	\State \Return $\xi$
\EndFunction
\end{algorithmic}
\end{algorithm}

Now, we ask the converse question: Can an efficient log-partition algorithm be used for efficient sampling? To achieve such a reduction, we note that we can apply a log-partition algorithm not only to the target function $f$ but also, for example, to multiple shifted and rescaled versions of $f$, which amounts to computing the log-partition function on subsets of the cube $\calX$. This is exploited in \Cref{alg:bisection_sampling}, which we refer to as bisection sampling. Bisection sampling has been studied, for example, by \cite{marteau-ferey_sampling_2022}. We give an upper bound on its error in the sup-log distance:

\begin{restatable}[Convergence of bisection sampling]{theorem}{thmSampByLog} \label{thm:sampling_by_logpartition}
Let $m \geq 1, B \geq 0$ and $M \in \bbN_0$. Let $f \in \calF_{d, m, B}$ and let $\tilde L$ be a log-partition estimator with worst-case error $E \geq 0$ on $\calF_{d, m, B}$. Let $f \in C^m(\calX)$ and let $\tilde P_f$ be the distribution of samples produced by $\textsc{BisectionSampling}(f, \tilde L, M)$ in \Cref{alg:bisection_sampling}. Then,
\begin{IEEEeqnarray*}{+rCl+x*}
\Dsuplog(P_f, \tilde P_f) & \leq & 2MdE + 2^{-M} d \|f\|_{C^1}~.
\end{IEEEeqnarray*}
\end{restatable}

Of course, \Cref{thm:sampling_by_logpartition}, which is proven in \Cref{sec:appendix:relations:sampling_logpartition}, also implies bounds on the TV and 1-Wasserstein distances using \Cref{prop:app_bounds}.
The first term in the upper bound grows with $M$, which stems from the possibility of making an error of order $2E$ per loop iteration. However, when the resulting error decays quickly enough in the loop, it is possible to make the first term independent of $M$. For example, this could arise because the log-partition algorithm achieves smaller errors for smoother functions. It is also possible if we consider the 1-Wasserstein distance, which provides better error bounds on smaller hyperrectangles. 

To analyze the resulting convergence rates, suppose that the log-partition algorithm $\tilde L$ uses $N$ evaluation points. Ignoring rounding issues, we can set $M = \log_2(N^{m/d})$ and obtain the rate
\begin{IEEEeqnarray*}{+rCl+x*}
\Dsuplog(P_f, \tilde P_f) & \leq & O_{m, d}(E \log(N) + \|f\|_{C^1} N^{-m/d})~,
\end{IEEEeqnarray*}
where \textsc{BisectionSampling} uses up to $n \equalDef 2MdN = O_{m, d}(N\log(N))$ function evaluations. Hence, we typically only lose polylogarithmic terms in the convergence rate, unlike for thermodynamic integration. Even for log-partition algorithms with deterministic evaluation points, the resulting sampling algorithm uses stochastic evaluation points. Again, bisection sampling can be performed on top of an approximation of $f$ instead, similar to \Cref{ex:tradeoff_more_points}.

\subsection{Relation to Optimization} \label{sec:relations:optimization}

Due to the relationship between sampling and optimization, a natural question is in which sense approximate sampling algorithms can perform approximate optimization. Actually, we can consider two kinds of optimization problems, similar to \cite{novak_deterministic_1988}:
\begin{enumerate}[(OPT), wide=0pt]
\item[(OPT)] The problem of outputting $x \in \calX$ such that $|M_f - f(x)|$ is small can be seen as the low-temperature limit of the sampling problem.
\item[(OPT${}^*$)] The problem of outputting an estimate $\tilde M_f$ such that $|M_f - \tilde M_f|$ is small can be seen as the low-temperature limit of the log-partition problem.
\end{enumerate}
As special cases of the reductions between sampling and log-partition estimation in \Cref{sec:relations:sampling_log}, we can obtain reductions between (OPT) and (OPT${}^*$): For (OPT${}^*$), we can simply evaluate $f$ at the estimate $x$ obtained from (OPT), which can be seen as a simple special case of thermodynamic integration. On the other hand, for (OPT), we can recursively use (OPT${}^*$) to see whether the optimum is contained in a subdomain of $\calX$, which corresponds to the low-temperature limit of bisection sampling.

To obtain a bound for approximate (OPT${}^*$) via approximate log-partition estimation $\tilde L$, we note that \Cref{lemma:lipschitz_maximization_bound} directly yields
\begin{IEEEeqnarray*}{+rCl+x*}
|M_f - \eps \tilde L_{f/\eps}| & \leq & |M_f - \eps L_{f/\eps}| + \eps |L_{f/\eps} - \tilde L_{f/\eps}| \\
& \leq & \eps d\log(1+3d^{-1/2}\eps |f|_1) + \eps |L_{f/\eps} - \tilde L_{f/\eps}| \IEEEyesnumber \label{eq:opt_by_apx_logp}
\end{IEEEeqnarray*}
for temperatures $\eps > 0$.

When performing approximate (OPT) via sampling from an approximate distribution $Q = \tilde P_f$, the result depends on the employed distance metric. Since the result of sampling is stochastic, we will upper-bound probabilities of the form $Q(\{x \in \calX \mid f(x) \leq \alpha\})$ of obtaining a function value $f(x) \leq \alpha$ when drawing $x$ from $Q$.

\begin{restatable}[Optimization by approximate sampling]{proposition}{propOptByApxSampling} \label{prop:opt_by_apx_sampling}
Let $Q$ be a probability distribution on $\calX$. Then, for any $\delta \in (0, 1]$ and $\eps > 0$, %
\begin{enumerate}[(a)]
\item $Q(\{x \in \calX \mid f(x) \leq \eps L_{f/\eps} - \eps \log(1/\delta) - \eps \Dsuplog(P_{f/\eps}, Q)\}) \leq \delta$,
\item $Q(\{x \in \calX \mid f(x) \leq \eps L_{f/\eps} - \eps \log(1/\delta)\}) \leq \delta + \DTV(P_{f/\eps}, Q)$,
\item $Q(\{x \in \calX \mid f(x) < \eps L_{f/\eps} - \eps \log(2/\delta) - 2\delta^{-1} |f|_1 \DW(P_{f/\eps}, Q)\}) \leq \delta$.
\end{enumerate}
\end{restatable}

\Cref{prop:opt_by_apx_sampling} is proven in \Cref{sec:appendix:relations:optimization}. If $Q = P_{g/\eps}$ for some bounded $g: \calX \to \bbR$, the bound in $(a)$ recovers the known optimization bound $f(\argmax g) \geq M_f - 2\|f - g\|_\infty$ in the limit $\varepsilon \searrow 0$ using \Cref{lemma:lipschitz_maximization_bound} and \Cref{prop:app_bounds}. We can deduce from (a) that a sampling algorithm achieving the optimal rate $O_{m, d}(\|f\|_{C^m} n^{-m/d})$ in terms of $\Dsuplog$ can be used (with sufficiently small $\eps$) to achieve the optimal rate for (OPT) as well. On the other hand, the bounds (b) and (c) are much weaker.
For example, (a) still gives a good bound for $\Dsuplog(P_{f/\eps}, Q) = 1/2$, but (b) only gives a low-probability bound for $\DTV(P_{f/\eps}, Q) = 1/2$ and (c) is trivial for $\DW(P_{f/\eps}, Q) = 1/2$. Note that an argument analogous to (b) has been used in Corollary 1 by \cite{ma_sampling_2019} to analyze the convergence of Langevin algorithms for approximate optimization.

The reductions presented in this section are summarized in \Cref{table:reductions}.

\begin{table}
\centering
\begin{tabular}{ccccc}
\toprule
Method & Old problem & New problem & New conv.\ rate & New runtime \\
\midrule
Dummy eval.\ & Any & same & $O(Bn^{-\gamma \alpha}), \gamma \leq 1$ & $O(n^{\gamma \beta})$ \\
Interpolant & Any & same & $O(Bn^{-\min\{m/d,\gamma \alpha\}})$ & $O(n^{\gamma \beta})$ \\ 
Therm.\ integr. & $P_f$ & $L_f$ & $O(Bn^{-\alpha/(2\alpha+1)})$ & $O(n^{\frac{2\alpha + \beta}{2\alpha + 1}})$ \\
Bisec.\ sampl. & $L_f$ & $P_f$ & $\tilde O(Bn^{-\alpha})$ & $\tilde O(n^\beta)$ \\
Direct & $L_f$ & $\max f$ & $O(Bn^{-\alpha})$ & $O(n^\beta)$ \\
Direct & $P_f$ & $\argmax f$ & $O(Bn^{-\alpha})$ & $O(n^\beta)$ \\
\bottomrule
\end{tabular}
\vspace{0.05in}
\caption{Reductions using a given algorithm (\quot{old}) with convergence rate $O(Bn^{-\alpha})$ and runtime $O(n^\beta)$ to obtain a solution to a different (\quot{new}) problem. For thermodynamic integration, we use $N = \Theta(n^{2\alpha/(2\alpha+1)})$ samples to get the best possible bound from \Cref{thm:thermodynamic_integration}. We use $\tilde O$ notation for bisection sampling to ignore terms of the form $\log(n)^\gamma$. All sampling bounds hold for $\Dsuplog$ but sometimes also for weaker metrics. The bounds below are from \Cref{ex:tradeoff_less_points}, \Cref{ex:tradeoff_more_points}, \Cref{thm:thermodynamic_integration}, \Cref{thm:sampling_by_logpartition}, \Eqref{eq:opt_by_apx_logp}, and \Cref{prop:opt_by_apx_sampling}.} \label{table:reductions}
\end{table}

\section{Algorithms} \label{sec:algorithms}

In this section, we investigate the convergence rates of different algorithmic approaches toward the sampling and log-partition problems. 

\subsection{Approximation-based Algorithms}

First, we study approximation-based algorithms. In \Cref{sec:ibc:deterministic_points}, we have seen that, in principle, approximation-based methods can achieve the optimal rates for the sampling and log-partition problems with deterministic points. However, for most approximations $g$, it is unclear how to sample from $P_g$ or compute $L_g$. In the following, we will consider a few cases where this is possible:

\subsubsection{Piecewise Constant Approximation}  \label{sec:pc_approx}

A very simple approximation method is piecewise constant approximation. Here, we study the convenient setting where $n = N^d$ for some $N \in \bbN$:
\begin{itemize}
\item Divide $\calX$ into $N^d$ equally-sized cubes $\calX_1, \hdots, \calX_n$ by dividing $[0, 1]$ into $N$ intervals.
\item Output the function $g_{f, n}$ that is piecewise constant on each cube and interpolates $f$ at the center $x^{(i)}$ of the cube~$\calX_i$. Boundary points can be assigned to an arbitrary adjacent cube.
\end{itemize}
Given a piecewise constant function $g_{f, n}$, we can easily compute $L_{g_{f, n}} = \log\left(\frac{1}{n} \sum_{i=1}^n e^{f(x^{(i)})}\right)$ in time $O_{m, d}(n)$. 
After an $O_{m, d}(n)$ preprocessing step \citep[see][]{vose_linear_1991}, it is even possible to sample from $g_{f, n}$ in time $O_{m, d}(1)$ in suitable computation models: First, sample a subcube $\calX_i$ with probability $p_i = e^{f(x^{(i)}) - L_{g_{f, n}}}$ using the method of \cite{vose_linear_1991}; then, draw a uniform random sample from $\calX_i$.
However, the convergence rate of the piecewise constant approximation is bad, as we prove in \Cref{sec:appendix:pc_approx}:

\begin{restatable}[Convergence rate of piecewise constant approximation]{theorem}{thmPCApprox} \label{thm:pc_approx}
Let $m \geq 1$ and $n = N^d$ as above. If $g_{f, n}$ is a piecewise constant interpolant as above, we have
\begin{IEEEeqnarray*}{+rCl+x*}
\sup_{f \in \calF_{d, m, B}} |L_f - L_{g_{f, n}}| & = & \begin{cases}
\Theta_{m, d}(Bn^{-1/d}) &, \text{ if $m=1$ or $Bn^{-1/d} > 1$} \\
\Theta_{m, d}(\max\{B, B^2\} n^{-2/d}) &, \text{ otherwise.}
\end{cases}  \\
\sup_{f \in \calF_{d, m, B}} \Dsuplog(P_f, P_{g_{f, n}}) & = & \Theta_{m, d}(Bn^{-1/d})~.
\end{IEEEeqnarray*}
\end{restatable}

The rates of piecewise constant approximation are thus optimal for $m=1$, but not for $m>1$. For $m>1$, using a combination with higher-order function approximation as in \Cref{ex:tradeoff_more_points}, it is possible to achieve the rate $O_{m, d}(Bn^{-m/d})$ with runtime $O_{m, d}(n^m)$. 
The result above also shows that the piecewise constant log-partition method can achieve the faster rate $O(Bn^{-2/d})$ of midpoint quadrature only outside of the optimization regime. We leave it as an open problem whether such faster rates are also achieved for sampling with $\DTV$ or $\DW$. 
\cite{achddou_minimax_2019} analyze a combination of piecewise constant approximation with rejection sampling, but in a setting incomparable to ours. They also note that piecewise constant approximation achieves optimal rates for Hölder classes of functions. 

Beyond piecewise constant approximations, piecewise linear approximations also allow for efficient sampling and log-partition estimation, and they should allow achieving convergence rates of $O_{m, d}(Bn^{-2/d})$. We leave a precise analysis of this approach as an open problem.

\subsubsection{Density-based Approximation} 

Another option to obtain tractable sampling and log-partition algorithms is to directly approximate the unnormalized density $p(x) = e^{f(x)}$. Since probability distributions are normalized, approximating $\lambda p$ with $\lambda q$ yields the same sampling and log-partition errors as approximating $p$ with $q$, but the approximation error $\|\lambda p - \lambda q\|_\infty$ depends on $\lambda > 0$. To obtain a scale-invariant bound for the sampling and log-partition errors, we need to divide the approximation bound by a normalization constant:

\begin{restatable}[Density approximation bounds]{proposition}{PropDensityApxBound} \label{prop:density_apx_bound}
Let $p, q: \bbR \to [0, \infty)$ be bounded and measurable such that $I_p, I_q > 0$, where $I_p \equalDef \int_{\calX} p(x) \diff x$. Define probability distributions $P, Q$ with densities $p/I_p$ and $q/I_q$, respectively. Then,
\begin{IEEEeqnarray*}{+rCl+x*}
|\log I_p - \log I_q| & \leq & \log \left(\frac{1}{1 - \|p - q\|_\infty /I_p}\right) \quad \text{if $\|p - q\|_\infty < I_p$}, \\
\DTV(P, Q) & \leq & \frac{\|p - q\|_\infty}{\max\{I_p, I_q\}} \leq \frac{\|p - q\|_\infty}{I_p}~.
\end{IEEEeqnarray*}
\end{restatable}

We prove \Cref{prop:density_apx_bound} in \Cref{sec:appendix:density_based}. While $\|p\|_\infty = \|e^f\|_\infty$ can be exponential in $\|f\|_\infty$, \Cref{prop:density_apx_bound} demonstrates that we need to incorporate the normalization constant to obtain reasonable estimates for sampling and log-partition computation. After incorporating the normalization constant, we arrive at terms of the form $\|p\|/I_p = \|e^f\|/Z_f = \|e^f/Z_f\| = \|p_f\|$. Hence, the norm of the (normalized) density plays an important role for convergence rates of density-based approximation approaches. As it turns out, $\|p_f\|_{C^m}$ does not scale exponentially in $\|f\|_{C^m}$ for $m \geq 1$, but still badly:\footnote{The assumption $m \geq 1$ is necessary: Define $f_{a,b}(x) = ae^{-bx}$. Then $\|f_{a, b}\|_{C^0} = a$ but $\lim_{b \to \infty} \|p_{f_{a, b}}\|_{C^0} = \exp(a)$.}

\begin{restatable}[Density norm]{theorem}{thmDensityNorm} \label{thm:density_norm}
For $m \geq 1$, we have
\begin{IEEEeqnarray*}{+rCl+x*}
\sup_{f \in \calF_{d, m, B}} \|p_f\|_{C^m} = \Theta_{m, d}\left(\max\{1, B\}^{m+d}\right)
\end{IEEEeqnarray*}
and this asymptotic rate is attained by $f_{d, m, B}(x) = Bd^{-1}(x_1 + \cdots + x_d)$.
\end{restatable}

\Cref{thm:density_norm} is proven in \Cref{sec:appendix:density_based}. Suppose that $f \in C^m(\calX)$ and we can approximate $p(x) = e^{f(x)}$ with a non-negative function $q$ with rate $O_{m, d}(\|p\|_{C^m} n^{-m/d})$, which is worst-case optimal if we only know $\|p\|_{C^m}$ and forget that $p = e^f$ with small $\|f\|_{C^m}$.
By combining \Cref{prop:density_apx_bound} and \Cref{thm:density_norm}, the distribution $Q$ associated with the unnormalized density $q$ then satisfies
\begin{IEEEeqnarray*}{+rCl+x*}
\DTV(P_f, Q) \leq O_{m, d}(\max\{1, \|f\|_{C^m}\}^{m+d} n^{-m/d})~. \IEEEyesnumber \label{eq:density_apx_rate}
\end{IEEEeqnarray*}
Although this rate is optimal in terms of $n$ for deterministic evaluation points, it is bad in terms of $\|f\|_{C^m}$, cf.\ also \Cref{ex:bayes_runtime}.

\cite{marteau-ferey_sampling_2022} propose a sampling algorithm based on approximating the density with a (non-negative) sum-of-squares model. Specifically, for a Gibbs distribution, they suggest approximating $\sqrt{p}$ with $q$ and then using $q^2$ as an unnormalized density. They achieve a rate of $O_{m, d, f}(n^{-m/d})$ in polynomial time without explicitly stating the dependence on $\|f\|$, but we conjecture that the dependence on $\|f\|$ is similar to \Eqref{eq:density_apx_rate}.

\subsection{Simple Stochastic Algorithms}

We now analyze the convergence rates for some simple stochastic algorithms.

\subsubsection{Rejection Sampling With Uniform Proposal Distribution} \label{sec:rs_uniform}

A simple stochastic algorithm is rejection sampling with a uniform proposal distribution. The following proposition shows that this can achieve better rates in terms of the TV distance than the density-based approximation rates in \Eqref{eq:density_apx_rate} if the maximum $M_f$ of $f$ is known:

\begin{restatable}[Convergence of rejection sampling]{proposition}{propRejectionSampling} \label{prop:rejection_sampling}
Let $m \geq 1$ and let $f \in C^1(\calX)$. Then, the distribution $\tilde P_f$ produced by \textsc{RejectionSampling}($f$, $M_f$, $n$) (see \Cref{alg:rejection_sampling}) satisfies
\begin{IEEEeqnarray*}{+rCl+x*}
\Dsuplog(P_f, \tilde P_f) & \leq & \min\left\{2\|f\|_\infty, \exp\left(2\|f\|_\infty - n/\|p_f\|_\infty\right)\right\} \\
\DTV(P_f, \tilde P_f) & \leq & \min\{1, 2\|f\|_\infty\} \exp(-n/\|p_f\|_\infty)  \\
& \leq & O_{m, d}(\min\{1, \|f\|_\infty\} \max\{1, \|f\|_{C^1}\}^m n^{-m/d})~.
\end{IEEEeqnarray*}
\end{restatable}

A proof can be found in \Cref{sec:appendix:rs}. Lower bounds for the convergence of rejection sampling could be obtained using \Cref{lemma:rejection_sampling}, but the resulting formula would not be easy to interpret. In any case, an argument similar to the one in \Cref{sec:mc_sampling} and \Cref{sec:appendix:mc_sampling} can be made to show that \textsc{RejectionSampling}($f$, $M_f$, $n$) cannot achieve the rate $O_{m, d}(Bn^{-m/d})$.
We leave it as an open question whether similar rates to \Cref{prop:rejection_sampling} can be achieved when $M_f$ is approximately known or when a guess for $M_f$ is used that slowly increases with $n$. Note that \cite{talwar_computational_2019} studies a similar setting where rejection sampling is not stopped after $n$ rejections.

\subsubsection{Monte Carlo Log-partition} \label{sec:mc_logpartition}

Since the log-partition problem involves an integral, it is natural to approximate the integral by Monte Carlo (MC) quadrature. The following theorem gives an upper bound on the convergence rate:

\begin{restatable}[Upper bounds for MC log-partition]{theorem}{thmMCLogPartition} \label{thm:mc_log_partition}
Let $f: \calX \to \bbR$ be Lipschitz, let $X_1, \hdots, X_n \sim \calU(\calX)$ be independent and let
\begin{IEEEeqnarray*}{+rCl+x*}
\tilde L_n \equalDef \log S_n, \qquad S_n \equalDef \frac{1}{n} \sum_{i=1}^n \exp(f(X_i)).
\end{IEEEeqnarray*}
Then, for any $\delta \in (0, 1]$, the following convergence rates hold:
\begin{enumerate}[(a)]
\item \textbf{Optimization regime:} If $n \leq 4\log(2/\delta)(1+3d^{-1/2}|f|_1)^d$, we have
\begin{IEEEeqnarray*}{+rCl+x*}
|\tilde L_n - L_f| \leq d^{1/2} (\log(1/\delta))^{1/d} |f|_1 n^{-1/d} + \log(4\log(2/\delta)) + d\log(1+3d^{-1/2}|f|_1)
\end{IEEEeqnarray*} 
with probability $\geq 1 - \delta$.
\item \textbf{Quadrature regime:} If $n \geq 4\log(2/\delta)(1+3d^{-1/2}|f|_1)^d$, we have
\begin{IEEEeqnarray*}{+rCl+x*}
|\tilde L_n - L_f| \leq 4 \log(2/\delta)^{1/2} (1+3d^{-1/2}|f|_1)^{d/2} n^{-1/2}
\end{IEEEeqnarray*}
with probability $\geq 1 - \delta$.
\end{enumerate}
\end{restatable}

We prove \Cref{thm:mc_log_partition} in \Cref{sec:appendix:mc_logpartition}. Roughly speaking, the rates in the theorem above behave like $|f|_1 n^{-1/d}$ until an error of $O(1)$ is reached, and then they change to $|f|_1^{d/2} n^{-1/2}$. Intuitively, the log-partition $L_f$ is quite close to the maximum $M_f$. %
Initially, the behavior of MC log-partition is characterized quite well by the error $|M_f - \max_{i \leq n} f(X_i)|$ of MC optimization. Once points close to the maximum are reached, MC log-partition behaves more like MC quadrature, because the average in $S_n$ is not dominated by a single point. 
\Cref{fig:mc_log_partition} and our experiments later in \Cref{fig:logpartition} show that this reflects the qualitative behavior of the error on linear $f$ in practice. For more general $C^2$ functions and large $B$, we expect the convergence rates to depend on the behavior around the optimum: Similar rates should be observed when the maximum is attained in a corner of $[0, 1]^d$, whereas faster convergence rates for the optimization regime should be possible when the maximum is attained in the interior of $[0, 1]^d$ due to the local quadratic behavior around the maximum.

\begin{figure}[tb]
\centering
\includegraphics[scale=1.0]{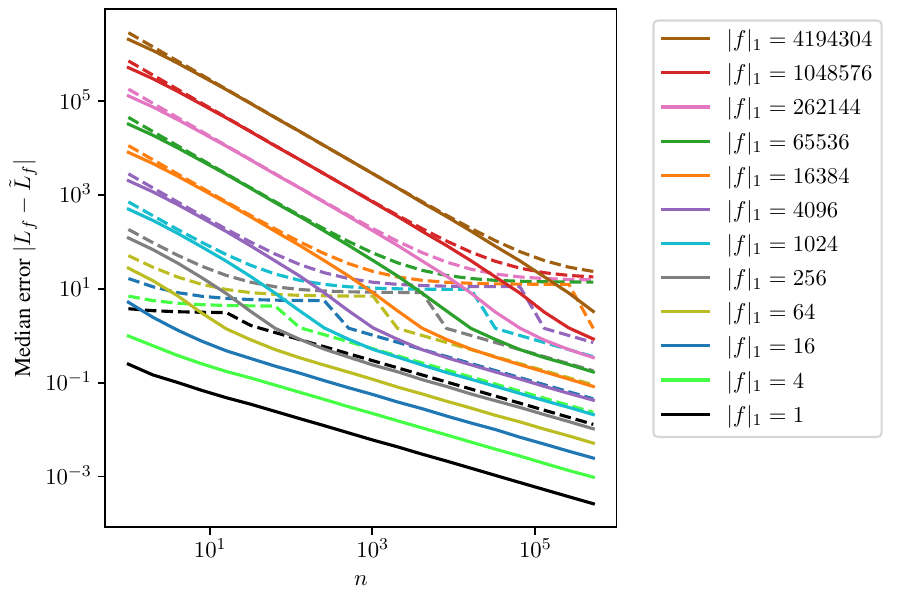}
\caption{Median error of MC log-partition for the function $f: [0, 1] \to \bbR, x \mapsto \beta x$ (with $d=1$), for varying numbers of points $n$ and values of $\beta = |f|_1 > 0$. Medians were computed out of $10001$ repetitions. The dashed lines show the corresponding upper bounds from \Cref{thm:mc_log_partition} for the median ($\delta = 1/2$).} \label{fig:mc_log_partition}
\end{figure}

\subsubsection{Monte Carlo Sampling} \label{sec:mc_sampling}

We can also consider a sampling version of the Monte Carlo log-partition method considered in \Cref{sec:mc_logpartition}. The following theorem shows that it cannot achieve good rates in the optimization regime either:

\begin{restatable}[Lower bound for MC sampling]{theorem}{thmMCSampling} \label{thm:mc_sampling}
Let $f: \calX \to \bbR$ be bounded and measurable. Let $X_1, \hdots, X_n \sim \calU(\calX)$ and let the random index $I \in \{1, \hdots, n\}$ be distributed as
\begin{IEEEeqnarray*}{+rCl+x*}
P(I=i) = \frac{\exp(f(X_i))}{\sum_{j=1}^n \exp(f(X_j))}.
\end{IEEEeqnarray*}
Consider the distribution $\tilde P_f$ of the random sample $X_I$. Then, for all $B > 0$ and $n \geq 1$ with $Bn^{-1/d} \geq 4d\log(4d)$,
\begin{IEEEeqnarray*}{+rCl+x*}
\sup_{f \in \calF_{d, m, B}} \DTV(P_f, \tilde P_f) & \geq & \frac{1}{2}~.
\end{IEEEeqnarray*}
\end{restatable}

The lower bounds in \Cref{thm:mc_sampling}, proven in \Cref{sec:appendix:mc_sampling}, show that $n \geq \Omega_{m, d}(B^d)$ points are required to achieve an error below $O(1)$, which is significantly worse than the around $\Theta_{m, d}(B^{d/m})$ points required by a method with the optimal rate for deterministic evaluation points. The proof only uses that the density $\tilde p_f$ is upper-bounded by $n$, and would apply analogously (using $n+1$ instead of $n$) to rejection sampling with uniform proposal distribution as considered in \Cref{sec:rs_uniform}.

\subsection{Markov Chain Monte Carlo Algorithms}

Markov Chain Monte Carlo (MCMC) methods are a very popular class of sampling algorithms. In particular, gradient-based MCMC algorithms such as versions of Langevin MCMC and Hamiltonian Monte Carlo \citep{duane_hybrid_1987} have been studied intensively in recent years. While most theoretical guarantees only consider the case of concave~$f$, there have been a few extensions where $f$ is allowed to be non-concave in a compact region of the domain. For example, \cite{ma_sampling_2019} study a certain class of functions whose gradient is $L$-Lipschitz and that are non-concave in a region with radius~$R$ but $\alpha$-strongly convex outside of it. For the Metropolis-adjusted Langevin algorithm (MALA) to reach a TV distance error $E > 0$, they obtain the mixing time bound
\begin{IEEEeqnarray*}{+rCl+x*}
n \leq O\left( \frac{e^{40LR^2}}{\alpha} (L/\alpha)^{3/2} d^{1/2} (d\ln(L/\alpha) + \ln(1/E))^{3/2} \right)~. \IEEEyesnumber \label{eq:mala_upper_bound}
\end{IEEEeqnarray*}
We used $n$ here for the mixing time since it corresponds to the number of gradient evaluations, which are potentially more informative than the function evaluations normally allowed in our setting but can be approximated using $d+1$ function evaluations. The dependence of the upper bound in \Eqref{eq:mala_upper_bound} on $L$, which is related to $\|f\|_{C^2}$ in our setting, is exponential. We are not aware of a lower bound, but conjecture that a tight lower bound will also have an exponential dependence on $\|f\|$ in some fashion. This indicates that Langevin MCMC could perform worse than rejection sampling in our setting.

Beyond Langevin MCMC, there are many other popular MCMC methods, such as variants of Hamiltonian Monte Carlo, parallel tempering (or replica exchange MCMC), and simulated tempering. Obtaining convergence rates for these methods on function classes like $\calF_{d, m, B}$ is an interesting problem, but left open in this paper. While \cite{woodard_sufficient_2009} prove torpid (slow) mixing for parallel and simulated tempering in some settings, they show an exponential dependency on $d$ for certain mixtures of Gaussians, which does not appear to imply suboptimal rates in our setting.

\subsection{Variational Formulation for Log-Partition Estimation} \label{sec:variational_formulation}

In the following, we will introduce the variational approach to the log-partition problem by \cite{bach_sum--squares_2025}. We will first start with the simpler optimization setting. Let $\calP(\calX)$ be the space of probability measures on $\calX$. We start with the formulation
\begin{IEEEeqnarray*}{+rCl+x*}
M_f = \max_{x \in \calX} f(x) = \sup_{P \in \calP(\calX)} \int f(x) \diff P(x)~, \IEEEyesnumber \label{eq:measure_opt}
\end{IEEEeqnarray*}
which converts a finite-dimensional non-concave maximization problem into an infinite-dimensional concave maximization problem.
To apply the approach by \cite{bach_sum--squares_2025}, we need to approximate the function $f$ by a model of the form
\begin{IEEEeqnarray*}{+rCl+x*}
g(x) = \varphi(x)^* H \varphi(x),
\end{IEEEeqnarray*}
where $H$ is a Hermitian matrix and $\varphi: \calX \to \bbC^N$ is a suitable feature map. For example, for $d=1$, if $f$ is periodic and we use Fourier features $\varphi(x) = (1, e^{ix}, \hdots, e^{(N-1)ix})^\top$, then $H$ can be determined by trigonometric interpolation, see also \cite{woodworth_non-convex_2022}.

For a probability distribution $P \in \calP(\calX)$, we define the moment matrix
\begin{IEEEeqnarray*}{+rCl+x*}
\Sigma_P \equalDef \int_{\calX} \varphi(x) \varphi(x)^* \diff P(x)~.
\end{IEEEeqnarray*}
Because of 
\begin{IEEEeqnarray*}{+rCl+x*}
\int_{\calX} g(x) \diff P(x) = \int_{\calX} \tr[\varphi(x)^* H \varphi(x)] \diff P(x) = \int_{\calX} \tr[H \varphi(x) \varphi(x)^*] \diff P(x) = \tr[H \Sigma_P]~,
\end{IEEEeqnarray*}
we then obtain 
\begin{IEEEeqnarray*}{+rCl+x*}
M_g = \sup_{P \in \calP(\calX)} \tr[H\Sigma_P] = \sup_{\Sigma \in \calK} \tr[H\Sigma]~, \label{eq:sigma_opt}
\end{IEEEeqnarray*}
where $\calK$ is the (convex) set of all possible values of $\Sigma_P$. This reduces the infinite-dimensional convex optimization problem in \Eqref{eq:measure_opt} to a finite-dimensional convex optimization problem, and at least for certain feature maps, the set $\calK$ has a sufficiently nice structure for optimization.

To extend this approach to the log-partition problem, \cite{bach_sum--squares_2025} uses the following variational formulation by \cite{donsker_asymptotic_1983} for general base distributions $Q$, where $\DKL{P}{Q} = \int \log\left(\frac{\diff P}{\diff Q}\right) \diff P$ is the KL divergence:
\begin{IEEEeqnarray*}{+rCl+x*}
L_f(Q) & \equalDef & \log \int_{\calX} e^{f(x)} \diff Q(x) = \sup_{P \in \calP(\calX)} \int_{\calX} f(x) \diff P(x) - \DKL{P}{Q}~.
\end{IEEEeqnarray*}
Again, after approximating $f$ by $g$, we can replace the integral by $\tr[H\Sigma_P]$. However, to obtain a finite-dimensional optimization problem, we also need to replace the KL divergence with something that only depends on $\Sigma_P$ instead of $P$. Since this is not possible exactly, \cite{bach_sum--squares_2025} proposes multiple lower bounds, of which the tightest one (and most difficult to compute) is
\begin{IEEEeqnarray*}{+rCl+x*}
\Dopt{\Sigma_P}{\Sigma_Q} & \equalDef & \inf_{\tilde P, \tilde Q \in \calP(\calX): \Sigma_P = \Sigma_{\tilde P}, \Sigma_Q = \Sigma_{\tilde Q}} \DKL{\tilde P}{\tilde Q}~.
\end{IEEEeqnarray*}
This yields the following upper bound on the log-partition function:
\begin{IEEEeqnarray*}{+rCl+x*}
\Lopt_g(Q) & \equalDef & \sup_{P \in \calP(\calX)} \int_{\calX} g(x) \diff P(x) - \Dopt{\Sigma_P}{\Sigma_Q} \\
& = & \sup_{P \in \calP(\calX)} \tr[H\Sigma_P] - \Dopt{\Sigma_P}{\Sigma_Q} = \sup_{\Sigma \in \calK} \tr[H\Sigma] - \Dopt{\Sigma}{\Sigma_Q} \\
& \geq & L_g(Q)~. \IEEEyesnumber \label{eq:opt_relaxation}
\end{IEEEeqnarray*}
Our investigation begins here: After inserting the definition of $\Doptsingle$, a simple calculation shows that due to the minus sign, the infimum over $\tilde P$ merges with the supremum over $P$, and the infimum over $\tilde Q$ turns into a supremum:

\begin{restatable}{lemma}{lemLOPT}  \label{lemma:lopt}
For a model of the form $g(x) = \varphi(x)^* H \varphi(x)$ as above, we have
\begin{IEEEeqnarray*}{+rCl+x*}
\Lopt_g(Q) = \sup_{\tilde Q \in \calP(\calX): \Sigma_{\tilde Q} = \Sigma_Q} L_g(\tilde Q)~.
\end{IEEEeqnarray*}
\end{restatable}

This formulation allows us to show a lower bound on the achievable convergence rate. The basic idea is as follows: Since $Q$ is only known through finitely many moments $\Sigma_Q$, we can find a discrete distribution $\tilde Q$ with the same moments. 
We then choose $f$ such that it attains its maximum at one of the discrete points. We conclude that whenever $g$ is a sufficiently good approximation to $f$, the variational method produces an estimate that is close to the maximum of $f$, and therefore not very close to the true log-partition value $L_f$.

\begin{restatable}[Lower bound for OPT relaxation]{theorem}{thmLOPTLowerBound} \label{thm:LOPT_lower_bound}
Let $\varphi: \calX \to \bbC^N$ be continuous. Let 
\begin{IEEEeqnarray*}{+rCl+x*}
n \equalDef \dim_\bbC \Vlin, \qquad \Vlin \equalDef \Span_\bbC \left\{\varphi(x)\varphi(x)^* \mid x \in \calX\right\} \subseteq \bbC^{N \times N}~.
\end{IEEEeqnarray*}
In other words, $n$ is the number of effective degrees of freedom of the model $g(x) = \varphi(x)^* H \varphi(x)$, and hence corresponds to the maximum number of points where such a model can interpolate arbitrary function values. Then, there exists a point $z \in \calX$ depending only on $\varphi$, such that the periodic and analytic function %
\begin{IEEEeqnarray*}{+rCl+x*}
f: \calX \to \bbR, x \mapsto \sum_{i=1}^d \cos(2\pi(x_i - z_i))
\end{IEEEeqnarray*}
satisfies
\begin{IEEEeqnarray*}{+rCl+x*}
|\Lopt_{g}(\calU(\calX)) - L_{\beta f}(\calU(\calX))| \geq \log\left(\frac{\beta^{d/2}}{2n+1}\right) - \|g - \beta f\|_\infty \IEEEyesnumber \label{eq:LOPT_error_bound}
\end{IEEEeqnarray*}
for any model $g(x) = \varphi(x)^* H \varphi(x)$ and any $\beta > 0$.
\end{restatable}

What are the implications of \Cref{thm:LOPT_lower_bound}, which is proven in \Cref{sec:appendix:variational}, on convergence rates? To answer this question, we need to consider the limit $n \to \infty$, which means that $N, \varphi, f, g$ in general depend on $n$, and we will denote them by $N_n, \varphi_n, f_n, g_n$, respectively. We also consider an inverse temperature $\beta_n \equalDef (e(2n+1))^{2/d}$. Since $f_n$ is analytic, an approximation method with optimal rate should achieve the rate $\|g_n - \beta_n f_n\|_\infty \leq O_{m, d}(\|\beta_n f_n\|_{C^m} n^{-m/d}) = O_{m, d}(n^{-(m-2)/d})$ for every $m \in \bbN$. Suppose that this is at least achieved for $m=3$, such that $\lim_{n \to \infty} \|g_n - \beta_n f_n\|_\infty = 0$. Then,
\begin{IEEEeqnarray*}{+rCl+x*}
|\Lopt_{g_n}(\calU([0, 1])) - L_{\beta_n f_n}(\calU([0, 1]))| & \geq & \log\left(\frac{\beta_n^{d/2}}{2n+1}\right) - \|g_n - \beta_n f_n\|_\infty \\
& \geq & 1 - O_{m, d}(n^{-(m-2)/d}) \\
& = & \Omega_{m, d}(\beta_n n^{-2/d}) \text{  for sufficiently large $n$.}
\end{IEEEeqnarray*}
In other words, the approximation error and the log-partition error of the OPT relaxation in \Eqref{eq:opt_relaxation} cannot both achieve a rate strictly better than $O_{m, d}(\|f\|n^{-2/d})$ even for infinitely smooth functions, no matter which (continuous) feature map $\varphi$ is chosen.

\subsection{Summary}

We want to emphasize a few takeaways from our analysis in this section:
\begin{itemize}
\item Direct approximation of $f$ can lead to tractable algorithms in settings that cannot exploit higher orders of smoothness. Applying these algorithms to evaluations of a smooth surrogate (\Cref{ex:tradeoff_more_points}) can lead to faster rates $O_{m, d}(Bn^{-m/d})$ but with slow runtime $O_{m, d}(n^m)$.
\item The approximation of $p_f$ instead leads to tractable algorithms for higher smoothness, which exhibit good rates in $n$ but bad rates in $\|f\|_{C^m}$.
\item Some methods exhibit multiple regimes, like MC integration and piecewise constant approximation for the log-partition problem.
\item Some methods can achieve exponential rates in $n$ but behave badly in $\|f\|_{C^m}$, such as rejection sampling or (upper bounds for) Langevin MCMC, although the situation for MCMC methods requires further study.
\item An attempt to adapt a promising nonconvex optimization method failed to achieve comparable rates because the studied approach is \quot{too close to optimization} in an intermediate regime $\|f\|_{C^m} = \Theta_{m, d}(n^{2/d})$.
\end{itemize}

\section{Experiments} \label{sec:sampling:experiments}

To further investigate the convergence behavior of some simple algorithms, we study them numerically on functions of the form $f: [0, 1]^3 \to \bbR, x \mapsto \beta(x_1 + x_2 + x_3)$. While these functions are simple (and concave), they pose a challenge to some general algorithms as they have a large range in relation to their Lipschitz constant. The dimension $d=3$ has been chosen for visualization purposes, to be able to distinguish the convergence rates $n^{-1/d}$ and $n^{-2/d}$ from the typical MC convergence rate of $n^{-1/2}$. Our plots can be reproduced using the code at
\begin{IEEEeqnarray*}{+rCl+x*}
\text{\url{github.com/dholzmueller/sampling_experiments}}
\end{IEEEeqnarray*}

\subsection{Log-partition Estimation}

For the log-partition problem, we consider the following algorithms:
\begin{itemize}
\item \textbf{PC}: Compute the log-partition function of a piecewise constant approximation as in \Cref{sec:pc_approx}.
\item \textbf{MC}: Monte carlo log-partition estimation as in \Cref{sec:mc_logpartition}.
\item \textbf{PC+MC}: We use importance sampling, specifically MC quadrature on top of a piecewise constant approximation as described in \Cref{sec:ibc:stochastic_points}: We use $n/2$ function evaluations to compute a piecewise constant approximation $g$ of $f$ and then use the other $n/2$ function evaluations for an MC approximation of the right-hand side in
\begin{IEEEeqnarray*}{+rCl+x*}
L_f & = & L_g + \log\left(\bbE_{x \sim P_g} \left[\exp(f(x) - g(x))\right]\right)~.
\end{IEEEeqnarray*}
\end{itemize}
All of the methods above can be implemented in linear time $O_{m, d}(n)$.

\begin{figure}[t]
\centering
\includegraphics[width=0.9\textwidth]{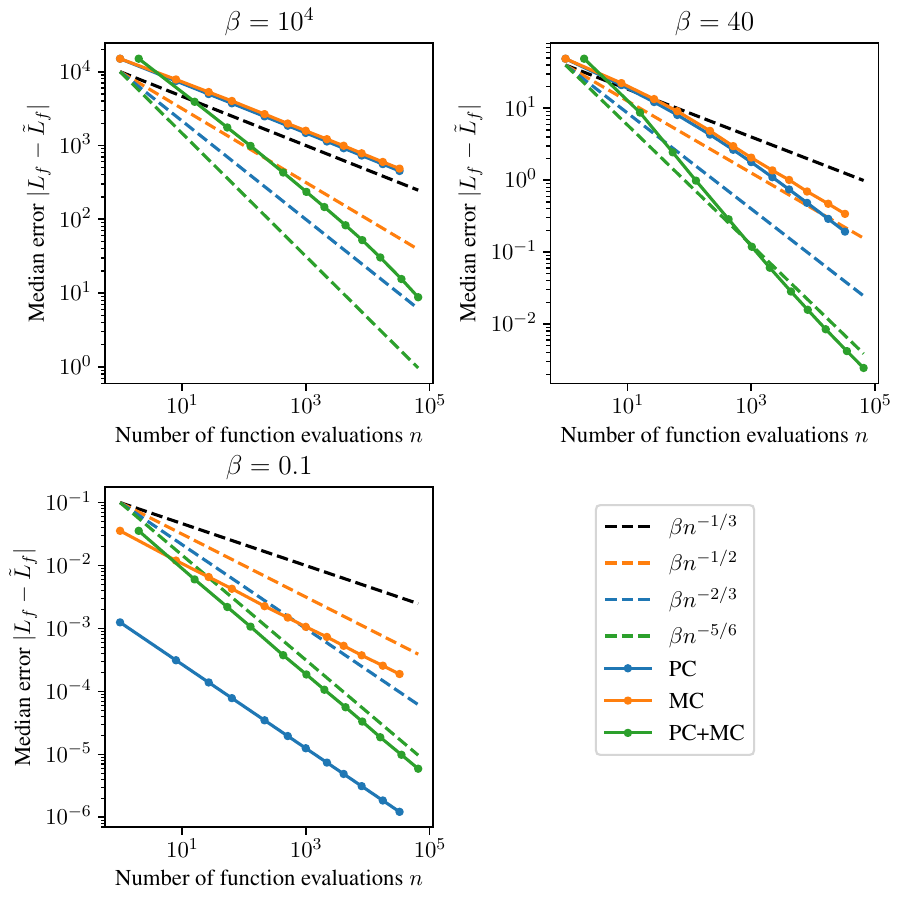}
\caption{Convergence of the (median) error $|L_f - \tilde L_f|$ for different values of $\beta \in \{0.1, 40, 10000\}$. For the stochastic methods MC and PC+MC, the median is taken over 10001 independent runs.} \label{fig:logpartition}
\end{figure}

\Cref{fig:logpartition} shows the convergence of these methods for $\beta \in \{0.1, 40, 10000\}$. For $\beta = 10000$, the methods are in an optimization regime, where PC and MC follow the rate $O(n^{-1/3})$ of the corresponding upper bounds in \Cref{thm:pc_approx} and \Cref{thm:mc_log_partition}. Meanwhile, PC+MC follows the rate $O(n^{-2/3})$. This can be understood intuitively by noting that due to the linear nature of $f$, the PC proposal distribution will mostly propose points close to the optimum, such that the MC component can get much closer to the optimum than with a uniform proposal distribution. 

For $\beta = 40$, we observe a transition between an optimization regime and a quadrature regime. In the quadrature regime, the convergence rate of MC is the classical MC quadrature rate $O(n^{-1/2})$, matching the upper bound in \Cref{thm:mc_log_partition}. Meanwhile, the convergence rate of PC transitions to $O(n^{-2/3})$, matching the worst-case bound in \Cref{thm:pc_approx}, whose proof uses a linear $f$ for the lower bound. The combination PC+MC approaches a convergence rate around $O(n^{-5/6})$. This can be understood as the MC rate $O(n^{-1/2})$ combined with the \emph{approximation} rate (not log-partition rate) of PC, which is $O(n^{-1/3})$. The rate $O(n^{-5/6})$ can be proven formally using arguments analogous to the proof of \Cref{thm:upper_stoch_log} in \Cref{sec:appendix:stochastic_points}.

For $\beta = 0.1$, we see the same quadrature regime rates as for $\beta = 30$, except that now the constant in the rate for PC is smaller than those of MC and PC+MC. This can be explained by an observation in the proof of \Cref{thm:pc_approx} in \Cref{sec:appendix:approx_algs}: Since PC performs midpoint quadrature, its error depends on the curvature of $\exp(f)$. Since $f$ is linear, the curvature of $\exp(f)$ is significantly smaller than the worst-case curvature when $\beta \ll 1$. On the other hand, the convergence rate of PC+MC depends on the \emph{approximation} rate of PC, which depends on the gradient and not the curvature.

\subsection{Sampling}

To study convergence rates for sampling, we need a way to estimate distances between probability distributions through samples. While this can be achieved for the Wasserstein distance, and more efficiently for the related Sinkhorn distances, an even more efficient and easy-to-compute measure is the energy distance \citep[see e.g.][]{szekely_energy_2013} given by
\begin{IEEEeqnarray*}{+rCl+x*}
\Denergy(P, Q)^2 = 2\bbE_{x \sim P, x' \sim Q} \|x - x'\|_2 - \bbE_{x \sim P, x' \sim P} \|x - x'\|_2 - \bbE_{x \sim Q, x' \sim Q} \|x - x'\|_2~.
\end{IEEEeqnarray*}
We estimate the energy distance $\Denergy(P_f, \tilde P_f)$ by sampling a finite number of samples $x_1, \hdots, x_N \sim P_f$ and $\tilde x_1, \hdots, \tilde x_N \sim \tilde P_f$ and then computing the energy distance $\Denergy(Q, \tilde Q)$ of the empirical distributions
\begin{IEEEeqnarray*}{+rCl+x*}
Q \equalDef \frac{1}{N} \sum_{i=1}^N \delta_{x_i}, \qquad \tilde Q \equalDef \frac{1}{N} \sum_{i=1}^N \delta_{\tilde x_i}~,
\end{IEEEeqnarray*}
where $\delta_x$ is the Dirac distribution at $x$. We compare the following sampling algorithms:
\begin{itemize}
\item \textbf{PC}: Sampling from a piecewise constant approximation as in \Cref{sec:pc_approx}.
\item \textbf{MC}: Monte carlo sampling as defined in \Cref{sec:mc_sampling}.
\item \textbf{RS}: We return $\textsc{RejectionSampling}(f, M_f, n)$ as defined in \Cref{alg:rejection_sampling} and investigated in \Cref{sec:rs_uniform}. Here, we know $M_f$ explicitly due to the simple nature of $f$.
\item \textbf{PC+MC}: Performing MC sampling on top of a piecewise constant proposal distribution: We compute a piecewise constant approximant $g$ of $f$ with $n/2$ points, then draw samples $X_1, \hdots, X_{n/2} \sim P_g$ and output $X_I$, where
\begin{IEEEeqnarray*}{+rCl+x*}
P(I = i) = \frac{\exp(f(X_i) - g(X_i))}{\sum_{j=1}^{n/2} \exp(f(X_j) - g(X_j))}~.
\end{IEEEeqnarray*}
\item \textbf{PC+RS}: We use $n/2$ points to compute a piecewise constant approximation $g$ of $f$ and then return $\textsc{RejectionSampling}(f, g + M_{f-g}, n/2)$ as defined in \Cref{alg:rejection_sampling}. Here, we know $M_{f-g}$ explicitly due to the simple nature of $f$.
\end{itemize}
All of the above methods can be implemented in linear time $O_{m, d}(n)$.

\begin{figure}[t]
\centering
\includegraphics[scale=1.0]{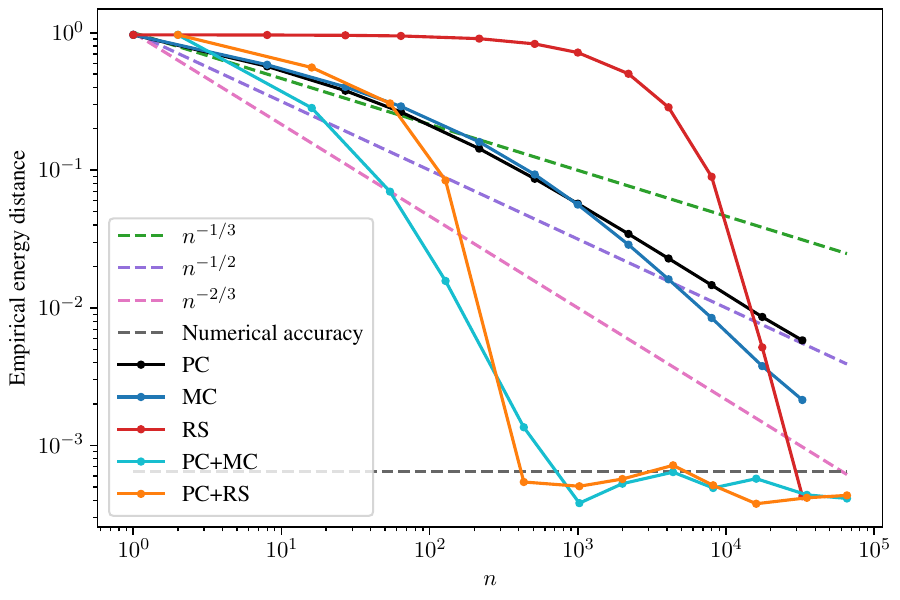}
\caption{Convergence of different sampling methods in terms of the empirical energy distance, computed using $N=10^6$ samples for each distribution, to the true distribution $P_f$ for $\beta = 15$. Here, $n$ denotes the number of function evaluations used for drawing a sample, where PC uses the same function evaluations for each sample, while MC and RS need new function evaluations for every drawn sample. The gray dashed line corresponds to the maximum empirical energy distance of two sets of $N=10^6$ samples, both drawn from $P_f$, where the maximum is taken over three random draws.} \label{fig:sampling}
\end{figure}

For the sampling algorithms in \Cref{fig:sampling}, the behavior in terms of convergence rates is less clear than for the log-partition algorithms. For PC and MC, we observe a transition between optimization and quadrature regimes with different rates. For PC, we would not expect such a transition from \Cref{thm:pc_approx}, but the analysis there is for $\Dsuplog$ and not for the energy distance. In \Cref{fig:sampling}, we also see that combining approximation-based and stochastic methods performs better than either of the two in isolation. While we do not analyze these combinations in our theory, many proof ideas should generalize to PC+MC and PC+RS. 
Our budget-limited variant of rejection sampling, RS, initially performs poorly in \Cref{fig:sampling} while reaching fast convergence for larger values of $n$, when the probability of overall rejection becomes small. This behavior matches the behavior of the bounds in \Cref{prop:rejection_sampling}.

Overall, our experiments show the promise of combining function approximation with other methods for the log-partition and sampling problems. However, they are only toy experiments and should not be seen as practical advice. A practical evaluation including MCMC methods is beyond the scope of this paper. Moreover, testing the variational approach of \cite{bach_sum--squares_2025} would require deriving a tractable version for a suitable non-periodic feature map. For experiments on the Boolean hypercube, we refer to \cite{beyler_variational_2025}.

\section{Conclusion} \label{sec:sampling:conclusion}

In this paper, we studied the convergence rates of sampling and log-partition estimation methods on classes of $m$-smooth functions on the $d$-dimensional unit cube $\calX = [0, 1]^d$. In \Cref{sec:ibc}, we showed that without computational constraints, the optimal achievable convergence rates are of the form $O_{m, d}(Bn^{-m/d})$ or even better depending on the setting. We then investigated several computational reductions between problems in \Cref{sec:relations}, showing that several problems are similarly hard. In \Cref{sec:algorithms}, we studied convergence rates of specific algorithms, which are far from being optimal unless one is willing to spend a computational effort on the order of $O(n^m)$, that is, exponential in the smoothness $m$ for which the optimal rate should be achieved. Our experimental study nonetheless confirms practical differences between the convergence rates of some of the investigated efficient algorithms, although it is limited to a toy problem and simple algorithms.

Our work poses the central question of whether near-optimal convergence rates for smooth functions can be achieved with runtimes that are of fixed polynomial order $O_{m, d}(n^k)$, i.e., where $k$ does not depend on $m$ or $d$. Moreover, for many sampling algorithms, it is unclear which convergence rates they can achieve in our setting. For example, variants of parallel tempering are often employed for non-log-concave problems, and diffusion models might prove to be relevant if the score function can be approximated efficiently \citep{chen_sampling_2023}. An analysis of (mixtures of) Laplace approximations might also be interesting in this context \citep{shun_laplace_1995, ruli_improved_2016, eschenhagen_mixtures_2021}. Beyond specific algorithms, proving lower bounds outside of the optimization regime is still an open question except for some special cases \citep{chewi_query_2022}, and other probability distance measures such as the KL divergence could be considered as well.

\acks{We thank Alessandro Rudi, Viktor Zaverkin, Hans Kersting, Ingo Steinwart, Davoud Mirzaei, Marc Lambert, and Eric Moulines for helpful discussions. 
Funded by Deutsche Forschungsgemeinschaft (DFG, German Research Foundation) under Germany's Excellence Strategy - EXC 2075 – 390740016. The authors thank the International Max Planck Research School for Intelligent Systems (IMPRS-IS) for supporting David Holzmüller. Francis Bach acknowledges support from the French government under the management of the Agence Nationale de la Recherche as part of the ``Investissements d'avenir'' program, reference ANR-19-P3IA0001 (PRAIRIE 3IA Institute). This work was also supported by the European Research Council (grant SEQUOIA 724063).}

\begin{appendixenv}

\section{Proofs for Introduction} \label{sec:appendix:introduction}

\LemLipMaxBound*

\begin{proof}
For the first part of the lemma, let $x^* \in \calX$ be a maximizer of $f$. Without loss of generality, assume that $f$ is shifted such that $f(x^*) = 0$. 

\textbf{Step 1: Upper bound.} We have
\begin{IEEEeqnarray*}{+rCl+x*}
L_f = \log \int_{\calX} e^{f(x)} \diff x \leq \log \int_{\calX} 1 \diff x = \log 1 = 0 = M_f~.
\end{IEEEeqnarray*}

\textbf{Step 2: Lower bound.} To show a lower bound on $L_f$, define the side length $R \equalDef (\max\{1, d^{-1/2} |f|_1\})^{-1}$. Since $R \leq 1$, $\calX$ contains an axis-aligned subcube $\tilde \calX$ of side length $R$ containing $x^*$. Each point $x \in \tilde\calX$ has distance at most $\sqrt{d}R$ from $x^*$, and hence by Lipschitzness, we have 
\begin{IEEEeqnarray*}{+rCl+x*}
f(x) \geq f(x^*) - |f|_1 \sqrt{d}R = -|f|_1\sqrt{d}R~.
\end{IEEEeqnarray*}

We consider two cases:
\begin{enumerate}[(a)]
\item \textbf{Case 1: $d^{-1/2}|f|_1 \leq 1$.} In this case, we have $R = 1$ and hence
\begin{IEEEeqnarray*}{+rCl+x*}
L_f & = & \log \int_{\calX} e^{f(x)} \diff x \geq \log \int_{\calX} e^{-|f|_1\sqrt{d}} \diff x = -|f|_1\sqrt{d} = -d(d^{-1/2}|f|_1)~.
\end{IEEEeqnarray*}
The function $h(x) \equalDef \log(1+3x)-x$ is concave and $h(0), h(1) \geq 0$, which shows $h(x) \geq 0$ for $x \in [0, 1]$. Hence,
\begin{IEEEeqnarray*}{+rCl+x*}
L_f \geq -d\log(1+3d^{-1/2}|f|_1)~.
\end{IEEEeqnarray*}
\item \textbf{Case 2: $d^{-1/2}|f|_1 > 1$.} In this case, we have $R = (d^{-1/2}|f|_1)^{-1}$ and hence $f(x) \geq -d$ for $x \in \tilde\calX$. This yields
\begin{IEEEeqnarray*}{+rCl+x*}
L_f & = & \log \int_{\calX} e^{f(x)} \diff x \geq \log \int_{\tilde \calX} e^{f(x)} \diff x \\
& \geq & \log \int_{\tilde \calX} e^{-d} \diff x = -d + d\log(R) = -d - d\log(d^{-1/2} |f|_1) \\
& = & -d\log(ed^{-1/2}|f|_1) \geq -d\log(1+3d^{-1/2}|f|_1)~.
\end{IEEEeqnarray*}
\end{enumerate}

\textbf{Step 3: Including the temperature.} By replacing $f$ with $f/\eps$, we obtain
\begin{IEEEeqnarray*}{+rCl+x*}
|M_f - \eps L_{f/\eps}| & = & |\eps M_{f/\eps} - \eps L_{f/\eps}| = \eps |M_{f/\eps} - L_{f/\eps}| \leq \eps d \log(1+3d^{-1/2}|f|_1/\eps)~.
\end{IEEEeqnarray*}

\textbf{Step 4: Probabilistic bound.} We have
\begin{IEEEeqnarray*}{+rCl+x*}
P_{f/\eps}(x: f(x) \leq \eps L_{f/\eps} - \eps \log(1/\delta)) & \leq & \int_{\{x \in \calX: f(x) \leq \eps L_{f/\eps} - \eps \log(1/\delta)\}} \exp(f(x)/\eps - L_{f/\eps}) \diff x \\
& \leq & \int_{\calX} \exp(-\log(1/\delta)) \diff x = \delta~. & \qedhere
\end{IEEEeqnarray*}
\end{proof}

The following lemma will be useful to deal with Lipschitz constants:

\begin{lemma} \label{lemma:lipschitz_constant}
Let $f \in C^m(\calX), m \geq 1$. Then, $|f|_1 \leq d^{1/2} \|f\|_{C^m}$.
\end{lemma}

\begin{proof}
We have
\begin{IEEEeqnarray*}{+rCl+x*}
|f|_1 & = & \sup_{x \in \calX} \|\nabla f(x)\|_2 \leq \sup_{x \in \calX} d^{1/2} \|\nabla f(x)\|_\infty \leq d^{1/2} \|f\|_{C^1} \leq d^{1/2} \|f\|_{C^m}~. & \qedhere
\end{IEEEeqnarray*}
\end{proof}

\section{Proofs for Information-based Complexity} \label{sec:appendix:ibc}

Most of our lower bounds rely on the common strategy of hiding smooth functions with small support somewhere in the domain \citep[see e.g.][]{novak_deterministic_1988}. We will consider the following bump functions:

\begin{definition}[Bump functions] \label{def:bump_functions}
We define the template one-dimensional bump function
\begin{IEEEeqnarray*}{+rCl+x*}
\tilde b: \bbR \to \bbR, x \mapsto \begin{cases}
\exp(4 - (1 - x)^{-1} - (x + 1)^{-1}) &,\text{ if } x \in (-1, 1) \\
0 &,\text{ otherwise}
\end{cases}
\end{IEEEeqnarray*}
and, for given dimension $d$, the template multi-dimensional bump function
\begin{IEEEeqnarray*}{+rCl+x*}
b: \bbR^d \to \bbR, x \mapsto \tilde b(x_1) \cdots \tilde b(x_d)~.
\end{IEEEeqnarray*}
for $z \in \bbR^d$ and $\delta > 0$, the shifted and scaled bump functions
\begin{IEEEeqnarray*}{+rCl+x*}
b_{z, \delta}: \bbR^d \to \bbR, x \mapsto b\left(\frac{x-z}{\delta}\right)~.
\end{IEEEeqnarray*}
Moreover, we define the open cube
\begin{IEEEeqnarray*}{+rCl+x*}
B_\infty(x, \delta) & \equalDef & \{z \in \bbR^d \mid \|z - x\|_\infty < \delta\}~. & \qedhere
\end{IEEEeqnarray*}
\end{definition}

The following lemma illustrates some important properties of these bump functions:

\begin{lemma}[Bump functions] \label{lemma:bump_functions}
The bump functions $b_{z, \delta}$ from \Cref{def:bump_functions} satisfy
\begin{enumerate}[(a)]
\item $b_{z, \delta}$ is zero outside of $B_\infty(z, \delta)$.
\item $b_{z, \delta}$ is infinitely often continuously differentiable and all of its derivatives are bounded,
\item there exists a constant $C_{m, d} > 0$ independent of $z$ and $\delta$ such that for all $z \in \bbR^d$ and $\delta > 0$,
\begin{IEEEeqnarray*}{+rCl+x*}
\|b_{z, \delta}\|_{C^m(\bbR^d)} \leq C_{m, d} \max\{1, \delta^{-m}\}~.
\end{IEEEeqnarray*}
\item For $x \in B_\infty(z, \delta/2)$, we have $b_{z, \delta}(x) \geq 1$.
\end{enumerate}

\begin{proof}
\leavevmode
\begin{enumerate}[(a)]
\item This is easy to verify from the definition.
\item It is well-known, see e.g.\ Remark 3.4 (d) in Chapter V.3 of \cite{amann_analysis_2005}, that the function
\begin{IEEEeqnarray*}{+rCl+x*}
\hat b: \bbR \to \bbR, x \mapsto \begin{cases}
0 &, x \leq 0 \\
\exp(-1/x) &, x > 0
\end{cases}
\end{IEEEeqnarray*}
is $C^\infty$. Since $\tilde b(x) = e^4\hat b(1-x) \cdot \hat b(x+1)$, $\tilde b$ is also $C^\infty$, and so must be $b$ and $b_{z, \delta}$. Moreover, since $b_{z, \delta}$ has compact support, all the derivatives are bounded.
\item Let $C_{m, d} \equalDef \|b\|_{C^m(\bbR^d)}$. We have
\begin{IEEEeqnarray*}{+rCl+x*}
\sup_{x \in \bbR^d} \left|\frac{\partial^\alpha b_{z, \delta}(x)}{\partial x^\alpha}\right| = \delta^{-|\alpha|_1} \sup_{x \in \bbR^d} \left|\frac{\partial^\alpha b(x)}{\partial x^\alpha}\right|~.
\end{IEEEeqnarray*}
Therefore, by definition of the $C^m$-norm, we have
\begin{IEEEeqnarray*}{+rCl+x*}
\|b_{z, \delta}\|_{C^m(\bbR^d)} \leq C_{m, d} \max\{1, \delta^{-m}\}~.
\end{IEEEeqnarray*}
\item It is easy to verify that $\tilde b(x) \geq 1$ for $|x| \leq 1/2$. For $x \in B_\infty(z, \delta/2)$, we have
\begin{IEEEeqnarray*}{+rCl+x*}
\|(x-z)/\delta\|_\infty \leq 1/2~,
\end{IEEEeqnarray*}
hence
\begin{IEEEeqnarray*}{+rCl+x*}
b_{z, \delta}(x) & = & b\left(\frac{x-z}{\delta}\right) \geq 1 \cdots 1 = 1~. & \qedhere
\end{IEEEeqnarray*}
\end{enumerate}
\end{proof}
\end{lemma}

The following lemma is useful to bound the number of bump functions that we can hide in a domain:

\begin{lemma} \label{lemma:subcubes}
For $k \in \bbN_{\geq 1}$, a third-slice $\tilde \calX \equalDef [0, 1/3] \times [0, 1]^{d-1}$ of the cube $\calX$ contains at least $k$ disjoint open cubes $B_\infty(z_1, r_k), \hdots, B_\infty(z_k, r_k)$ with radius
\begin{IEEEeqnarray*}{+rCl+x*}
r_k = \frac{k^{-1/d}}{12}~.
\end{IEEEeqnarray*}
\end{lemma}

\begin{proof}
Choose $N \equalDef \lceil k^{1/d}\rceil$. We can divide $\tilde \calX$ into $N \cdot (3N)^{d-1} \geq N^d \geq k$ cubes of side length $(3N)^{-1}$ and radius
\begin{IEEEeqnarray*}{+rCl+x*}
r & = & \frac{1}{6N} \geq \frac{1}{6(k^{1/d}+1)} \geq \frac{1}{12k^{1/d}} = \frac{k^{-1/d}}{12} = r_k~. & \qedhere
\end{IEEEeqnarray*} 
\end{proof}

\subsection{Deterministic Evaluation Points} \label{sec:appendix:deterministic_points}

We first adapt some results from \cite{novak_deterministic_1988} to our setting.

\thmNovak*

\begin{proof}
\textbf{Step 1: Upper bounds.} For $S \in \{\Sapp, \Sopts, \Sint\}$, \cite{novak_deterministic_1988} states upper bounds of the form $O_{m, d}(n^{-m/d})$ for bounded classes of functions in the Sobolev space $W^{m, d}_\infty$, which contain $\calF_{d, m, B_{m, d}}$ for some $B_{m, d} > 0$ (see Section 1.3.11 and 1.3.12 in \cite{novak_deterministic_1988}). Hence, for the corresponding metric $D$, we have
\begin{IEEEeqnarray*}{+rCl+x*}
e_n(\calF_{d, m, B_{m, d}}, S, D) \leq O_{m, d}(n^{-m/d})~.
\end{IEEEeqnarray*}
For another value of $B$, we can then take a near-optimal $\tilde S \in \calA_n$ for $\calF_{d, m, B_{m, d}}$ and define
\begin{IEEEeqnarray*}{+rCl+x*}
\hat{\tilde S}(f) \equalDef \frac{B}{B_{m, d}} \tilde S\left(\frac{B_{m, d}}{B} f\right)~,
\end{IEEEeqnarray*}
and by positive homogeneity of $S$ and $D$, this then achieves the rate $O_{m, d}(Bn^{-m/d})$.

\textbf{Step 2: Lower bounds.} For lower bounds, it is again sufficient to consider $\calF_{d, m, B_{m, d}}$ for a single $B_{m, d} > 0$. \cite{novak_deterministic_1988} uses bump functions created by rescaling and shifting the template bump function
\begin{IEEEeqnarray*}{+rCl+x*}
\Phi(x) = \begin{cases}
a \prod_{i=1}^d (1 - x_i^2)^m &, x \in [-1, 1]^d \\
0 &, \text{ otherwise}
\end{cases} 
\end{IEEEeqnarray*}
for some appropriate constant $a>0$. This function is in $W^{m, d}_\infty$ but not all of its weak $m$-th derivatives are continuous. Hence, the constructed counterexamples do not directly apply to $\calF_{d, m, B_{m, d}}$. However, it is possible to replace $\Phi$ by the $C^\infty$ bump function $b$ from \Cref{def:bump_functions} since the norms of the derivatives behave in the same fashion for scaled and shifted versions of $b$, as shown in \Cref{lemma:bump_functions}. Hence, the same lower bounds still apply to $\calF_{d, m, B_{m, d}}$. 
\end{proof}

We can now turn to our upper bounds through approximation:

\propAppUpperBounds*

\begin{proof}
Since $L_f$ and $P_f$ are not influenced by changing $f$ on null sets, we will ignore exceptional null sets in the essential supremum in the definition of $\|\cdot\|_\infty$ in the following.
\begin{enumerate}[(a)]
\item We have
\begin{IEEEeqnarray*}{+rCl+x*}
L_g & = & \log \int_\calX e^{g(x)} \diff x  \\
& \leq & \log \int_\calX e^{f(x) + \|f - g\|_\infty} \diff x = \log \left(e^{\|f - g\|_\infty} \cdot \int_\calX e^{f(x)} \diff x\right) \\
& = & \left(\log \int_\calX e^{f(x)} \diff x\right) + \|f - g\|_\infty = L_f + \|f - g\|_\infty~,
\end{IEEEeqnarray*}
and the other inequality follows analogously.
\item We have $p_f(x) = \exp(f(x) - L_f)$ and $p_g(x) = \exp(g(x) - L_g)$, hence
\begin{IEEEeqnarray*}{+rCl+x*}
\Dsuplog(P_f, P_g) & = & \left\|\log\left(\frac{p_f}{p_g}\right)\right\|_{\infty} = \|(f - L_f) - (g - L_g)\|_\infty \\
& \leq & \|f - g\|_\infty + |L_f - L_g| \\
& \stackrel{\text{(a)}}{\leq} & 2\|f - g\|_\infty~.
\end{IEEEeqnarray*}
Let $\bar f \equalDef f - L_f$. By a well-known property of the TV distance \citep[see e.g.\ Lemma 2.1 in][]{tsybakov_introduction_2009},
\begin{IEEEeqnarray*}{+rCl+x*}
\DTV(P_f, P_g) & = & D_{\mathrm{TV}}(P_{\bar f}, P_{\bar g}) = \frac{1}{2} \int_{\calX} |e^{\bar f(x)} - e^{\bar g(x)}| \diff x~.
\end{IEEEeqnarray*}
Now, consider a fixed $x \in \calX$. Without loss of generality, assume $\bar f(x) \leq \bar g(x)$. Then, 
\begin{IEEEeqnarray*}{+rCl+x*}
e^{\bar g(x) - \|\bar f - \bar g\|_\infty} \leq e^{\bar f(x)} \leq e^{\bar g(x)}~,
\end{IEEEeqnarray*}
which yields
\begin{IEEEeqnarray*}{+rCl+x*}
|e^{\bar f(x)} - e^{\bar g(x)}| & \leq & (1 - e^{-\|\bar f - \bar g\|_\infty}) e^{\bar g(x)} \leq (1 - e^{-\|\bar f - \bar g\|_\infty}) (e^{\bar f(x)} + e^{\bar g(x)}) \\
& \leq & \|\bar f - \bar g\|_\infty (e^{\bar f(x)} + e^{\bar g(x)})~.
\end{IEEEeqnarray*}
Therefore,
\begin{IEEEeqnarray*}{+rCl+x*}
D_{\mathrm{TV}}(P_f, P_g) & \leq & \frac{1}{2} \int_{\calX} \|\bar f - \bar g\|_\infty (e^{\bar f(x)} + e^{\bar g(x)}) \diff \mu(x) = \|\bar f - \bar g\|_\infty = \Dsuplog(P_f, P_g)~.
\end{IEEEeqnarray*}
The bound $\DW(P_f, P_g) \leq \diam(\calX)\DTV(P_f, P_g) = d^{1/2}\DTV(P_f, P_g)$ for the 1-Wasser\-stein distance, where $\diam(\calX)$ is the diameter of $\calX$, is well-known \citep[see e.g.][]{gibbs_choosing_2002}. \qedhere
\end{enumerate}
\end{proof}

The following technical lemmas will be used for the lower bound afterward.

\begin{lemma} \label{lemma:helper_fraction}
Let $a, b > 0$. Then,
\begin{IEEEeqnarray*}{+rCl+x*}
\frac{a}{a+b} \geq \frac{1}{2} \min\left\{1, \frac{a}{b}\right\}~.
\end{IEEEeqnarray*}

\begin{proof}
If $a \leq b$, we have
\begin{IEEEeqnarray*}{+rCl+x*}
\frac{a}{a+b} \geq \frac{a}{2b} \geq \frac{1}{2} \min\left\{1, \frac{a}{b}\right\}~.
\end{IEEEeqnarray*}
Similarly, if $a \geq b$, we have
\begin{IEEEeqnarray*}{+rCl+x*}
\frac{a}{a+b} & \geq & \frac{a}{2a} = \frac{1}{2} \geq \frac{1}{2} \min\left\{1, \frac{a}{b}\right\}~. & \qedhere
\end{IEEEeqnarray*}
\end{proof}
\end{lemma}

\begin{lemma} \label{lemma:log_function}
Let $c \in (0, 1]$. Then, the function
\begin{IEEEeqnarray*}{+rCl+x*}
h: [0, \infty) \to \bbR, x \mapsto \log(1 + c(e^x - 1))
\end{IEEEeqnarray*}
satisfies $h(x) \geq cx$ for all $x \geq 0$.
\end{lemma}

\begin{proof}
For all $x \geq 0$, we have
\begin{IEEEeqnarray*}{+rCl+x*}
h'(x) & = & \frac{ce^x}{1 + c(e^x - 1)} = \frac{c}{c + (1-c)e^{-x}} \geq \frac{c}{c + (1-c)} = c~.
\end{IEEEeqnarray*}
Therefore,
\begin{IEEEeqnarray*}{+rCl+x*}
h(x) & = & h(0) + \int_0^x h(u) \diff u \geq \int_0^x c \diff u = cx~. & \qedhere
\end{IEEEeqnarray*}

\end{proof}

Now, we are ready to prove the exact minimax optimal rates. The main technical difficulty is that for the lower bound in the 1-Wasserstein distance, we need to hide many bumps that are far apart, and we need to bound the resulting Wasserstein distance.

\thmDeterministicPointsRates*

\begin{proof}
\textbf{Step 0: Upper bounds.} We know from \Cref{thm:novak} that the rate $O_{m, d}(Bn^{-m/d})$ can be achieved for approximation with non-adaptive deterministic evaluation points, and we know from \Cref{prop:app_bounds} that this rate can therefore also be achieved for the log-partition problem and the sampling problem with $\Dsuplog$, $\DTV$, and $\DW$. Moreover, since $\DTV(P, Q) \leq 1$ for all distributions $P, Q$, we obtain an upper bound of $O_{m, d}(\max\{1, Bn^{-m/d}\})$ for $\DTV$. Similarly, since $\calX$ has diameter $d^{1/2}$, $\DW$ is upper bounded by $d^{1/2} = O_{m, d}(1)$, and hence we also obtain an upper bound of $O_{m, d}(\max\{1, Bn^{-m/d}\})$ for $\DW$. The upper bounds also hold for the adaptive setting since it is more permissive.

In the following, we will derive matching asymptotic lower bounds for the adaptive setting, which then also hold for the non-adaptive setting. To this end, let $\tilde S \in \Aad_n$ for the log-partition or sampling problem on the function class $\calF_{d, m, B}$.

\textbf{Step 1: Defining grids in the cube.} We can cut the cube $\calX$ along one axis into three equally shaped slices $\calC_0, \calC_1, \calC_2$:
\begin{IEEEeqnarray*}{+rCl+x*}
\calC_k \equalDef [k/3, (k+1)/3] \times [0, 1]^{d-1}, \qquad k \in \{0, 1, 2\}~.
\end{IEEEeqnarray*}
Then, by \Cref{lemma:subcubes}, we can find a finite set of points $\calG_k \subseteq \calC_k$ with $|\calG_k| = 2n$ such that the open cubes $B_\infty(x, \delta_n)$ for $x \in \calG_k$ and radius
\begin{IEEEeqnarray*}{+rCl+x*}
\delta_n = \frac{(2n)^{-1/d}}{12} \geq \frac{n^{-1/d}}{24}
\end{IEEEeqnarray*}
are contained in $\calC_k$ and disjoint.

\textbf{Step 2: Removing points close to queried points.} Let $\calX_n$ denote the $\leq n$ points where $\tilde S$ queries the zero function $\tilde f(x) = 0$.
For fixed $k \in \{0, 2\}$, the $2n$ cubes $(B_\infty(x, \delta_n))_{x \in \calG_k}$ are disjoint. Hence, there must be a subset $\tilde \calG_k \subseteq \calG_k$ containing $n$ points whose corresponding cubes do not contain any point from $\calX_n$.

\textbf{Step 3: Two different functions.} Now, for $k \in \{0, 2\}$ and $C_{m, d}$ as in \Cref{lemma:bump_functions}, define the functions
\begin{IEEEeqnarray*}{+rCl+x*}
f_k(x) \equalDef BC_{m, d}^{-1} \delta_n^m \sum_{z \in \tilde\calG_k} b_{z, \delta_n}(x)~.
\end{IEEEeqnarray*}
We have $\delta_n \leq 1$ and hence $\|b_{z, \delta_n}\|_{C^m} \leq C_{m, d} \delta_n^{-m}$ by \Cref{lemma:bump_functions}. Because the support of the bump functions does not overlap, we have $\|f_k\|_{C^m} \leq B$ by \Cref{lemma:bump_functions} and hence $f_k \in \calF_{d, m, B}$. By the construction of $\tilde\calG_k$, $f_0$ and $f_2$ are zero on $\calX_n$. Hence, even an adaptive $\tilde S$ must also query $f_k$ at the points in $\calX_n$, and since both are equal at those points, we must have
\begin{IEEEeqnarray*}{+rCl+x*}
\tilde S(f_0) = \tilde S(f_2)~.
\end{IEEEeqnarray*}

\textbf{Step 4: Wasserstein distance of both functions.} Because $f_0$ and $f_2$ use the same number of equally wide bump functions whose support is fully contained in $\calX$, we have
\begin{IEEEeqnarray*}{+rCl+x*}
L_{f_0} = L_{f_2}~. \IEEEyesnumber \label{eq:equal_integral}
\end{IEEEeqnarray*} 
To lower-bound the $1$-Wasserstein distance, we use its dual formulation and choose the $1$-Lipschitz function $\varphi(x) \equalDef x_1 - 1/3$. This yields
\begin{IEEEeqnarray*}{+rCl+x*}
W_1(P_{f_0}, P_{f_2}) & \geq & \bbE_{x \sim P_{f_2}} \varphi(x) - \bbE_{x \sim P_{f_0}} \varphi(x) = \int_{\calX} \varphi(x) (e^{\ovl f_2(x)} - e^{\ovl f_0(x)}) \diff x \\
& \stackrel{\text{\Eqref{eq:equal_integral}}}{=} & e^{-L_{f_0}} \int_{\calX} \varphi(x) (e^{f_2(x)} - e^{f_0(x)}) \diff x~. \IEEEyesnumber \label{eq:wasserstein_dual_lower}
\end{IEEEeqnarray*}

\textbf{Step 5: Lower-bounding the normalization constant.} We first define the \quot{bump integral}
\begin{IEEEeqnarray*}{+rCl+x*}
I_n \equalDef \int_{B_\infty(z, \delta_n)} (e^{BC_{m, d}^{-1} \delta_n^m b_{z, \delta_n}(x)} - 1) \diff x~,
\end{IEEEeqnarray*}
which is independent of $z$. Then, we have
\begin{IEEEeqnarray*}{+rCl+x*}
I_n & \stackrel{\text{\Cref{lemma:bump_functions}}}{\geq} & \int_{B_\infty(z, \delta_n/2)} (e^{BC_{m, d}^{-1} \delta_n^m} - 1) \diff x \\
& = & \delta_n^d (e^{BC_{m, d}^{-1} \delta_n^m} - 1)~. \IEEEyesnumber \label{eq:In_lower}
\end{IEEEeqnarray*}
We then obtain
\begin{IEEEeqnarray*}{+rCl+x*}
e^{L_{f_0}} & = & \int_{\calX} e^0 \diff x + \sum_{z \in \tilde\calG_0} \int_{B_\infty(z, \delta_n)} (e^{BC_{m, d}^{-1} \delta_n^m b_{z, \delta_n}(x)} - e^0) \diff x \\
& = & 1 + n I_n~. \IEEEyesnumber\label{eq:normalization_bound}
\end{IEEEeqnarray*}

\textbf{Step 6: Lower-bounding the integral.} By construction of the functions $\varphi$, $f_0$, and $f_2$, we know that
\begin{IEEEeqnarray*}{+rCl+x*}
\int_{\calX} \varphi(x) (e^{f_2(x)} - e^{f_0(x)}) \diff x & \geq & \int_{\calC_2} \varphi(x) (e^{f_2(x)} - e^{f_0(x)}) \diff x~. \IEEEyesnumber \label{eq:integral_cube_thirds}
\end{IEEEeqnarray*}
Using $\varphi(x) \geq 1/3$ and $f_2(x) \geq f_0(x)$ for $x \in \calC_2$, we can lower-bound the latter integral as
\begin{IEEEeqnarray*}{+rCl+x*}
\int_{\calC_2} \varphi(x) (e^{f_2(x)} - e^{f_0(x)}) \diff x & \geq & \sum_{z \in \tilde\calG_2} \int_{B_\infty(z, \delta_n)} \frac{1}{3} (e^{BC_{m, d}^{-1} \delta_n^m b_{z, \delta_n}(x)} - 1) \diff x \\
& = & \frac{1}{3} nI_n~.
\end{IEEEeqnarray*}

\textbf{Step 7: Wasserstein distance lower bound.}
By combining the previous lower bounds with Equations \eqref{eq:integral_cube_thirds}, \eqref{eq:normalization_bound}, and \eqref{eq:wasserstein_dual_lower}, we arrive at
\begin{IEEEeqnarray*}{+rCl+x*}
W_1(P_{f_0}, P_{f_2}) \geq \frac{(1/3)nI_n}{1 + nI_n} = \frac{1}{3} \frac{I_n}{I_n + n^{-1}}~.
\end{IEEEeqnarray*}
We can then apply \Cref{lemma:helper_fraction} and \Eqref{eq:In_lower} to obtain, for a suitable constant $c_{m, d} > 0$,
\begin{IEEEeqnarray*}{+rCl+x*}
W_1(P_{f_0}, P_{f_2}) & \geq & \frac{1}{6} \min\left\{1, \frac{I_n}{n^{-1}}\right\} \\
& \geq & \frac{1}{6} \min\left\{1, \frac{\delta_n^d (e^{BC_{m, d}^{-1} \delta_n^m} - 1)}{n^{-1}}\right\} \\
& \geq & \frac{1}{6} \min\left\{1, n \delta_n^d BC_{m, d}^{-1} \delta_n^m\right\} \\
& \geq & \frac{1}{6} \min\left\{1, c_{m, d} B n^{-m/d}\right\}~.
\end{IEEEeqnarray*}

\textbf{Step 8: Wasserstein minimax rate lower bound.} Suppose that we are considering the sampling problem. As argued before, we have $\tilde S(f_0) = \tilde S(f_2)$. Hence, by an application of the triangle inequality, we must have $k \in \{0, 2\}$ such that
\begin{IEEEeqnarray*}{+rCl+x*}
W_1(P_{f_k}, \tilde S(f_k)) & \geq & \frac{1}{12} \min\left\{1, c_{m, d} B n^{-m/d}\right\}~.
\end{IEEEeqnarray*}
The Wasserstein minimax lower bound then follows by setting $f \equalDef f_k$.

\textbf{Step 9: TV distance minimax lower bound.} Since 
\begin{IEEEeqnarray*}{+rCl+x*}
\DTV(P_f, \tilde S(f)) \geq d^{-1/2} \DW(P_f, \tilde S(f))
\end{IEEEeqnarray*}
\citep[see e.g.][]{gibbs_choosing_2002}, we obtain the same asymptotic lower bound for the TV distance.

\textbf{Step 10: Sup-log minimax lower bound.} We have
\begin{IEEEeqnarray*}{+rCl+x*}
\Dsuplog(P_{f_0}, P_{f_2}) & = & \|(f_0 - L_{f_0}) - (f_2 - L_{f_2})\|_\infty = \|f_0 - f_2\|_\infty \\
& \geq & \Omega_{m, d}(BC_{m, d}^{-1} \delta_n^m) = \Omega_{m, d}(Bn^{-m/d})~.
\end{IEEEeqnarray*}
Since $\tilde S(f_0) = \tilde S(f_2)$, by the triangle inequality, there must hence exist $k \in \{0, 2\}$ such that
\begin{IEEEeqnarray*}{+rCl+x*}
\Dsuplog(P_{f_k}, \tilde S(f_k)) \geq \Omega_{m, d}(Bn^{-m/d})~.
\end{IEEEeqnarray*}

\textbf{Step 11: Log-partition minimax lower bound.} Suppose that we instead consider the log-partition problem. Setting $c_d = 24^{-d}$, we obtain
\begin{IEEEeqnarray*}{+rCl+x*}
L_{f_2} & = & \log(1 + nI_n) \geq \log(1 + n\delta_n^d (e^{BC_{m, d}^{-1}\delta_n^m} - 1)) \\
& \geq & \log\left(1 + c_d (e^{BC_{m, d}^{-1} \delta_n^m} - 1)\right) \\
& \stackrel{\text{\Cref{lemma:log_function}}}{\geq} & c_d BC_{m, d}^{-1} \delta_n^m \\
& \geq & \Omega_{m, d}(Bn^{-m/d})~.
\end{IEEEeqnarray*}
Since $\tilde S$ cannot distinguish the zero function $f \equiv 0$ and $f_2$, we must have
\begin{IEEEeqnarray*}{+rCl+x*}
\max \{|L_f - \tilde S(f)|, |L_{f_2} - \tilde S(f_2)|\} & \geq & \Omega_{m, d}(Bn^{-m/d})~. & \qedhere
\end{IEEEeqnarray*}
\end{proof}

\subsection{Stochastic Evaluation Points} \label{sec:appendix:stochastic_points}

Again, we first adapt some related results from \cite{novak_deterministic_1988} to our setting.

\thmNovakStoch*

\begin{proof}
Analogous to the proof of \Cref{thm:novak} in \Cref{sec:appendix:deterministic_points}, this can be shown using the positive homogeneity of $S \in \{\Sapp, \Sopts, \Sint\}$ and $D \in \{\Dinfty, \Dabs\}$, and by replacing the bump functions in the lower bound by the $C^\infty$ bump functions from \Cref{def:bump_functions}.
\end{proof}

We now prove our upper bound for log-partition estimation with stochastic evaluation points through approximation and importance sampling:

\thmUpperStochLog*

\begin{proof}
The bound $O_{m, d}(Bn^{-m/d})$ can be achieved even through methods with deterministic evaluation points, as proven in \Cref{thm:deterministic_points_rates}, hence we only need to show the other bound. Since the first bound is always better for $n=1$, we can, in the following, assume $n \geq 2$. 

Let $\tilde S_n \in \calA_n$ be a sequence of methods for which the worst-case errors
\begin{IEEEeqnarray*}{+rCl+x*}
e_n \equalDef \sup_{f \in \calF_{d, m, B}} \Dinfty(\Sapp(f), \tilde S_n(f))
\end{IEEEeqnarray*}
achieve the optimal rate $O_{m, d}(Bn^{-m/d})$ for the approximation problem on $\calF_{d, m, B}$.

Set $N \equalDef \lfloor n/2\rfloor$, such that $N = \Omega(n)$ (since we assumed $n \geq 2$) and $2N \leq n$. Set $g \equalDef \tilde S_N(f)$. For $N$ i.i.d.\ random variables $X_1, \hdots, X_N \sim P_g$, set
\begin{IEEEeqnarray*}{+rCl+x*}
\mu_N & \equalDef & \frac{1}{N} \sum_{i=1}^N \exp(f(X_i) - g(X_i)) \\
\tilde S_L(f) & \equalDef & L_g + \log \mu_N~.
\end{IEEEeqnarray*}
Then, $\tilde S_L$ only uses $2N \leq n$ function evalutaions of $f$, hence $\tilde S_L \in \sC(\Aad_n)$. 

Since $\|f - g\|_\infty \leq e_N$ and since $\exp$ is $\exp(e_N)$-Lipschitz on $(-\infty, e_N)$, we have
\begin{IEEEeqnarray*}{+rCl+x*}
1 - e_N \leq \exp(-e_N) \leq \exp(f(X_i) - g(X_i)) \leq \exp(e_N) \leq 1 + e_N\exp(e_N)~.
\end{IEEEeqnarray*}
Hence, $|\exp(f(X_i) - g(X_i)) - \bbE \mu_N| \leq (\exp(e_N)+1)e_N$, which implies
\begin{IEEEeqnarray*}{+rCl+x*}
\Var \mu_N \leq N^{-1} ((\exp(e_N)+1)e_N)^2~.
\end{IEEEeqnarray*}
Additionally, 
\begin{IEEEeqnarray*}{+rCl+x*}
\log \bbE \mu_N = \log \int_{\calX} e^{f(x) - g(x)} e^{g(x) - L_g} \diff x = L_f - L_g~.
\end{IEEEeqnarray*}
Moreover, we have $e^{f(X_i) - g(X_i)} \in [\exp(-e_N), \exp(e_N)]$ and hence $\mu_N \in [\exp(-e_N), \exp(e_N)]$.
Since $\log$ is $\exp(e_N)$-Lipschitz on $[\exp(-e_N), \exp(e_N)]$, we obtain
\begin{IEEEeqnarray*}{+rCl+x*}
\bbE |L_f - \tilde S_L(f)| & = & \bbE |\log (\bbE \mu_N) - \log(\mu_N)| \leq \bbE \exp(e_N)|\mu_N - \bbE \mu_N| \\
& \leq & \exp(e_N)\sqrt{\bbE[(\mu_N - \bbE \mu_N)^2]} = \exp(e_N)\sqrt{\Var \mu_N} \\
& \leq & N^{-1/2} e_N\exp(e_N)(\exp(e_N)+1) \leq 2\exp(2e_N) N^{-1/2} e_N \\
& \leq & O_{m, d}(\exp(C_{m, d}Bn^{-m/d}) Bn^{-1/2-m/d})
\end{IEEEeqnarray*}
for a suitable constant $C_{m, d} > 0$.
\end{proof}

In the optimization regime, we can directly exploit the relation to optimization to get a lower bound:

\propLowerStochLog*

\begin{proof}
Take any stochastic log-partition method $\tilde S \in \sC(\Aad_n)$. We can also interpret this as a stochastic optimization method. Hence, we know from an adaptation of the corresponding lower bound by \cite{novak_deterministic_1988} that there exists a constant $c_{m, d} > 0$ and a function $f \in \calF_{d, m, B}$ such that $\bbE \Dabs(\tilde S(f), \Sopts(f)) \geq c_{m, d} Bn^{-m/d}$. But then, using $|f|_1 \leq d^{1/2} \|f\|_{C^1} \leq d^{1/2} B$ from \Cref{lemma:lipschitz_constant}, we obtain
\begin{IEEEeqnarray*}{+rCl+x*}
\bbE |\tilde S(f) - L_f| & \geq & \bbE |\tilde S(f) - M_f| - |M_f - L_f| \\
& = & \bbE \Dabs(\tilde S(f), \Sopts(f)) - |M_f - L_f| \\
& \stackrel{\text{\Cref{lemma:lipschitz_maximization_bound}}}{\geq} & c_{m, d} B n^{-m/d} - d\log(1+3B)~. & \qedhere
\end{IEEEeqnarray*}
\end{proof}

The following lemma will be useful to obtain a bound for rejection sampling in the sup-log distance:

\begin{lemma} \label{lemma:rejection_sampling_function}
Let $p \in [0, 1]$ and $c \geq 0$. Then, for any $a \in [-c, c]$, we have
\begin{IEEEeqnarray*}{+rCl+x*}
|\log(1 + p(e^a - 1))| \leq \min\{c, p(e^c - 1)\}~.
\end{IEEEeqnarray*}
\end{lemma}

\begin{proof}
For an upper bound, we use $\log(1+x) \leq x$ to obtain
\begin{IEEEeqnarray*}{+rCl+x*}
\log(1 + p(e^a - 1)) & \leq & \log(1 + p(e^c - 1)) \leq p(e^c - 1)~, \\
\log(1 + p(e^a - 1)) & \leq & \log(1 + (e^c - 1)) = c~.
\end{IEEEeqnarray*}
For lower bounds, we note that
\begin{IEEEeqnarray*}{+rCl+x*}
1 + p(e^a - 1) \geq 1 + p(e^{-c} - 1) \geq 1 + (e^{-c} - 1) = e^{-c}~.
\end{IEEEeqnarray*}
This immediately yields $\log(1 + p(e^a - 1)) \geq -c$. Moreover, because $\log$ is $e^c$-Lipschitz on $[e^{-c}, \infty)$, we have
\begin{IEEEeqnarray*}{+rCl+x*}
\log(1 + p(e^a - 1)) & \geq & \log(1 + p(e^{-c} - 1)) = \log(1 + p(e^{-c} - 1)) - \log(1) \\
& \geq & -e^c |p(e^{-c} - 1)| = -p(e^c - 1)~. & \qedhere
\end{IEEEeqnarray*}
\end{proof}

Now, we can prove upper bounds for rejection sampling:

\lemRejectionSampling*

\begin{proof}
\textbf{Step 1: Exact distribution.} We prove \Eqref{eq:rejection_distribution} via induction on $n$. For $n=0$, this is clear. Now, suppose the statement is true for $n \in \bbN_0$. Denote by $A$ the event that $\textsc{RejectionSampling}(f, g, n+1)$ accepts in the first iteration. Then, we have
\begin{IEEEeqnarray*}{+rCl+x*}
P(A) = \bbE_{x \sim P_g} \bbE_{u \sim \calU([0, 1])} \bbone[ue^{g(x)} \leq e^{f(x)}] \stackrel{f \leq g}{=} \bbE_{x \sim P_g} \frac{e^{f(x)}}{e^{g(x)}} = \int_{\calX} e^{f(x)-g(x)} \frac{e^{g(x)}}{Z_g} \diff x = \frac{Z_f}{Z_g}~.
\end{IEEEeqnarray*}
The density of $x \in \calX$ conditional on acceptance is
\begin{IEEEeqnarray*}{+rCl+x*}
p(x|A) \propto p(A|x)p(x) = e	^{f(x) - g(x)} \frac{e^{g(x)}}{Z_g} \propto p_f(x)~,
\end{IEEEeqnarray*}
hence $P(x|A) = P_f$. On the other hand, the distribution $P(x|A^c)$, i.e.\ the distribution of $x$ conditioned on non-acceptance is the distribution for $\textsc{RejectionSampling}(f, g, n)$, which we know from \Eqref{eq:rejection_distribution} by the induction hypothesis. Hence, the distribution $\tilde P_f$ for $\textsc{RejectionSampling}(f, g, n+1)$ is
\begin{IEEEeqnarray*}{+rCl+x*}
\tilde P_f & = & P(A)P(\cdot|A) + P(A^c)P(\cdot | A^c) = \frac{Z_f}{Z_g} P_f + \left(1 - \frac{Z_f}{Z_g}\right)\left(P_f + \left(1 - \frac{Z_f}{Z_g}\right)^n(P_g - P_f)\right) \\
& = & P_f + \left(1 - \frac{Z_f}{Z_g}\right)^{n+1}(P_g - P_f) = (1 - p_R)P_f + p_R P_g~.
\end{IEEEeqnarray*}
The argument above also shows that the overall rejection probability is $(1 - P(A)) (1 - Z_f/Z_g)^n = (1 - Z_f/Z_g)^{n+1}$. Moreover, the bound $(1 - Z_f/Z_g)^n \leq \exp(-nZ_f/Z_g)$ follows from $1-x \leq \exp(-x)$ for $x \geq 0$.

\textbf{Step 2: Sup-log distance.}
From step 1, we see that
\begin{IEEEeqnarray*}{+rCl+x*}
\Dsuplog(P_f, \tilde P_f) & = & \left\|\log\left((1 - p_R)e^{\bar f} + p_R e^{\bar g}\right) - \bar f\right\|_\infty \\
& = & \left\|\log\left((1 - p_R) + p_R e^{\bar g - \bar f}\right)\right\|_\infty \\
& = & \left\|\log\left(1 + p_R (e^{\bar g - \bar f} - 1)\right)\right\|_\infty \\
& \stackrel{\text{\Cref{lemma:rejection_sampling_function}}}{\leq} & \min\left\{\|\bar g - \bar f\|_\infty, p_R\left(e^{\|\bar g - \bar f\|_\infty} - 1\right)\right\} \\
& = & \min\left\{\Dsuplog(P_g, P_f), p_R(\exp(\Dsuplog(P_g, P_f)) - 1)\right\}~.
\end{IEEEeqnarray*}

\textbf{Step 3: TV distance.} Using \Eqref{eq:rejection_distribution}, we obtain for the TV distance:
\begin{IEEEeqnarray*}{+rCl+x*}
\DTV(P_f, \tilde P_f) & = & \sup_{A \subseteq \calX\text{ measurable}} |P_f(A) - \tilde P_f(A)| = \sup_{A \subseteq \calX\text{ measurable}} |p_R P_f(A) - p_R P_g(A)| \\
& = & p_R \DTV(P_f, P_g)~.
\end{IEEEeqnarray*}

\textbf{Step 4: 1-Wasserstein distance.} Using \Eqref{eq:rejection_distribution}, we obtain for the 1-Wasserstein distance:
\begin{IEEEeqnarray*}{+rCl+x*}
\DW(P_f, \tilde P_f) & = & \sup_{\varphi\text{ 1-Lipschitz}} \left(\int \varphi(x) \diff P_f(x) - \int \varphi(x) \diff \tilde P_f(x)\right) \\
& = & p_R \sup_{\varphi\text{ 1-Lipschitz}} \left(\int \varphi(x) \diff P_f(x) - \int \varphi(x) \diff P_g(x)\right)  \\
& = & p_R \DW(P_f, P_g)~. & \qedhere
\end{IEEEeqnarray*}
\end{proof}

With the upper bounds for rejection sampling proven above, we can analyze a combination of approximation and rejection sampling to prove the following upper bound:

\thmUpperStochSampling*

\begin{proof}
\textbf{Step 1: Sampling method definition.}
We consider the following sampling method:
\begin{enumerate}[(1)]
\item Use $\lfloor n/2 \rfloor$ function evaluations to create an approximation $g$ of $f$, using a near-optimal approximation method such that the worst-case sup-log error is $E_n \leq O_{m, d}(Bn^{-m/d})$.
\item Return a sample using $\textsc{RejectionSampling}(f, g+e_n, \lceil n/2\rceil)$.
\end{enumerate}
For step (1), we note that we have $\lfloor n/2 \rfloor \geq \Omega(n)$ except if $n=1$. However, in the case $n=1$, we can use the approximation $g=0$ with worst-case error $E_n = B \leq O_{m, d}(Bn^{-m/d})$. Thus, it is indeed possible to achieve the bound in step (1).

\textbf{Step 2: Upper bound.} Denote by $C_{m, d} > 0$ a constant such that $E_n \leq C_{m, d} B n^{-m/d}/2$. Moreover, denote by $\tilde P_f$ the distribution produced by the sampling method defined in step 1. By \Cref{lemma:rejection_sampling}, we have for $\tilde g \equalDef g + E_n$:
\begin{IEEEeqnarray*}{+rCl+x*}
&& \Dsuplog(P_f, \tilde P_f) \\
& \leq & \min\left\{\Dsuplog(P_f, P_g), \left(1 - \frac{Z_f}{Z_{\tilde g}}\right)^{\lceil n/2\rceil} (\exp(\Dsuplog(P_f, P_g)) - 1)\right\}~. \IEEEyesnumber \label{eq:apx_rs_suplog}
\end{IEEEeqnarray*}
The first bound $\Dsuplog(P_f, P_g)$ already yields the desired bound for $C_{m, d} B n^{-m/d} > 1$. Now, consider the case $C_{m, d} B n^{-m/d} \leq 1$. We have
\begin{IEEEeqnarray*}{+rCl+x*}
Z_f & = & \int_{\calX} e^{f(x)} \diff x \geq \int_{\calX} e^{\tilde g(x) - 2E_n} \diff x \geq \exp(-C_{m, d}Bn^{-m/d}) Z_{\tilde g}~.
\end{IEEEeqnarray*}
Now, the second bound in \Eqref{eq:apx_rs_suplog} yields
\begin{IEEEeqnarray*}{+rCl+x*}
\Dsuplog(P_f, \tilde P_f) & \leq & \left(1 - \exp(-C_{m, d}Bn^{-m/d})\right)^{n/2} (\exp(C_{m,d} B n^{-m/d}) - 1) \\
& \leq & \exp(C_{m,d} B n^{-m/d}) \left(1 - \exp(-C_{m, d}Bn^{-m/d})\right)^{n/2+1} \\
& \leq & e \cdot (C_{m, d} B n^{-m/d})^{n/2+1} \leq O_{m, d}((C_{m, d} B n^{-m/d})^{n/2+1})~. & \qedhere
\end{IEEEeqnarray*}
\end{proof}

Next, we prove corresponding lower bounds in the optimization regime, again using bump functions:

\thmLowerStochSampling*

\begin{proof}
We re-use some results from the proof of \Cref{thm:deterministic_points_rates} in \Cref{sec:appendix:deterministic_points}. We consider again the decomposition of the cube $\calX$ into three slices
\begin{IEEEeqnarray*}{+rCl+x*}
\calC_k \equalDef [k/3, (k+1)/3] \times [0, 1]^{d-1}, \qquad k \in \{0, 1, 2\}~.
\end{IEEEeqnarray*}
Consider a sampling algorithm $\tilde S \in \Aadstoch_n$ with stochastic evaluation points and consider a corresponding random sample $X_f = \phi(N(f, \omega), \omega)$ as defined in \Cref{sec:ibc:stochastic_points}.

\textbf{Step 1.1: Candidate functions for the sup-log distance.} By \Cref{lemma:subcubes}, $\calC_0$ contains $4(n+1)$ disjoint open balls $B_\infty(z_i, \delta_n)$ with radius
\begin{IEEEeqnarray*}{+rCl+x*}
\delta_n \equalDef r_{4(n+1)} = \frac{(4(n+1))^{-1/d}}{12} \geq \frac{n^{-1/d}}{96}~.
\end{IEEEeqnarray*}
Let $f_0 \equiv 0$ be the zero function. Consider the set $\calQ(\omega)$ containing the $n$ random points where $N(f_0, \omega)$ queries $f_0$ and the one random point $\phi(N(f_0, \omega), \omega)$ that the sampling method outputs. We can pick an $i \in \{1, \hdots, 4(n+1)\}$ such that the cube $B_\infty(z_i, \delta_n)$ contains a point from $\calQ(\omega)$ only with probability $\leq 1/4$. With $C_{m, d}$ as in \Cref{lemma:bump_functions}, we define
\begin{IEEEeqnarray*}{+rCl+x*}
f_1(x) \equalDef C_{m, d}^{-1} B \delta_n^{-1} b_{z_i, \delta_n}(x)~,
\end{IEEEeqnarray*}
which satisfies $f_1 \in \calF_{d, m, B}$. Using analogous arguments to the proof of \Cref{thm:deterministic_points_rates} in \Cref{sec:appendix:deterministic_points}, we obtain
\begin{IEEEeqnarray*}{+rCl+x*}
L_{f_1} & = & \log(1 + I_n) \geq \log(1 + \delta_n^d (e^{C_{m, d}^{-1} B \delta_n^m} - 1)) \geq C_{m, d}^{-1} B \delta_n^m + \log(\delta_n^d) \\
& \geq & \tilde c_{m, d} Bn^{-m/d} - \log(n) - d\log(96)
\end{IEEEeqnarray*}
for a suitable constant $\tilde c_{m, d} > 0$.

\textbf{Step 1.2: Bounding the distribution on $f_1$.} Now, the probability of the event $\phi(N(f_0, \omega), \omega) \in B_\infty(z_i, \delta_n)$ is at most $1/4$ by construction. Moreover, the probability of $N(f_0, \omega)$ querying $B_\infty(z_i, \delta_n)$ is also at most $1/4$, hence the probability of $N(f_1, \omega)$ querying $B_\infty(z_i, \delta_n)$ is also at most $1/4$. By the union bound, the probability that $\phi(N(f_1, \omega), \omega) \in B_\infty(z_i, \delta_n)$ is at most $1/2$. Now, to have $\Dsuplog(\tilde S(f_1), P_{f_1}) < \infty$, $\tilde S(f_1)$ must be of the form $P_g$ for some function $g: \calX \to \bbR$. Without loss of generality, we can assume $L_g = 0$. Then, since the set $\hat \calX \equalDef \calX \setminus B_\infty(z_i, \delta_n)$ satisfies $P_g(\hat \calX) \geq 1/2$, there exists $x \in \hat \calX$ with $p_g(x) \geq 1/2$, implying $g(x) \geq \log(1/2)$. But then,
\begin{IEEEeqnarray*}{+rCl+x*}
\Dsuplog(\tilde S(f_1), P_{f_1}) & \geq & |(g(x) - L_g) - (f_1(x) - L_{f_1})| = |g(x) + L_{f_1}| \\
& \geq & \tilde c_{m, d} Bn^{-m/d} - \log(n) - d\log(96) - \log(2) \\
& \geq & \tilde c_{m, d} Bn^{-m/d} - \log(n) - 6d~.
\end{IEEEeqnarray*}
Especially, for $Bn^{-m/d} \geq 12d \tilde c_{m, d}^{-1} (1 + \log(n))$, we have
\begin{IEEEeqnarray*}{+rCl+x*}
\Dsuplog(\tilde S(f_1), P_{f_1}) & \geq & \tilde c_{m, d} Bn^{-m/d} - \frac{1}{2} \tilde c_{m, d} Bn^{-m/d} = \Omega_{m, d}(Bn^{-m/d})~.
\end{IEEEeqnarray*}

\textbf{Step 2.1: Candidate functions for the Wasserstein distance.} By \Cref{lemma:subcubes}, for $M \in \bbN$ to be determined later, we can place $Mn$ subcubes each in $\calC_0$ and $\calC_2$ with radius
\begin{IEEEeqnarray*}{+rCl+x*}
\delta_n \equalDef r_{Mn} \equalDef \frac{(Mn)^{-1/d}}{12}~.
\end{IEEEeqnarray*}
By an analogous argument to Step 1.1, we can find subcubes $B_\infty(z_0, \delta_n)$ and $B_\infty(z_2, \delta_n)$ of $\calC_0$ and $\calC_2$ such that the probability of one of them being queried for $f_0$ is at most $2/M$. Following \Cref{lemma:bump_functions}, we construct the functions
\begin{IEEEeqnarray*}{+rCl+x*}
f_k(x) \equalDef C_{m, d}^{-1} B \delta_n^m b_{z_k, \delta_n}(x), \qquad k \in \{0, 2\}~,
\end{IEEEeqnarray*}
which are contained in $\calF_{d, m, B}$.

\textbf{Step 2.2: Bounding the Wasserstein distance.} We set $M \equalDef 20d$. Since the two subcubes are only queried with probability at most $2/M$, we know that
\begin{IEEEeqnarray*}{+rCl+x*}
\DW(\tilde S(f_0), \tilde S(f_2)) \leq d^{1/2} \DTV(\tilde S(f_0), \tilde S(f_2)) \leq d^{1/2} \frac{2}{M} \leq \frac{1}{10}~.
\end{IEEEeqnarray*}
With an argument analogous to the proof of \Cref{thm:deterministic_points_rates} in \Cref{sec:appendix:deterministic_points}, we obtain
\begin{IEEEeqnarray*}{+rCl+x*}
\DW(P_{f_0}, P_{f_2}) & \geq & \frac{1}{3} \frac{I_n}{1 + I_n} \stackrel{\text{\Cref{lemma:helper_fraction}}}{\geq} \frac{1}{6} \min\{1,I_n\}~, 
\end{IEEEeqnarray*}
and
\begin{IEEEeqnarray*}{+rCl+x*}
I_n & \geq & \delta_n^d (e^{BC_{m, d}^{-1}\delta_n^m} - 1) \geq \tilde c_{m, d} n^{-1} (e^{\tilde c_{m, d} B n^{-m/d}} - 1)
\end{IEEEeqnarray*}
for a suitable constant $\tilde c_{m, d} \in (0, 1)$. Now, suppose that
\begin{IEEEeqnarray*}{+rCl+x*}
Bn^{-m/d} & \geq & \tilde c_{m, d}^{-1} (1 + \log(\tilde c_{m, d}^{-1})) (1 + \log(n))~.
\end{IEEEeqnarray*}
We obtain 
\begin{IEEEeqnarray*}{+rCl+x*}
\tilde c_{m, d} Bn^{-m/d} & \geq & 1 + \log(\tilde c_{m, d}^{-1}) + \log(n) \geq \log(\tilde c_{m, d}^{-1} n + 1)
\end{IEEEeqnarray*}
and therefore
\begin{IEEEeqnarray*}{+rCl+x*}
\DW(P_{f_0}, P_{f_2}) & \geq & \frac{1}{6} \min\{1,I_n\} \geq \frac{1}{6} \min\{1, 1\} = \frac{1}{6}~.
\end{IEEEeqnarray*}
Since $\DW$ satisfies the triangle inequality, there must exist $k \in \{0, 2\}$ with
\begin{IEEEeqnarray*}{+rCl+x*}
\DW(\tilde S(f_k), P_{f_k}) \geq \frac{1}{2} \left(\frac{1}{6} - \frac{1}{10}\right) = \frac{1}{30} = \Omega_{m, d}(1)~.
\end{IEEEeqnarray*}

\textbf{Step 3: TV lower bound.} The corresponding lower bound for the TV distance follows from the inequality $\DW(P, Q) \leq d^{1/2}\DTV(P, Q)$. \qedhere
\end{proof}

Finally, we prove our auxiliary result on the complexity of sampling when the log-partition function is known:

\propExactSampling*

\begin{proof}
For $f \in \calF$, define $\tilde f(x) \equalDef \log(2\exp(f(x)) - 1)$ and $g(x) \equalDef \log(2)$. We have
\begin{IEEEeqnarray*}{+rCl+x*}
Z_{\tilde f} & = & \int_{\calX} (2\exp(f(x)) - 1) \diff x = 2Z_f - 1 = 1~, \\
Z_g & = & \int_{\calX} e^{\log(2)} \diff x = 2~.
\end{IEEEeqnarray*}
Let $\tilde P_f$ be the distribution of $\textsc{RejectionSampling}(\tilde f, g, 1)$, which only uses one evaluation of $\tilde f$ and therefore only one evaluation of $f$. By \Cref{lemma:rejection_sampling}, we have
\begin{IEEEeqnarray*}{+rCl+x*}
\tilde P_f & = & \frac{Z_{\tilde f}}{Z_g} P_{\tilde f} + \left(1 - \frac{Z_{\tilde f}}{Z_g}\right) P_g = \frac{1}{2} P_{\tilde f} + \frac{1}{2} P_g = P_f~. & \qedhere
\end{IEEEeqnarray*}
\end{proof}

\section{Proofs for Relations Between Different Problems} \label{sec:appendix:relations}

The proof of the following theorem adapts results from the literature, showing that they apply to our setting:

\thmMls*

\begin{proof}
\textbf{Step 1: The method.} The idea of the moving least squares method \citep{lancaster_surfaces_1981} is to obtain an approximation $g(x) = g_x(x)$ of $f(x)$ at an evaluation point $x$ by determining $g_x$ as the solution to a polynomial least-squares regression problem with data $(x_i, f(x_i))$, weighted with weights $\Phi(x, x_i)$ that (smoothly) vanish for large $\|x - x_i\|$. We will not state the exact method here but refer to the publications by \cite{li_error_2016} and \cite{mirzaei_analysis_2015}, whose analysis we are using here. While Theorem 4.1 of \cite{li_error_2016} essentially directly provides the result (a), it is unclear to us if the corresponding constants are independent of the evaluation points in the way that we need. Thus, in the following, we will try to verify the slightly stronger conditions of Theorem 3.12 of \cite{mirzaei_analysis_2015} and explain how it can be adapted to our setting with minor modifications.

\textbf{Step 2: Verifying the assumptions.} Now, we list the major assumptions of Theorem 3.12 of \cite{mirzaei_analysis_2015} and show that they are satisfied for a suitable choice of evaluation points and weighting function. The assumptions on smoothness are deferred until Step 3, where we will show how to adapt them to our setting. We define the number $N \equalDef \lfloor n^{1/d} \rfloor$ of grid points along each axis. By dividing each axis into $N$ equal intervals, we obtain a partition of $\calX$ into $N^d$ cubes. Let $X$ be the set of midpoints of these cubes. Hence, $|X| = N^d = \lfloor n^{1/d} \rfloor^d \geq (n^{1/d}/2)^d \geq \Omega_{m, d}(n)$. Here are the assumptions:

\begin{itemize}
\item The considered domain $\Omega$ is a bounded set with Lipschitz boundary. We want to consider $\Omega \equalDef \calX$, which is bounded and has a Lipschitz boundary. 
\item The maximum degree $m$ of the polynomial basis satisfies $m \geq 1$. While $m$ denotes the (known) smoothness of the target function $f$ in our context, we will assume that the maximum degree of the polynomial basis is also $m$. While a maximum degree of $m-1$ should be sufficient for our purposes (as it is in \cite{li_error_2016}), using a maximum degree of $m$ avoids notational confusion and simplifies the adaptation of the arguments of \cite{mirzaei_analysis_2015}.
\item The fill distance $h_{X, \Omega} = \sup_{x \in \Omega} \min_{x' \in X} \|x - x'\|_2$ satisfies $h_{X, \Omega} \leq \min\{h_0, 1\}$ for some given constant $h_0 > 0$. In our case, the fill distance is $h_{X, \Omega} = \sqrt{d}/(2N) = \Theta_{m, d}(n^{-1/d})$, which satisfies the assumption for large enough values of $n$. The errors for smaller $n$ do not affect the asymptotic rate.
\item The weight function is defined through a radial function $\phi: [0, \infty) \to \bbR$, which is supported in $[0, 1]$ and its even extension belongs to $C^m(\bbR)$. For this, we can just use the even and $C^\infty$-smooth bump function $b$ from \Cref{def:bump_functions} and set $\phi(x) \equalDef b(x)$.
\item The point set $X$ is quasi-uniform with constant independent of $f$ and $n$. This means that the separation distance
\begin{IEEEeqnarray*}{+rCl+x*}
q_{X, \Omega} \equalDef \frac{1}{2}\min_{x, x' \in X: x \neq x'} \|x - x'\|_2
\end{IEEEeqnarray*}
satisfies $q_{X, \Omega} \leq h_{X, \Omega} \leq c_{\mathrm{qu}} q_{X, \Omega}$ for a constant $c_{\mathrm{qu}}$ independent of $f$ and $n$. In our case, we have $q_{X, \Omega} = 1/(2N)$, and hence we can set $c_{\mathrm{qu}} \equalDef \sqrt{d}$.
\end{itemize}

\textbf{Step 3: Adapting the argument of \cite{mirzaei_analysis_2015}.} Let $f_n$ be the moving least squares approximation of $f$ with evaluation points $X$. By Corollary 4.5 in \cite{wendland_scattered_2004}, $f_n$ is in $C^m$ since the weight function is also in $C^m$. Hence, the norms $\|f - f_n\|_{C^m}$ and $\|f - f_n\|_{W^m_\infty}$ are equivalent, where $W_p^m(\Omega)$ is the Sobolev space of smoothness $m$ with the $p$-norm applied to the (weak) derivatives. Theorem 3.12 in \cite{mirzaei_analysis_2015} shows that
\begin{IEEEeqnarray*}{+rCl+x*}
\|f - f_n\|_{W_q^{|\alpha|}(\Omega)} \leq C h_{X, \Omega}^{m+s-|\alpha|-d\max\{0, 1/p-1/q\}} \|f\|_{W_p^{m+s}(\Omega)}
\end{IEEEeqnarray*}
for $p \in [1, \infty), q \in [1, \infty], s \in [0, 1)$ and a multi-index $\alpha$ satisfying $m > |\alpha| + d/p$. We would obtain (a) by setting $s = 0$, $p=q=\infty$, and $|\alpha| = k$. However, setting $p=\infty$ is not allowed by the assumptions of the theorem, and setting $|\alpha| = m$ for $m=k$ is also not allowed. Hence, we need to show that the theorem can be extended to $p=\infty$ and $|\alpha| = m$ in the special case $s=0$ and $q=\infty$. The assumption $p < \infty$ is used for the Sobolev extension operator, but it is noted in the proof that $p=\infty$ is allowed for $s=0$. The only other point where $p < \infty$ and $|\alpha| < m$ are required is in the invocation of Eq.~(3.4) in Lemma 3.3 of \cite{mirzaei_analysis_2015}. However, for the special case $s=0$, $p=q=\infty$ and $|\alpha| \leq m$, the statement of Lemma 3.3 also holds, as is shown by the Bramble-Hilbert lemma \citep[cf.\ Lemma (4.3.8) in][]{brenner_mathematical_2008}, which has also been employed by \citep{li_error_2016} for the same purpose.

\textbf{Step 4: Runtime bound.} For (b) and (c), we note that due to the local support of the weight function, evaluating the moving least squares approximation at a point $x \in \calX$ mainly requires the solution of a regression problem with $O_{m, d}(1)$ variables and evaluation points. This is shown, for example, above Lemma 3.6 in \cite{mirzaei_analysis_2015}. The only required pre-computation is allocating an array to store evaluated function values.
\end{proof}

\subsection{Proofs for Relation between Sampling and Log-Partition Estimation} \label{sec:appendix:relations:sampling_logpartition}

For analyzing thermodynamic integration, we are going to use Hoeffding's inequality in the form stated and proved in Theorem 6.10 in \cite{steinwart_support_2008}.

\begin{theorem}[Hoeffding's inequality] \label{thm:hoeffding}
Let $(\Omega, \calA, P)$ be a probability space, $a < b$ be two real numbers, $n \geq 1$ be an integer, and $\xi_1, \hdots, \xi_n: \Omega \to [a, b]$ be independent random variables. Then, for all $\tau > 0$, we have
\begin{IEEEeqnarray*}{+rCl+x*}
P\left(\frac{1}{n} \sum_{i=1}^n (\xi_i - \bbE_P \xi_i) \geq (b-a)\sqrt{\frac{\tau}{2n}} \right) \leq e^{-\tau}~.
\end{IEEEeqnarray*}
\end{theorem}

\thmThermodynamicIntegration*

\begin{proof}
Obviously, we have
\begin{IEEEeqnarray*}{+rCl+x*}
|\tilde L_f - L_f| \leq |L_f - \bbE \tilde L_f| + |\tilde L_f - \bbE \tilde L_f|~.
\end{IEEEeqnarray*}

\textbf{Step 1: Bounding the second term.} For bounding the second term, we use Hoeffding's inequality (\Cref{thm:hoeffding}) with $\xi_i \equalDef f(X_i)$, $n = N$ and $\tau \equalDef \log(2/\delta)$. Ignoring null sets, we can choose $b = \|f\|_\infty$ and $a = -\|f\|_\infty$. We then obtain
\begin{IEEEeqnarray*}{+rCl+x*}
\tilde L_f - \bbE \tilde L_f \geq 2\|f\|_\infty\sqrt{\frac{\log(2/\delta)}{2N}}
\end{IEEEeqnarray*}
with probability $\leq \exp(-\log(2/\delta)) = \delta/2$. By applying the same argument to $\xi_i = -f(X_i)$ and applying the union bound, we obtain
\begin{IEEEeqnarray*}{+rCl+x*}
|\tilde L_f - \bbE \tilde L_f| \leq 2\|f\|_\infty\sqrt{\frac{\log(2/\delta)}{2N}}
\end{IEEEeqnarray*}
with probability $\geq 1 - \delta$.

\textbf{Step 2: Bounding the first term.} We use
\begin{IEEEeqnarray*}{+rCl+x*}
|L_f - \bbE \tilde L_f| & = & \left|\bbE_{\beta \sim \calU([0, 1])} \bbE_{x \sim P_{\beta f}} [f(x)] - \bbE_{\beta \sim \calU([0, 1])} \bbE_{x \sim \tilde P_{\beta f}} [f(x)]\right| \\
& \leq & \sup_{\beta \in \calU([0, 1])} \left|\bbE_{x \sim P_{\beta f}} [f(x)] - \bbE_{x \sim \tilde P_{\beta f}} [f(x)]\right|~. \IEEEyesnumber \label{eq:sup_beta}
\end{IEEEeqnarray*}
We can assume that $|f|_1 \neq 0$ since the bound is clear otherwise. Using the dual formulation of the 1-Wasserstein distance and that $f/|f|_1$ is $1$-Lipschitz, we directly obtain
\begin{IEEEeqnarray*}{+rCl+x*}
\left|\bbE_{x \sim P_{\beta f}} [f(x)] - \bbE_{x \sim \tilde P_{\beta f}} [f(x)]\right| & = & |f|_1 \left|\bbE_{x \sim P_{\beta f}} [f(x)/|f|_1] - \bbE_{x \sim \tilde P_{\beta f}} [f(x)/|f|_1]\right| \\ 
& \leq & |f|_1 \DW(\tilde P_{\beta f}, P_{\beta f})~.
\end{IEEEeqnarray*}
Similarly, the bound on the TV distance follows from an alternative formulation of the TV distance \citep[see e.g.][]{gibbs_choosing_2002} given by
\begin{IEEEeqnarray*}{+rCl+x*}
\DTV(P, Q) & = & \frac{1}{2} \sup_{g: \|g\|_\infty \leq 1} \left| \int g \diff P - \int g \diff Q \right|~. & \qedhere
\end{IEEEeqnarray*}
\end{proof}

\begin{remark} \label{rem:ti_suplog}
In \Cref{thm:thermodynamic_integration}, we can hope for a better bound in terms of the sup-log distance. For example, suppose that $\tilde P_{\beta f} = P_{\beta g}$, where $g$ is an approximation of $f$ that is independent of $\beta$. Since $g$ is only determined up to a constant shift, we can assume that $L_g = L_f$. Then,
\begin{IEEEeqnarray*}{+rCl+x*}
|L_f - \bbE \tilde L_f| & = & \left|\bbE_{\beta \sim \calU([0, 1])} \bbE_{x \sim P_{\beta f}} [f(x)] - \bbE_{\beta \sim \calU([0, 1])} \bbE_{x \sim P_{\beta g}} [f(x)]\right| \\
& \leq & \left|\bbE_{\beta \sim \calU([0, 1])} \bbE_{x \sim P_{\beta f}} [f(x)] - \bbE_{\beta \sim \calU([0, 1])} \bbE_{x \sim P_{\beta g}} [g(x)]\right| \\
&& ~+~ \left|\bbE_{\beta \sim \calU([0, 1])} \bbE_{x \sim P_{\beta g}} [g(x)] - \bbE_{\beta \sim \calU([0, 1])} \bbE_{x \sim P_{\beta g}} [f(x)]\right| \\
& \leq & |L_f - L_g| + \|g - f\|_\infty = 0 + \Dsuplog(P_f, P_g) = \Dsuplog(P_f, P_g)~.
\end{IEEEeqnarray*}
However, in the general case, the approach in \Eqref{eq:sup_beta} of taking the supremum over $\beta$ cannot yield such a good bound. This can be seen by considering indicator functions $f = a \bbone_A$ and $g = (a+\delta) \bbone_A$. Instead, it appears that it would be necessary to obtain bounds depending on $\beta$ and $f$ and show that their integral over $\beta \in [0, 1]$ is sufficiently small for all $f$.
\end{remark}

\thmSampByLog*

\begin{proof}
\textbf{Step 1: Log-density analysis.} We want to show that $\tilde P_f$ has a density $\tilde p_f$ and bound $\|\log \tilde p_f - \log p_f\|_\infty$. Partition $\calX$ into $2^{Md}$ cubes of side length $2^{-M}$. Since a density is only defined up to a null set, it suffices to consider an arbitrary $x$ in the interior of one of these cubes, which we fix in the following. We denote the corresponding cube by $\calZ^{(Md)}$. We can then find exactly one sequence $\calZ^{(0)} = \calX, \calZ^{(1)}, \hdots, \calZ^{(Md)}$ of hyperrectangles which could have been visited during the execution of \Cref{alg:bisection_sampling} to obtain $x$. Since \Cref{alg:bisection_sampling} samples uniformly from $\calZ^{(Md)}$ and the volume of $\calZ^{(Md)}$ is $2^{-Md}$, we have the density
\begin{IEEEeqnarray*}{+rCl+x*}
\tilde p_f(x) = 2^{Md} \tilde P_f(\calZ^{(Md)})~.
\end{IEEEeqnarray*}
On the other hand, a simple integration argument shows that the target density satisfies
\begin{IEEEeqnarray*}{+rCl+x*}
\inf_{x' \in \calZ^{(Md)}} p_f(x') & \leq & 2^{Md} P_f(\calZ^{(Md)}) \leq \sup_{x' \in \calZ^{(Md)}} p_f(x')~.
\end{IEEEeqnarray*}
This yields
\begin{IEEEeqnarray*}{+rCl+x*}
|\log \tilde p_f(x) - \log p_f(x)| & \leq & \left|\log(2^{Md} \tilde P_f(\calZ^{(Md)})) - \log(2^{Md} P_f(\calZ^{(Md)}))\right| \\
&& ~+~ \left|\sup_{x' \in \calZ^{(Md)}} \log p_f(x') - \inf_{x'' \in \calZ^{(Md)}} \log p_f(x'')\right|~.
\end{IEEEeqnarray*}

\textbf{Step 2: Bounding the second term.} Since $\calZ^{(Md)}$ is an axis-aligned cube with side length $2^{-M}$, we have for $x', x'' \in \calZ^{(Md)}$:
\begin{IEEEeqnarray*}{+rCl+x*}
|f(x') - f(x'')| & \leq & \sum_{i=1}^d \|\partial^i f\|_\infty 2^{-M} \leq 2^{-M} d \|f\|_{C^1}~.
\end{IEEEeqnarray*}

\textbf{Step 3: Bounding the first term.} We can simplify
\begin{IEEEeqnarray*}{+rCl+x*}
\left|\log(2^{Md} \tilde P_f(\calZ^{(Md)})) - \log(2^{Md} P_f(\calZ^{(Md)}))\right| = \left|\log(\tilde P_f(\calZ^{(Md)})) - \log(P_f(\calZ^{(Md)}))\right|~.
\end{IEEEeqnarray*}
We want to show by induction over $k \in \{0, \hdots, Md\}$ that
\begin{IEEEeqnarray*}{+rCl+x*}
\left|\log(\tilde P_f(\calZ^{(k)})) - \log(P_f(\calZ^{(k)}))\right| & \leq & 2kE~,
\end{IEEEeqnarray*}
which will then yield the desired error bound for $k=Md$. This is obviously true for $k=0$. Now, suppose that it is true for some $k \in \{0, \hdots, Md-1\}$. Consider a partition of $\calZ^{(k)}$ into two equal-sized sub-hyperrectangles $\calZ_1$ and $\calZ_2$ as in \Cref{alg:bisection_sampling} such that $\calZ^{(k+1)} = \calZ_i$ for some $i \in \{1, 2\}$. Then,
\begin{IEEEeqnarray*}{+rCl+x*}
\tilde P_f(\calZ^{(k+1)}) = \sigma(\tilde L_{f_{\calZ_i}} - \tilde L_{f_{\calZ_{3-i}}}) \tilde P_f(\calZ^{(k)})~,
\end{IEEEeqnarray*}
which also holds for $i=2$ since the sigmoid function $\sigma$ satisfies $\sigma(-u) = 1-\sigma(u)$ for all $u \in \bbR$. Moreover, we have
\begin{IEEEeqnarray*}{+rCl+x*}
P_f(\calZ^{(k+1)}) & = & \frac{P_f(\calZ_i)}{P_f(\calZ_i) + P_f(\calZ_{3-i})} P_f(\calZ^{(k)}) = \frac{e^{L_{f_{\calZ_i}}}}{e^{L_{f_{\calZ_i}}} + e^{L_{f_{\calZ_{3-i}}}}} P_f(\calZ^{(k)}) \\
& = & \sigma(L_{f_{\calZ_i}} - L_{f_{\calZ_{3-i}}}) P_f(\calZ^{(k)})~.
\end{IEEEeqnarray*}
By definition of the functions $f_{\calZ_{i'}}$, $i' \in \{1, 2\}$ in \Cref{alg:bisection_sampling}, since the side-lengths $h_j$ of $\calZ_{i'}$ satisfy $h_j \leq 1$, we have $\|f_{\calZ_{i'}}\|_{C^m} \leq \|f\|_{C^m} \leq B$, which means $f_{\calZ_{i'}} \in \calF_{d, m, B}$. Hence, by assumption, the log-partition error is
\begin{IEEEeqnarray*}{+rCl+x*}
|\tilde L_{f_{\calZ_{i'}}} - L_{f_{\calZ_{i'}}}| \leq \varepsilon~.
\end{IEEEeqnarray*}
Now, the log-sigmoid function $h(u) \equalDef \log \sigma(u)$ satisfies $h'(u) = \frac{\sigma(u)(1-\sigma(u))}{\sigma(u)} = 1-\sigma(u) \in (0, 1)$ and is therefore $1$-Lipschitz. Hence,
\begin{IEEEeqnarray*}{+rCl+x*}
\left|\log\left(\tilde P_f(\calZ^{(k+1)})\right) - \log\left(P_f(\calZ^{(k+1)})\right)\right| & \leq & \left|\log\left(\tilde P_f(\calZ^{(k)})\right) - \log\left(P_f(\calZ^{(k)})\right)\right| \\
&& ~+~\left|h(\tilde L_{f_{\calZ_i}} - \tilde L_{f_{\calZ_{3-i}}}) - h(L_{f_{\calZ_i}} - L_{f_{\calZ_{3-i}}})\right| \\
& \leq & 2kE + 2E = 2(k+1)E~,
\end{IEEEeqnarray*}
which completes the induction.
\end{proof}

\subsection{Proofs for Relation to Optimization} \label{sec:appendix:relations:optimization}

\propOptByApxSampling*

\begin{proof}
\leavevmode
\begin{enumerate}[(a)]
\item Suppose $\Dsuplog(P_f, Q) < \infty$. Then, $Q = P_g$ for some $g$. For almost every $x \in \calX$, we have the implications
\begin{IEEEeqnarray*}{+rCl+x*}
f(x) \leq L_f - \log(1/\delta) - \Dsuplog(P_f, Q) & \Leftrightarrow & \bar f(x) \leq L_{\bar f} - \log(1/\delta) - \|\bar f - \bar g\|_\infty \\
& \Rightarrow & \bar g(x) \leq L_{\bar f} - \log(1/\delta) \\
& \Leftrightarrow & \bar g(x) \leq L_{\bar g} - \log(1/\delta)~.
\end{IEEEeqnarray*}
Hence, 
\begin{IEEEeqnarray*}{+rCl+x*}
&& Q(\{x \in \calX \mid f(x) \leq L_f - \log(1/\delta) - \Dsuplog(P_f, Q)\}) \\
& \leq & Q(\{x \in \calX \mid \bar g(x) \leq L_{\bar g} - \log(1/\delta)\}) \\
& = & P_{\bar g}(\{x \in \calX \mid \bar g(x) \leq L_{\bar g} - \log(1/\delta)\}) \stackrel{\text{\Cref{lemma:lipschitz_maximization_bound}}}{\leq} \delta~.
\end{IEEEeqnarray*}
By using $f/\eps$ instead of $f$ and multiplying both sides of the inequality by $\eps$, we obtain
\begin{IEEEeqnarray*}{+rCl+x*}
Q(\{x \in \calX \mid f(x) \leq \eps L_{f/\eps} - \eps \log(1/\delta) - \eps \Dsuplog(P_{f/\eps}, Q)\}) \leq \delta~.
\end{IEEEeqnarray*}

\item The TV norm bound follows from \Cref{lemma:lipschitz_maximization_bound} because for the considered event $A$,
\begin{IEEEeqnarray*}{+rCl+x*}
Q(A) \leq P_{f/\eps}(A) + \sup_{A'} |Q(A) - P_{f/\eps}(A')| = P_{f/\eps}(A) + \DTV(P_{f/\eps}, Q)~.
\end{IEEEeqnarray*}

\item Let $\tilde \eps > 0$. By definition of the Wasserstein distance, there exist random variables $X \sim P_{f/\eps}$ and $Y \sim Q$ on a common probability space $(\Omega, \calF, P_\Omega)$ such that $\bbE \|X - Y\|_2 \leq \DW(P_{f/\eps}, Q) + \tilde \eps$. By the Markov inequality, we then have
\begin{IEEEeqnarray*}{+rCl+x*}
\|X-Y\|_2 \leq 2(\DW(P_{f/\eps}, Q) + \tilde \eps)/\delta
\end{IEEEeqnarray*}
with probability $\geq 1 - \delta/2$. Moreover, by \Cref{lemma:lipschitz_maximization_bound}, we have
\begin{IEEEeqnarray*}{+rCl+x*}
f(X) > \eps L_{f/\eps} - \eps \log(2/\delta)
\end{IEEEeqnarray*}
with probability $\geq 1 - \delta/2$. By the union bound, we hence have
\begin{IEEEeqnarray*}{+rCl+x*}
f(Y) > f(X) - |f|_1 \|X - Y\|_2 > \eps L_{f/\eps} - \eps \log(2/\delta) - 2\delta^{-1}|f|_1(\DW(P_{f/\eps}, Q) + \tilde \eps)
\end{IEEEeqnarray*}
with probability $\geq 1 - \delta$. Since $\tilde \eps > 0$ was arbitrary, the claim follows. \qedhere
\end{enumerate}
\end{proof}

\section{Proofs for Algorithms} \label{sec:appendix:algorithms}

\subsection{Proofs for Approximation-based Algorithms} \label{sec:appendix:approx_algs}

\subsubsection{Proofs for Piecewise Constant Approximation} \label{sec:appendix:pc_approx}

To study the error of piecewise constant approximation, we study the log-partition function of linear functions~$f$. A first step is achieved using the following lemma:

\begin{lemma} \label{lemma:linear_logpartition}
Let
\begin{IEEEeqnarray*}{+rCl+x*}
r: \bbR \to \bbR, t \mapsto \log\left(\int_{-1/2}^{1/2} \exp(tu) \diff u\right) = \begin{cases}
0 &, t = 0 \\
\log\left(t^{-1}(\exp(t/2) - \exp(-t/2))\right) &, t \neq 0~.
\end{cases}
\end{IEEEeqnarray*}
Then, $r$ is even and $(1/2)$-Lipschitz with $r(t) \geq 0$ for all $t$ and we have more generally for $a > 0$ and $t_1, \hdots, t_d \in \bbR$:
\begin{IEEEeqnarray*}{+rCl+x*}
\log\left(\int_{[-a/2, a/2]^d} \exp\left(\sum_{k=1}^d t_k u_k\right) \diff u\right) = d\log(a) + \sum_{k=1}^d r(at_k)~.
\end{IEEEeqnarray*}
\end{lemma}

\begin{proof}
It follows from a simple symmetry argument that $r$ is even. We have
\begin{IEEEeqnarray*}{+rCl+x*}
\int_{-1/2}^{1/2} \exp(tu) \diff u \geq \int_{-1/2}^{1/2} (1+tu) \diff u = 1~,
\end{IEEEeqnarray*}
which shows $r(t) \geq 0$. Moreover, $\exp(hu) \in [\exp(-h/2), \exp(h/2)]$ for $h > 0$ and $u \in [-1/2, 1/2]$. Using the mean value theorem of integration, we obtain
\begin{IEEEeqnarray*}{+rCl+x*}
r(t+h) - r(t) & = & \log\left(\int_{-1/2}^{1/2} \exp(tu) \exp(hu) \diff u\right) - \log\left(\int_{-1/2}^{1/2} \exp(tu) \diff u\right) \\
& \in & [-h/2, h/2]~,
\end{IEEEeqnarray*}
which shows that $r$ is $1/2$-Lipschitz.

For the more general integral, we use that the integrand is a product of one-dimensional functions to decompose
\begin{IEEEeqnarray*}{+rCl+x*}
\log\left(\int_{[-a/2, a/2]^d} \exp\left(\sum_{k=1}^d t_k u_k\right) \diff u\right) & = & \log\left(\prod_{k=1}^d \int_{-a/2}^{a/2} \exp\left(t_k u_k\right) \diff u_k\right) \\
& \stackrel{\text{Subst.\ $u_k = av_k$}} = & \log\left(\prod_{k=1}^d \int_{-1/2}^{1/2} \exp\left(t_k a v_k\right) a \diff v_k\right) \\
& = & d\log(a) + \sum_{k=1}^d r(at_k)~. & \qedhere
\end{IEEEeqnarray*}
\end{proof}

Another ingredient for the analysis of piecewise constant approximation is to analyze the global error through the errors on individual subcubes:

\begin{lemma} \label{lemma:subcube_bound}
Let $f, g: \calX \to \bbR$ be bounded and measurable. Let $\calX_i$ be a partition of $\calX$. Let $L_f(\calX_i) \equalDef \log \left(\int_{\calX_i} \exp(f(x)) \diff x\right)$. Then,
\begin{IEEEeqnarray*}{+rCl+x*}
\inf_i [L_f(\calX_i) - L_g(\calX_i)] \leq L_f - L_g \leq \sup_i [L_f(\calX_i) - L_g(\calX_i)]~.
\end{IEEEeqnarray*}
\end{lemma}

\begin{proof}
We prove the second inequality here; the first one follows analogously. Let $s \equalDef \sup_i [L_f(\calX_i) - L_g(\calX_i)]$. Then,
\begin{IEEEeqnarray*}{+rCl+x*}
L_f & = & \log\left(\sum_i \exp(L_f(\calX_i))\right) \leq \log\left(\sum_i \exp(s) \exp(L_g(\calX_i))\right) = L_g + s~. & \qedhere
\end{IEEEeqnarray*}
\end{proof}

We now prove convergence rates for piecewise constant approximation, using a combination of different approaches:

\thmPCApprox*

\begin{proof}
Recall from \Cref{sec:pc_approx} that $g_{f, n}$ is piecewise constant on the cubes $\calX_i$, interpolating $f$ in the cube centers $x^{(i)}$.

\textbf{Step 1: Lipschitz-type upper bounds.} Let $f \in \calF_{d, m, B}$. Since $f$ is $Bd^{1/2}$-Lipschitz by \Cref{lemma:lipschitz_constant}, it is easy to see that $\|f - g_{f, n}\|_\infty \leq O_{m, d}(B/N) = O_{m, d}(Bn^{-1/d})$. Then, it follows directly from \Cref{prop:app_bounds} that $|L_f - L_{g_{f, n}}| \leq \|f - g_{f, n}\|_\infty \leq O_{m, d}(Bn^{-1/d})$ and $\Dsuplog(P_f, P_{g_{f, n}}) \leq 2\|f - g_{f, n}\|_\infty \leq O_{m, d}(Bn^{-1/d})$.

\textbf{Step 2: Lipschitz-type lower bound for log-partition with $m=1$.} If $m=1$, it follows that
\begin{IEEEeqnarray*}{+rCl+x*}
\sup_{f \in \calF_{d, m, B}} |L_f - L_{g_{f, n}}| & \geq & e_n(\calF_{d, m, B}, S_L, \Dabs) \stackrel{\text{\Cref{thm:deterministic_points_rates}}}{\geq} \Omega_{m, d}(Bn^{-1/d})~.
\end{IEEEeqnarray*}

\textbf{Step 3: Lipschitz-type lower bound for sampling.} Take $f(x) = \beta (x_1 + \cdots + x_d)$, where $\beta = B/d$, such that $f \in \calF_{d, m, B}$. Pick the cube $\calX_1 = (0, 1/N)^d$. Then, we have
\begin{IEEEeqnarray*}{+rCl+x*}
\Dsuplog(P_f, P_{g_{f, n}}) = \|\bar f - \bar g_{f, n}\|_\infty \geq \frac{1}{2} \left(\sup_{x \in \calX_1} f(x) - \inf_{x \in \calX_1} f(x)\right) = \frac{1}{2} B/N = \frac{1}{2} Bn^{-1/d}~.
\end{IEEEeqnarray*}

\textbf{Step 4: Lower bound for log-partition with $m \geq 2$.} As in Step 3, take $f(x) = \beta (x_1 + \cdots + x_d)$, where $\beta = B/d$, such that $f \in \calF_{d, m, B}$. To prove a lower bound on $L_f - L_{g_{f, n}}$, we follow \Cref{lemma:subcube_bound} and lower-bound the errors $L_f(\calX_i) - L_{g_{f, n}}(\calX_i)$ on individual subcubes $\calX_i$. 

\textbf{Step 4.1: First lower bound.} Fix a subcube $\calX_i$ and set $a = 1/N$ and $t \equalDef \frac{\partial f}{\partial x_k}(x^{(i)}) = (\beta, \hdots, \beta)$. Denote the volume of $\calX_i$ by $V_n = 1/n = a^d$. Since $g_{f, n}$ is constant on $\calX_i$, we have
\begin{IEEEeqnarray*}{+rCl+x*}
&& L_f(\calX_i) - L_{g_{f, n}}(\calX_i) \\
& = & \log\left(\exp(f(x^{(i)})) \int_{\calX_i} \exp(\langle t, x - x^{(i)}\rangle) \diff x\right) - \log\left(V_n \exp(f(x^{(i)}))\right) \\
& \stackrel{\text{\Cref{lemma:linear_logpartition}}}{=} & d\log(a) + \left(\sum_{k=1}^d r(at_k)\right) - \log(V_n) = \sum_{k=1}^d r(at_k) = dr(\beta/N) \\
& \stackrel{\text{\Cref{lemma:pc_exp_lin_function}}}{\geq} & cd \min\{|\beta/N|, |\beta/N|^2\} \geq \Omega_{m, d}(\min\{B n^{-1/d}, B^2 n^{-2/d}\})~,
\end{IEEEeqnarray*}
This lower bound is independent of $i$, hence by \Cref{lemma:subcube_bound}, we obtain
\begin{IEEEeqnarray*}{+rCl+x*}
L_f - L_{g_{f, n}} \geq \Omega_{m, d}(\min\{B n^{-1/d}, B^2 n^{-2/d}\})~.
\end{IEEEeqnarray*}

\textbf{Step 4.2: Second lower bound.} The lower bound above implicitly uses the strong convexity of $\exp(f)$. However, all of the curvature of $\exp(f)$ in the bound above comes from $\exp$ and none from $f$. This is not sufficient in the case $B \ll 1$, where the quadratic dependency on $B$ in $B^2 n^{-2/d}$ is overly optimistic.
For the case $B < 1$, we put the curvature into $f$ by setting $f(x) \equalDef \beta \sum_{k=1}^d x_k^2$. We then have $\partial_k f(x) = 2\beta x_k$ and $\partial_k^2 f(x) = 2\beta$. Hence, we set $\beta \equalDef B/(2d)$ to ensure that $f \in \calF_{d, m, B}$. We now again consider a subcube $\calX_i$, for which we compute
\begin{IEEEeqnarray*}{+rCl+x*}
\int_{\calX_i} \exp(g_{f, n}(x)) \diff x & = & V_n \exp(f(x^{(i)}))~, \\
\int_{\calX_i} \exp(f(x)) \diff x & = & \exp(f(x^{(i)})) \int_{\calX_i} \exp(f(x) - f(x^{(i)})) \diff x \\
& \geq & \exp(f(x^{(i)})) \int_{\calX_i} \left(1 + (f(x) - f(x^{(i)}))\right) \diff x \\
& = & \exp(f(x^{(i)})) \left(V_n + \int_{\calX_i} \left(\sum_{k=1}^d \beta (x_k - x^{(i)}_k)^2\right) \diff x\right) \\
& = & \exp(f(x^{(i)})) \left(V_n + d a^{d-1} \beta \left[\frac{1}{3} u^3\right]_{-a/2}^{a/2}\right) \\
& = & \exp(f(x^{(i)})) V_n \left(1 + d \frac{2}{3 \cdot 2^3} \beta a^2\right)~.
\end{IEEEeqnarray*}
Hence, we have
\begin{IEEEeqnarray*}{+rCl+x*}
L_f(\calX_i) - L_g(\calX_i) & \geq & \log\left(1 + \frac{d}{12} \beta n^{-2/d}\right) = \log\left(1 + \frac{1}{24} Bn^{-2/d}\right) \geq \frac{1}{48} B n^{-2/d}~,
\end{IEEEeqnarray*}
where we used $\log'(x) = 1/x \geq 1/2$ for $x \in [1, 2]$ in the last step.

\textbf{Step 5: Better upper bound for log-partition for $m \geq 2$.} Let $m \geq 2$, $f \in \calF_{d, m, B}$ and $Bn^{-1/d} \leq 1$. We define the piece-wise first-order approximant $h_n$, where for $x$ in the interior of $\calX_i$, we set 
\begin{IEEEeqnarray*}{+rCl+x*}
h_n(x) = f(x^{(i)}) + \langle \nabla f(x^{(i)}), x - x^{(i)}\rangle~.
\end{IEEEeqnarray*}
Our goal is to use
\begin{IEEEeqnarray*}{+rCl+x*}
|L_f - L_{g_{f, n}}| \leq |L_f - L_{h_n}| + |L_{h_n} - L_g|~.
\end{IEEEeqnarray*}

\textbf{Step 5.1: Bounding the first term.}
To bound the first term, we will bound $\|f - h_n\|_\infty$. Let $\delta(t) = f(x^{(i)} + t(x - x^{(i)})) - h_n(x^{(i)} + t(x - x^{(i)}))$. Then, we use Taylor's theorem to bound
\begin{IEEEeqnarray*}{+rCl+x*}
f(x) - h_n(x) & = & \delta(1) = \delta(0) + 1 \cdot \delta'(0) + \frac{1^2}{2} \cdot \delta''(\xi)~,
\end{IEEEeqnarray*}
where $\xi \in (0, 1)$. Since $h_n$ is constructed such that $\delta(0) = \delta'(0) = 0$, we have
\begin{IEEEeqnarray*}{+rCl+x*}
|f(x) - h_n(x)| & = & \frac{1}{2} |\delta''(\xi)| = \frac{1}{2} (x - x^{(i)})^\top [\nabla^2 f(x^{(i)} + \xi (x - x^{(i)}))] (x - x^{(i)}) \\
& \leq & d^2 \|x - x^{(i)}\|_\infty^2 \|f\|_{C^2} \leq \frac{d^2 B}{(2N)^2} = O_{m, d}(BN^{-2}) = O_{m, d}(Bn^{-2/d})~.
\end{IEEEeqnarray*}
This shows $\|f - h_n\|_\infty \leq O_{m, d}(Bn^{-2/d})$ and therefore $|L_f - L_{h_n}| \leq O_{m, d}(Bn^{-2/d})$. 

\textbf{Step 5.2: Bounding the second term.} To bound $|L_{h_n} - L_{g_{f, n}}|$, we follow \Cref{lemma:subcube_bound} and bound the errors $|L_{h_n}(\calX_i) - L_{g_{f, n}}(\calX_i)|$ on individual subcubes $\calX_i$. Fix a subcube $\calX_i$ and set $a = 1/N$ and $t \equalDef \frac{\partial f}{\partial x_k}(x^{(i)})$. Denote the volume of $\calX_i$ by $V_n = 1/n = a^d$. Since $g_{f, n}$ is constant on $\calX_i$, we have
\begin{IEEEeqnarray*}{+rCl+x*}
&& |L_{h_n}(\calX_i) - L_{g_{f, n}}(\calX_i)| \\
& = & \left|\log\left(\exp(f(x^{(i)})) \int_{\calX_i} \exp(\langle t, x - x^{(i)}\rangle) \diff x\right) - \log\left(V_n \exp(f(x^{(i)}))\right)\right| \\
& \stackrel{\text{\Cref{lemma:linear_logpartition}}}{=} & \left|d\log(a) + \left(\sum_{k=1}^d r(at_k)\right) - \log(V_n)\right| = \left|\sum_{k=1}^d r(at_k)\right| \\
& \stackrel{\text{\Cref{lemma:pc_exp_lin_function}}}{\leq} & \sum_{k=1}^d C\min\{|at_k|, |at_k|^2\} \leq Cd\min\{B n^{-1/d}, B^2 n^{-2/d}\}~,
\end{IEEEeqnarray*}
where we used $|t_k| \leq \|f\|_{C^m}$ and $a = 1/N = n^{-1/d}$ in the last step. Using \Cref{lemma:subcube_bound}, we now obtain
\begin{IEEEeqnarray*}{+rCl+x*}
|L_{h_n} - L_{g_{f, n}}| & \leq & Cd\min\{B n^{-1/d}, B^2 n^{-2/d}\}~,
\end{IEEEeqnarray*}
which concludes the upper bound.
\end{proof}

The following two lemmas provide some additional bounds that have been used in the previous proof:

\begin{lemma} \label{lemma:exp_diff_quot}
Let
\begin{IEEEeqnarray*}{+rCl+x*}
H: \bbR \to \bbR, t & \mapsto & \sum_{k=0}^\infty \frac{t^k}{(k+1)!} = \begin{cases}
\frac{\exp(t) - 1}{t} &, t \neq 0 \\
1 &, t = 0~.
\end{cases} \\
h: \bbR \to \bbR, t & \mapsto & \sum_{k=0}^\infty \frac{t^k}{(k+2)!} = \begin{cases}
\frac{H(t)-1}{t} &, t \neq 0 \\
1/2 &, t = 0~.
\end{cases}
\end{IEEEeqnarray*}
Then, $H$ and $h$ are $C^\infty$. Moreover, 
\begin{itemize}
\item $H$ and $H'$ are increasing on $[0, \infty)$ with $H(0) = 1$, $H'(0) = 1/2$, and $H(t) > 0$ for all $t \in (0, \infty)$.
\item $h$ and $h'$ are increasing on $[0, \infty)$ with $h(0) = 1/2$, $h'(0) = 1/6$.
\end{itemize}
\end{lemma}

\begin{proof}
Using the series representation, it follows that $H$ and $h$ are $C^\infty$. We also directly obtain $H(0) = 1$ and $H'(0) = 1/2$ as well as $h(0) = 1/2$ and $h'(0) = 1/6$. Moreover, it follows that $H', H'' \geq 0$ on $[0, \infty]$, which implies that $H'$ and $H$ are increasing on $[0, \infty)$. The inequality $H(t) > 0$ follows from the non-series representation together with $H(0) = 1$. The results for $h$ can be derived analogously.
\end{proof}

\begin{lemma} \label{lemma:pc_exp_lin_function}
Consider the function $r$ from \Cref{lemma:linear_logpartition}. Then, there exist constants $c, C > 0$ such that
\begin{IEEEeqnarray*}{+rCl+x*}
c \min\{|t|, |t|^2\} \leq r(t) \leq C \min\{|t|, |t|^2\}
\end{IEEEeqnarray*}
for all $t \in \bbR$. 
\end{lemma}

\begin{proof}
Since $r$ is an even function, it suffices to prove the inequalities for $t \geq 0$.

\textbf{Step 1: Simplifying the derivative.} First, we compute the derivative of $r$ for $t \neq 0$:
\begin{IEEEeqnarray*}{+rCl+x*}
r'(t) & = & \frac{t}{\exp(t/2) - \exp(-t/2)} \cdot \left(\frac{\exp(t/2) + \exp(-t/2)}{2t} - \frac{\exp(t/2) - \exp(-t/2)}{t^2}\right) \\
& = & \frac{1}{2} \frac{\exp(t) + 1}{\exp(t) - 1} - \frac{1}{t} = \frac{1}{2} + \frac{1}{\exp(t) - 1} - \frac{1}{t} \\
& = & \frac{1}{2} + \frac{1}{t}\left(\frac{1}{\left(\frac{\exp(t) - 1}{t}\right)} - 1\right) = \frac{1}{2} + \frac{1}{t}\left(\frac{1}{H(t)} - 1\right) = \frac{1}{2} + \frac{1}{t} \cdot \frac{1-H(t)}{H(t)} \\
& = & \frac{1}{2} - \frac{h(t)}{H(t)}~, \IEEEyesnumber \label{eq:r_derivative}
\end{IEEEeqnarray*}
where we used the functions $h$ and $H$ from \Cref{lemma:exp_diff_quot}. Since $h$ and $H$ are also continuous in $t = 0$, the equation
\begin{IEEEeqnarray*}{+rCl+x*}
r'(t) = \frac{1}{2} - \frac{h(t)}{H(t)}
\end{IEEEeqnarray*}
holds for all $t \in \bbR$.

\textbf{Step 2: Upper bound.} Since $r$ is $1/2$-Lipschitz, we obtain $r(t) = r(t) - r(0) \leq t/2$ for $t \geq 0$. For $t \in [0, 1]$, we can use $r'(0) = 0$ to obtain
\begin{IEEEeqnarray*}{+rCl+x*}
r(t) = r(0) + t r'(0) + \frac{t^2}{2} r''(\xi) \leq \frac{t^2}{2} \sup_{u \in [0, 1]} r''(u) \leq Ct^2
\end{IEEEeqnarray*} 
for some $\xi \in [0, t]$ and the constant $C = \frac{1}{2}\sup_{u \in [0, 1]} r''(u) > 0$. This shows $r(t) \leq O(\min\{t, t^2\})$.

\textbf{Step 3: Lower bound.} We can now use \Cref{lemma:exp_diff_quot} to further simplify for $t \geq 0$
\begin{IEEEeqnarray*}{+rCl+x*}
r'(t) & = & \frac{1}{2} - \frac{h(t)}{H(t)} \geq \frac{1}{2} - \frac{1/2}{H(t)} = \frac{1}{2} \left(1 - \frac{1}{H(t)}\right) \defEqual \tilde r(t)~.
\end{IEEEeqnarray*}
We find that
\begin{IEEEeqnarray*}{+rCl+x*}
\tilde r'(t) & = & \frac{H'(t)}{2H(t)^2}
\end{IEEEeqnarray*}
satisfies $\tilde r'(0) = 1/4$ and $\tilde r'(t) > 0$ for all $t \in [0, \infty)$. 

Set $\tilde c \equalDef \inf_{t \in [0, 1]} \tilde r'(t) > 0$. Since $H(0) = 1$, we have $\tilde r(0) = 0$. For $t \in [0, 1]$, this yields
\begin{IEEEeqnarray*}{+rCl+x*}
\tilde r(t) = r(0) + \int_0^t \tilde r(u) \diff u \geq \tilde ct~.
\end{IEEEeqnarray*}
For $t > 1$, since $\tilde r$ is increasing, we obtain $\tilde r(t) \geq \tilde c$. In total, this yields
\begin{IEEEeqnarray*}{+rCl+x*}
\tilde r(t) \geq \tilde c\min\{1, t\} \qquad \text{ for } t \in [0, \infty)~.
\end{IEEEeqnarray*}
Now, we obtain for $t \in [0, 1]$
\begin{IEEEeqnarray*}{+rCl+x*}
r(t) & = & r(0) + \int_0^t r'(u) \diff u \geq \int_0^t \tilde r(u) \diff u \geq \int_0^t \tilde c\min\{1, u\} \diff u = \frac{\tilde c}{2} t^2
\end{IEEEeqnarray*}
and for $t > 1$
\begin{IEEEeqnarray*}{+rCl+x*}
r(t) & = & r(1) + \int_1^t r'(u) \diff u \geq r(1) + \int_1^t \tilde r(u) \diff u \geq r(1) + \int_1^t \tilde c\min\{1, u\} \diff u \\
& \geq & \frac{\tilde c}{2} + \tilde c(t-1) \geq \frac{\tilde c}{2} t~. & \qedhere
\end{IEEEeqnarray*}
\end{proof}

\subsubsection{Proofs for Density-based Approximation} \label{sec:appendix:density_based}

In the following, we analyze how log-partition and sampling errors can be bounded in terms of the underlying unnormalized densities:

\PropDensityApxBound*

\begin{proof}
\textbf{Step 1: Log-partition function.} We have
\begin{IEEEeqnarray*}{+rCl+x*}
\log I_p - \log I_q & \leq & \log I_p - \log \left(\int_{\calX} (p(x) - \|p - q\|_\infty) \diff x\right) = \log \left(\frac{I_p}{I_p - \|p - q\|_\infty}\right) \\
& = & \log \left(\frac{1}{1 - \|p - q\|_\infty /I_p}\right) \\
\log I_q - \log I_p & \leq & \log \left(\int_{\calX} (p(x) + \|p - q\|_\infty) \diff x\right) - \log I_p = \log \left(\frac{I_p + \|p - q\|_\infty}{I_p}\right) \\
& = & \log (1 + \|p - q\|_\infty/I_p) \\
& \leq & \log \left(\frac{1}{1 - \|p - q\|_\infty /I_p}\right).
\end{IEEEeqnarray*}

\textbf{Step 2: Total variation distance.}
Without loss of generality, assume that $I_q \leq I_p = 1$, such that $\max\{I_p, I_q\} = 1$. Define the normalized density function $\tilde q \equalDef q/I_q$. Then,
\begin{IEEEeqnarray*}{+rCl+x*}
\DTV(P, Q) & = & \frac{1}{2}\int_{\calX} |p(x) - \tilde q(x)| \diff x \leq \frac{1}{2}\int_{\calX} (|p(x) - q(x)| + q(x)|(1 - 1/I_q)|) \diff x \\
& \leq & \frac{1}{2} \|p - q\|_\infty + \frac{1}{2} I_q |1 - 1 / I_q| \leq \|p - q\|_\infty~,
\end{IEEEeqnarray*}
where we have used $I_q |1 - 1/I_q| = |I_q - 1| = |I_q - I_p| \leq \|p - q\|_\infty$ in the last step.
\end{proof}

Next, we turn to bounding $\|p_f\|$ in terms of $\|f\|$. Our first result will provide a bound for the sup-norm:

\begin{lemma} \label{lemma:sup_density}
For $f \in C^1(\calX)$, we have
\begin{IEEEeqnarray*}{+rCl+x*}
\|p_f\|_\infty \leq O_{m, d}(\max\{1, \|f\|_{C^1}\}^d)~.
\end{IEEEeqnarray*}
\end{lemma}

\begin{proof}
By \Cref{lemma:lipschitz_constant}, we have $|f|_1 \leq d^{1/2} \|f\|_{C^1}$, and hence
\begin{IEEEeqnarray*}{+rCl+x*}
\|p_f\|_\infty & = & \frac{\exp(M_f)}{Z_f} = \exp(M_f - L_f) \\
& \stackrel{\text{\Cref{lemma:lipschitz_maximization_bound}}}{\leq} & \exp(d\log(1+3d^{-1/2}|f|_1)) = (1 + 3d^{-1/2}|f|_1)^d \\
& \leq & O_{m, d}(\max\{1, \|f\|_{C^1}\}^d)~. & \qedhere
\end{IEEEeqnarray*}
\end{proof}

The following lemma helps to bound the norms of products, which occur in the derivatives of $\exp(f)$:

\begin{lemma} \label{lemma:norm_function_product}
Let $f, g \in C^m(\calX)$ for $m \geq 0$. Then, $\|fg\|_{C^m} \leq 2^m \|f\|_{C^m} \|g\|_{C^m}$.

\begin{proof}
We use induction on $m$. This claim is obviously true for $m=0$. Now suppose it is true for $m$ and that $f, g \in C^{m+1}(\calX)$. Take any $\alpha \in \bbN_0^d$ with $|\alpha|_1 = m+1$. Then, we can write $\partial_\alpha = \partial_\beta \partial_j$ for some $j \in \{1, \hdots, d\}$ and $\beta \in \bbN_0^d$ with $|\beta| = m$. We then have
\begin{IEEEeqnarray*}{+rCl+x*}
\|\partial_\alpha (fg)\|_\infty & = & \|\partial_\beta \partial_j (fg)\|_\infty = \|\partial_\beta ((\partial_j f)g + f(\partial_j g))\|_\infty \\
& \leq & \|(\partial_j f)g + f(\partial_j g)\|_{C^m} \leq 2^m \left(\|\partial_j f\|_{C^m} \|g\|_{C^{m}} + \|f\|_{C^{m}} \|\partial_j g\|_{C^{m}}\right) \\
& \leq & 2^{m+1} \|f\|_{C^{m+1}} \|g\|_{C^{m+1}}~.
\end{IEEEeqnarray*}
Moreover, for any $\alpha \in \bbN_0^d$ with $|\alpha|_1 \leq m$, we have
\begin{IEEEeqnarray*}{+rCl+x*}
\|\partial_\alpha (fg)\|_\infty \leq \|fg\|_{C^m} \leq 2^m \|f\|_{C^m} \|g\|_{C^m} \leq 2^{m+1} \|f\|_{C^{m+1}} \|g\|_{C^{m+1}}~.
\end{IEEEeqnarray*}
This completes the proof of the induction step.
\end{proof}
\end{lemma}

Now, we can indeed bound higher-order norms of $\exp(f)$:

\begin{lemma} \label{lemma:norm_exp_f}
Let $f \in C^m(\calX)$ for $m \geq 0$. Then, 
\begin{IEEEeqnarray*}{+rCl+x*}
\|e^f\|_{C^m} \leq 2^{m(m-1)/2} \max\{1, \|f\|_{C^m}\}^m \|e^f\|_\infty~.
\end{IEEEeqnarray*}

\begin{proof}
We prove this by induction on $m$. The claim is obviously true for $m=0$. Now, suppose the claim is true for some $m \geq 0$ and let $f \in C^{m+1}(\calX)$. Take any $\alpha \in \bbN_0^d$ with $|\alpha|_1 = m+1$. Then, we can write $\partial_\alpha = \partial_\beta \partial_j$ for some $j \in \{1, \hdots, d\}$ and $\beta \in \bbN_0^d$ with $|\beta| = m$. Thus,
\begin{IEEEeqnarray*}{+rCl+x*}
\|\partial_\alpha e^f\|_\infty & = & \|\partial_\beta \partial_j e^f\|_\infty = \|\partial_\beta (\partial_j f) e^f\|_\infty \\
& \leq & \|(\partial_j f) e^f\|_{C^m} \stackrel{\text{\Cref{lemma:norm_function_product}}}{\leq} 2^m \|\partial_j f\|_{C^m} \|e^f\|_{C^m} \\
& \leq & 2^m \|f\|_{C^{m+1}} 2^{m(m-1)/2} \max\{1, \|f\|_{C^m}\}^m \|e^f\|_\infty \\
& \leq & 2^{(m+1)m/2}\max\{1, \|f\|_{C^{m+1}}\}^{m+1} \|e^f\|_\infty~.
\end{IEEEeqnarray*}
Moreover, for any $\alpha \in \bbN_0^d$ with $|\alpha|_1 \leq m$, we have by the induction hypothesis
\begin{IEEEeqnarray*}{+rCl+x*}
\|\partial_\alpha e^f\|_\infty & \leq & \|e^f\|_{C^m} \leq 2^{m(m-1)/2} \max\{1, \|f\|_{C^m}\}^m \|e^f\|_\infty \\
& \leq & 2^{(m+1)m/2} \max\{1, \|f\|_{C^{m+1}}\}^{m+1} \|e^f\|_\infty~.
\end{IEEEeqnarray*}
This completes the proof of the induction step.
\end{proof}
\end{lemma}

It might be possible to improve the dependence on $m$ in the previous lemma; we ignored this since we do not study the dependence on $m$. Combining the previous lemmas, we arrive at a higher-order norm bound for the density:

\thmDensityNorm*

\begin{proof}
\textbf{Step 1: Upper bound.} For $f \in \calF_{d, m, B}$, we have
\begin{IEEEeqnarray*}{+rCl+x*}
\|p_f\|_{C^m} & = & \|e^f/Z_f\|_{C^m} = \frac{1}{Z_f} \|e^f\|_{C^m} \\
& \stackrel{\text{\Cref{lemma:norm_exp_f}}}{\leq} & 2^{m(m-1)/2} \max\{1, \|f\|_{C^m}\}^m \|e^f/Z_f\|_\infty \\
& \stackrel{\text{\Cref{lemma:sup_density}}}{\leq} & 2^{m(m-1)/2} \max\{1, \|f\|_{C^m}\}^m O_{d}(\max\{1, \|f\|_{C^1}\}^d) \\
& \leq & O_{m, d}(\max\{1, \|f\|_{C^m}\}^{d+m})~.
\end{IEEEeqnarray*}

\textbf{Step 2: Lower bound.} Consider $f: \calX \to \bbR, x \mapsto Bd^{-1}(x_1 + \cdots + x_d)$. A simple calculation yields $\|f\|_{C^m} = B$, hence $f \in \calF_{d, m, B}$. Moreover, we have $M_f = f(1, \hdots, 1) = B$. We can also calculate
\begin{IEEEeqnarray*}{+rCl+x*}
Z_f & = & \int_{\calX}\left(\prod_{i=1}^d \exp(Bd^{-1}x_i)\right) \diff x = \left(\int_0^1 \exp(Bd^{-1} x_1) \diff x_1\right)^d \\
& = & (B^{-1}d [\exp(Bd^{-1}) - 1])^d \leq (B^{-1} d \exp(Bd^{-1}))^d = (B^{-1} d)^d \exp(B)~.
\end{IEEEeqnarray*}
Finally, we compute
\begin{IEEEeqnarray*}{+rCl+x*}
\partial_{x_1}^m p_f(x)\vert_{x=(1, \hdots, 1)} & = & \partial_{x_1}^m \frac{\exp(f(x))}{Z_f} \Bigg\vert_{x=(1, \hdots, 1)} = (Bd^{-1})^m \frac{f(1, \hdots, 1)}{Z_f} \geq (Bd^{-1})^{d+m}~.
\end{IEEEeqnarray*}
We know that $\sup_{x \in \calX} p_f(x) \geq 1$ because $p_f$ is a density on a unit-volume domain. Hence, we conclude $\|p_f\|_{C^m} \geq \Omega_{m, d}(\max\{1, B\}^{d+m})$.
\end{proof}

\subsection{Proofs for Simple Stochastic Algorithms}

\subsubsection{Proofs for Rejection Sampling} \label{sec:appendix:rs}

The following bound for rejection sampling with uniform proposal distribution is a consequence of the general rejection sampling bounds in \Cref{lemma:rejection_sampling}.

\propRejectionSampling*

\begin{proof}
\textbf{Step 1: Rejection probability.} Set $g(x) \equalDef M_f$. Since
\begin{IEEEeqnarray*}{+rCl+x*}
\|p_f\|_\infty = \frac{\exp(M_f)}{Z_f} = \frac{Z_g}{Z_f}~,
\end{IEEEeqnarray*}
the overall rejection probability $p_R$ from \Cref{lemma:rejection_sampling} satisfies
\begin{IEEEeqnarray*}{+rCl+x*}
p_R \stackrel{\text{\Cref{lemma:rejection_sampling}}}{\leq} \exp(-nZ_f/Z_g) = \exp(-n/\|p_f\|_\infty)~.
\end{IEEEeqnarray*}

\textbf{Step 2: Sup-log distance.} We have $\Dsuplog(P_f, P_g) = \|\bar f - \bar g\|_\infty = \|\bar f\|_\infty \leq 2\|f\|_\infty$. We then obtain from \Cref{lemma:rejection_sampling} that
\begin{IEEEeqnarray*}{+rCl+x*}
\Dsuplog(P_f, \tilde P_f) & \leq & \min\{\Dsuplog(P_f, P_g), p_R(\exp(\Dsuplog(P_f, P_g)) - 1)\} \\
& \leq & \min\left\{2\|f\|_\infty, \exp\left(2\|f\|_\infty - n/\|p_f\|_\infty\right)\right\}~.
\end{IEEEeqnarray*}

\textbf{Step 3: TV distance.} For the TV distance, we compute 
\begin{IEEEeqnarray*}{+rCl+x*}
\DTV(P_f, P_g) \leq \Dsuplog(P_f, P_g) \leq 2\|f\|_\infty
\end{IEEEeqnarray*}
and also employ the trivial bound $\DTV(P_f, P_g) \leq 1$. From \Cref{lemma:rejection_sampling}, we obtain
\begin{IEEEeqnarray*}{+rCl+x*}
\DTV(P_f, \tilde P_f) & = & p_R \DTV(P_f, P_g) \leq \exp(-n/\|p_f\|_\infty) \min\{1, 2\|f\|_\infty\}~.
\end{IEEEeqnarray*}
Using that $\exp(-x) \leq O_{m, d}(x^{-m/d})$ for $x > 0$, we obtain
\begin{IEEEeqnarray*}{+rCl+x*}
\exp(-n/\|p_f\|_\infty) & \leq & \|p_f\|_\infty^{m/d} n^{-m/d} \stackrel{\text{\Cref{thm:density_norm}}}{\leq} O_{m, d}\left(\max\{1, \|f\|_{C^1}\}^m n^{-m/d}\right)~. & \qedhere
\end{IEEEeqnarray*}
\end{proof}

\subsubsection{Proofs for Monte Carlo Log-partition} \label{sec:appendix:mc_logpartition}

For analyzing the Monte Carlo log-partition estimator, we are going to use Bernstein's inequality in the form stated and proved in Theorem 6.10 in \cite{steinwart_support_2008}:

\begin{theorem}[Bernstein's inequality]
Let $(\Omega, \calA, P)$ be a probability space, $B > 0, \sigma > 0, n \geq 1$. Moreover, let $\xi_1, \hdots, \xi_n: \Omega \to \bbR$ be independent random variables with
\begin{itemize}
\item $\bbE_P \xi_i = 0$
\item $\|\xi_i\|_\infty \leq B$
\item $\bbE_P \xi_i^2 \leq \sigma^2$
\end{itemize}
for all $i \in \{1, \hdots, n\}$. Then,
\begin{IEEEeqnarray*}{+rCl+x*}
P\left(\frac{1}{n} \sum_{i=1}^n \xi_i \geq \sqrt{\frac{2\sigma^2 \tau}{n}} + \frac{2B \tau}{3n}\right) \leq e^{-\tau}
\end{IEEEeqnarray*}
for all $\tau > 0$.
\end{theorem}

We will use Bernstein's inequality since it yields better bounds than Hoeffding's inequality (\Cref{thm:hoeffding}) when $n$ is large and $\sigma$ is significantly smaller than $B$, which will be the case in the following proof.

\thmMCLogPartition*

\begin{proof}
Without loss of generality, we assume that $f$ is shifted such that $M_f = 0$. 

\begin{enumerate}[(a)]
\item \textbf{Step A.1: Simple one-sided bound.} We have
\begin{IEEEeqnarray*}{+rCl+x*}
\tilde L_n - L_f \leq 0 - L_f \stackrel{\text{\Cref{lemma:lipschitz_maximization_bound}}}{\leq} d\log(1+3d^{-1/2}|f|_1)~.
\end{IEEEeqnarray*}

\textbf{Step A.2: Bounding the other side.} Define the empirical maximum 
\begin{IEEEeqnarray*}{+rCl+x*}
M_n \equalDef \max_{i \in \{1, \hdots, n\}} f(X_i)~.
\end{IEEEeqnarray*}
Since $\tilde L_n = \log(\sum_{i=1}^n \exp(f(X_i))) - \log(n) \geq M_n - \log(n)$, we obtain
\begin{IEEEeqnarray*}{+rCl+x*}
L_f - \tilde L_n \leq 0 - \tilde L_n \leq \log(n) - M_n \leq \log(4\log(2/\delta)) + d\log(1+3d^{-1/2}|f|_1) - M_n~.
\end{IEEEeqnarray*}
It remains to provide a lower bound on $M_n$. Define $R \equalDef (\log(1/\delta))^{1/d} n^{-1/d}$.
\begin{itemize}
\item \textbf{Case 1: $R \geq 1$.} In this case, for all $x \in \calX$, we have
\begin{IEEEeqnarray*}{+rCl+x*}
f(x) = f(x) - M_f \geq - d^{1/2} |f|_1 \geq -d^{1/2}|f|_1 R = -d^{1/2} (\log(1/\delta))^{1/d} |f|_1 n^{-1/d}~,
\end{IEEEeqnarray*}
which implies that
\begin{IEEEeqnarray*}{+rCl+x*}
-M_n \leq d^{1/2} (\log(1/\delta))^{1/d} |f|_1 n^{-1/d}
\end{IEEEeqnarray*}
with probability $1$.
\item \textbf{Case 2: $R \leq 1$.} Let $x^*$ be a maximizer of $f$. In this case, there exists an axis-aligned subcube $\tilde \calX$ of $\calX$ with side length $R$ containing $x^*$. For each $x \in \tilde \calX$, we have $f(x) = f(x) - f(x^*) \geq -|f|_1 d^{1/2} R$. Moreover, 
\begin{IEEEeqnarray*}{+rCl+x*}
P(M_n \leq -|f|_1 d^{1/2} R) & = & P(f(X_1) \leq -|f|_1 d^{1/2} R)^n \\
& \leq & P(X_1 \notin \tilde\calX)^n = (1-P(X_1 \in \tilde\calX))^n = (1-R^d)^n \\
& \leq & e^{-R^d n} = \delta~.
\end{IEEEeqnarray*}
\end{itemize}

\item Define $\tau \equalDef \log(2/\delta)$. To apply Bernstein's inequality to $\xi_i \equalDef e^{f(X_i)} - e^{L_f}$, we need a sup-bound and a variance bound.

\textbf{Step B.1: Sup-bound.} Since $e^{f(\calX)} \subseteq [0, 1]$ by assumption on $f$, we have $\|\xi_i\|_\infty \leq 1$.

\textbf{Step B.2: Variance bound.} Since $e^{f(\calX)} \subseteq [0, 1]$, we have:
\begin{IEEEeqnarray*}{+rCl+x*}
\Var(\xi_i) = \Var e^{f(X_i)} \leq \bbE [(e^{f(X_i)})^2] \leq \bbE e^{f(X_i)} = e^{L_f}~.
\end{IEEEeqnarray*}

\textbf{Step B.3: Concentration of $S_n$.} By Bernstein's inequality, we obtain for $\tau \in (0, 1)$:
\begin{IEEEeqnarray*}{+rCl+x*}
P\left(|S_n - e^{L_f}| \geq \sqrt{\frac{2e^{L_f}\tau}{n}} + \frac{2\tau}{3n}\right) \leq 2e^{-\tau} = \delta~.
\end{IEEEeqnarray*}

\textbf{Step B.4: Lower bound on expectation.} From \Cref{lemma:lipschitz_maximization_bound}, we obtain
\begin{IEEEeqnarray*}{+rCl+x*}
e^{L_f} & = & e^{L_f - M_f} \geq e^{-d\log(1+3d^{-1/2}|f|_1)} = (1+3d^{-1/2}|f|_1)^{-d}~.
\end{IEEEeqnarray*}

\textbf{Step B.5: Concentration of $\tilde L_n$.} 
By combining the previous steps, we obtain with probability $\geq 1 - 2e^{-\tau}$:
\begin{IEEEeqnarray*}{+rCl+x*}
\left|\frac{S_n - e^{L_f}}{e^{L_f}}\right| & \leq & \sqrt{\frac{2\tau}{n(1+3d^{-1/2}|f|_1)^{-d}}} + \frac{2\tau}{3n(1+3d^{-1/2}|f|_1)^{-d}}~.
\end{IEEEeqnarray*}
Since we assumed $n \geq 4\log(2/\delta)(1+3d^{-1/2}|f|_1)^d = 4\tau(1+3d^{-1/2}|f|_1)^d$, the right-hand-side is less than $1/2$. Since the logarithm is $2$-Lipschitz on $[1/2, 3/2]$, we obtain
\begin{IEEEeqnarray*}{+rCl+x*}
|\tilde L_n - L_f| & = & \left|\log\left(\frac{S_n}{e^{L_f}}\right)\right| \\
& = & \left|\log\left(1 + \frac{S_n - e^{L_f}}{e^{L_f}}\right) - \log(1)\right| \\
& \leq & 2\left|\frac{S_n - e^{L_f}}{e^{L_f}}\right| \\
& \leq & 2\left(\sqrt{\frac{2\tau}{n(1+3d^{-1/2}|f|_1)^{-d}}} + \frac{2\tau}{3n(1+3d^{-1/2}|f|_1)^{-d}}\right) \\
\end{IEEEeqnarray*}
By using the assumption $n \geq 4\log(2/\delta)(1+3d^{-1/2}|f|_1)^d = 4\tau(1+3d^{-1/2}|f|_1)^d$ from (b), we can further bound
\begin{IEEEeqnarray*}{+rCl+x*}
|\tilde L_n - L_f| & \leq & 2\Bigg(\sqrt{\frac{2\tau}{n(1+3d^{-1/2}|f|_1)^{-d}}} \\
&& \quad ~+~ \frac{2\tau}{3n^{1/2} (4\tau (1+3d^{-1/2}|f|_1)^{d})^{1/2} (1+3d^{-1/2}|f|_1)^{-d}}\Bigg) \\
& = & (\sqrt{8} + 2/3)\sqrt{\frac{\tau}{n(1+3d^{-1/2}|f|_1)^{-d}}} \\
& \leq & 4 (1+3d^{-1/2}|f|_1)^{d/2} \tau^{1/2} n^{-1/2}~. & \qedhere
\end{IEEEeqnarray*}
\end{enumerate}
\end{proof}

\subsubsection{Proofs for Monte Carlo Sampling} \label{sec:appendix:mc_sampling}

We now prove a simple lower bound for a simple Monte Carlo sampling algorithm. We use the TV distance, but the general approach could also be used to prove lower bounds for the sup-log and 1-Wasserstein distances.

\thmMCSampling*

\begin{proof}
Set $f(x) \equalDef -\beta(x_1 + \cdots + x_d)$, where $\beta \equalDef Bd^{-1}$, such that $f \in \calF_{d, m, B}$. For $\delta \in (0, 1]$, set $\calX_\delta \equalDef [0, \delta]^d$. Then, similar as in \Cref{lemma:linear_logpartition}, we can compute
\begin{IEEEeqnarray*}{+rCl+x*}
Z(\delta) \equalDef \int_{\calX_\delta} \exp(f(x)) \diff x = \left(\frac{1 - \exp(-\beta \delta)}{\beta}\right)^d~.
\end{IEEEeqnarray*}
We then obtain
\begin{IEEEeqnarray*}{+rCl+x*}
P_f(\calX_\delta) = \frac{Z(\delta)}{Z(1)} = \left(\frac{1 - \exp(-\beta \delta)}{1 - \exp(-\beta)}\right)^d \geq (1 - \exp(-\beta \delta))^d \geq 1 - d\exp(-\beta \delta)~,
\end{IEEEeqnarray*}
where we used Bernoulli's inequality in the last step. Setting $\delta \equalDef \log(4d)/\beta$, we obtain
\begin{IEEEeqnarray*}{+rCl+x*}
P_f(\calX_\delta) & \geq & 1 - d\exp(-\beta \delta) \geq \frac{3}{4}~.
\end{IEEEeqnarray*}

On the other hand, we have
\begin{IEEEeqnarray*}{+rCl+x*}
\tilde P_f(\calX_\delta) = P(X_I \in \calX_\delta) \leq \sum_{j=1}^n P(X_j \in \calX_\delta) = n \delta^d = n\frac{(\log(4d))^d}{\beta^d} = n \frac{(d\log(4d))^d}{B^d}~.
\end{IEEEeqnarray*}
Hence, if $Bn^{-1/d} \geq 4d\log(4d)$, then $\tilde P_f(\calX_\delta) \leq 4^{-d}$, which implies
\begin{IEEEeqnarray*}{+rCl+x*}
\DTV(P_f, \tilde P_f) & \geq & P_f(\calX_\delta) - \tilde P_f(\calX_\delta) \geq \frac{3}{4} - 4^{-d} \geq \frac{1}{2}~. & \qedhere
\end{IEEEeqnarray*}
\end{proof}

\subsection{Proofs for Variational Formulation} \label{sec:appendix:variational}

The following simple lemma is central to our lower bound for the variational formulation:

\lemLOPT*

\begin{proof}
We have
\begin{IEEEeqnarray*}{+rCl+x*}
\Lopt_g(Q) & = & \sup_{P \in \calP(\calX)} \tr[H\Sigma_P] - \inf_{\tilde P, \tilde Q \in \calP(\calX): \Sigma_{\tilde P} = \Sigma_P, \Sigma_{\tilde Q} = \Sigma_Q} \DKL{\tilde P}{\tilde Q} \\
& = & \sup_{\Sigma \in \calK} \tr[H\Sigma] - \inf_{\tilde P, \tilde Q \in \calP(\calX): \Sigma_{\tilde P} = \Sigma, \Sigma_{\tilde Q} = \Sigma_Q} \DKL{\tilde P}{\tilde Q} \\
& = & \sup_{\Sigma \in \calK} \sup_{\tilde P, \tilde Q \in \calP(\calX): \Sigma_{\tilde P} = \Sigma, \Sigma_{\tilde Q} = \Sigma_Q} \tr[H\Sigma] - \DKL{\tilde P}{\tilde Q} \\
& = & \sup_{\tilde Q \in \calP(\calX): \Sigma_{\tilde Q} = \Sigma_Q} \sup_{\tilde P \in \calP(\calX)} \tr[H\Sigma_P] - \DKL{\tilde P}{\tilde Q} \\
& = & \sup_{\tilde Q \in \calP(\calX): \Sigma_{\tilde Q} = \Sigma_Q} L_g(\tilde Q). & \qedhere
\end{IEEEeqnarray*}
\end{proof}

Before proving the lower bound, we prove a Taylor-based bound on the cosine function, which is then used in the subsequent lemma to bound an integral of the form $\int_{\calX} \exp(\beta\cos(x-z)) \diff x$.

\begin{lemma} \label{lemma:cos_bound}
For all $x \in [-1/4, 1/4]$, we have
\begin{IEEEeqnarray*}{+rCl+x*}
\cos(2\pi x) \leq 1 - \frac{8}{5} \pi^2 x^2~.
\end{IEEEeqnarray*}

\begin{proof}
For $f(x) \equalDef \cos(2\pi x)$, we have
\begin{IEEEeqnarray*}{+rCl+x*}
f'(x) & = & -2\pi \sin(2\pi x) \\
f''(x) & = & -4\pi^2 \cos(2\pi x) \\
f'''(x) & = & 8\pi^3 \sin(2\pi x) \\
f^{(4)}(x) & = & 16\pi^4 \cos(2\pi x)
\end{IEEEeqnarray*}
Applying Taylor's theorem with the Lagrange form of the remainder for $x \in [-1/4, 1/4]$ around $x_0 = 0$, we obtain for some $\xi \in [-1/4, 1/4]$:
\begin{IEEEeqnarray*}{+rCl+x*}
f(x) = 1 - 2\pi^2 x^2 + \frac{2}{3} \pi^4 \cos(2\pi \xi) x^4~.
\end{IEEEeqnarray*}
Therefore, we obtain for $x \in [-1/4, 1/4]$:
\begin{IEEEeqnarray*}{+rCl+x*}
f(x) & \leq & 1 - 2\pi^2 x^2 + \left(\frac{2}{3} \pi^4 (1/4)^2\right) \cdot x^2 \leq 1 - 2\pi^2 (1 - 1/5) x^2 = 1 - \frac{8}{5} \pi^2 x^2~. & \qedhere
\end{IEEEeqnarray*} 
\end{proof}
\end{lemma}

We can now bound the normalizing constant of a rescaled cosine function, which will be used in the lower bound for the variational formulation:

\begin{lemma} \label{lemma:cos_integral}
For any $z \in \bbR^d$, define
\begin{IEEEeqnarray*}{+rCl+x*}
g_z: \calX \to \bbR, x \mapsto \sum_{i=1}^d \cos(2\pi (x_i - z_i))~.
\end{IEEEeqnarray*}
Then, for $\beta > 0$,
\begin{IEEEeqnarray*}{+rCl+x*}
Z_{\beta g_z} \leq \beta^{-d/2} e^{\beta d}~.
\end{IEEEeqnarray*}
\end{lemma}

\begin{proof}
Since $g_z$ is $1$-periodic, $\exp(\beta g_z)$ is $1$-periodic. Moreover, we have $\calX = [0, 1]^d$ and hence
\begin{IEEEeqnarray*}{+rCl+x*}
Z_{\beta g_z} & = & \int_{\calX} \exp(\beta g_z(x)) \diff x = \int_{\calX} \exp(\beta g_0(x)) \diff x \\
& = & \int_{\calX} \exp(\beta \cos(2\pi x_1)) \cdots \exp(\beta \cos(2\pi x_d)) \diff x \\
& = & \left(\int_0^1 \exp(\beta \cos(2\pi x_1)) \diff x_1\right)^d. \IEEEyesnumber \label{eq:cos_integral_1d}
\end{IEEEeqnarray*}
From expanding the inequality $(1 - \sqrt{\beta})^2 \geq 0$, we obtain
\begin{IEEEeqnarray*}{+rCl+x*}
e^\beta \geq 1 + \beta \geq 2\sqrt{\beta}~. \IEEEyesnumber \label{eq:exp_beta}
\end{IEEEeqnarray*}
This allows us to upper-bound the one-dimensional integral in \Eqref{eq:cos_integral_1d} as
\begin{IEEEeqnarray*}{+rCl+x*}
\int_0^1 e^{\beta \cos(2\pi x_1)} \diff x_1  & = & \int_{-1/4}^{3/4} e^{\beta \cos(2\pi x)} \diff x \\
& \leq & \int_{[-1/4, 1/4]} e^{\beta \cos(2\pi x)} \diff x + \int_{[1/4, 3/4]} e^{\beta \cos(2\pi x)} \diff x \\
& \stackrel{\text{\Cref{lemma:cos_bound}}}{\leq} & \int_{[-1/4, 1/4]} e^{\beta (1-\frac{8}{5} \pi^2 x^2)} \diff x + \int_{[1/4, 3/4]} 1 \diff x \\
& \leq & \frac{1}{2} + e^{\beta} \int_{-\infty}^\infty \exp\left(-\frac{x^2}{2 \cdot \frac{5}{16}\pi^{-2}\beta^{-1}}\right) \\
& = & \frac{1}{2} + e^{\beta} \sqrt{2\pi \frac{5}{16} \pi^{-2} \beta^{-1}} \\
& = & \frac{1}{2} + e^{\beta} \sqrt{\frac{5\beta^{-1}}{8\pi}} \\
& \stackrel{\text{\Eqref{eq:exp_beta}}}{\leq} & e^{\beta} \left(\frac{\sqrt{\beta^{-1}}}{4} + \sqrt{\frac{5\beta^{-1}}{8\pi}}\right) \\
& \leq & \beta^{-1/2} e^{\beta}~. & \qedhere
\end{IEEEeqnarray*}
\end{proof}

Finally, we can use some elementary convex geometry to prove our lower bound for the error of the variational formulation.

\thmLOPTLowerBound*

\begin{proof}
\textbf{Step 1: Representability by discrete distributions.}
Let $Q \equalDef \calU(\calX)$. We want to find a discrete distribution $\tilde Q = \sum_{k=1}^M \lambda^{(k)} \delta_{x^{(k)}}$ with $\Sigma_{\tilde Q} = \Sigma_Q$. Using the feature map $\Phi: \calX \to \bbC^{k \times k}, x \mapsto \varphi(x) \varphi(x)^*$, we can write
\begin{IEEEeqnarray*}{+rCl+x*}
\Sigma_{\tilde Q} = \int_{\calX} \varphi(x) \varphi(x)^* \diff \tilde Q(x) = \sum_{k=1}^M \lambda^{(k)} \Phi(x^{(k)})~,
\end{IEEEeqnarray*}
which shows that the matrices $\Sigma_{\tilde Q}$ are exactly those in the convex hull $\conv(\Phi(\calX))$ of $\Phi(\calX)$. Since $\Sigma_Q \in \calK$ by definition of $\calK$, we first want to show that $\calK = \conv(\Phi(\calX))$. As we have just demonstrated, the inclusion $\calK \supseteq \conv(\Phi(\calX))$ is simple. Moreover, because the integral $\Sigma_Q = \int_{\calX} \Phi(x) \diff Q(x)$ is a limit of finite sums, we obtain $\calK \subseteq \ovl{\conv(\Phi(\calX))}$. Since $\Phi$ is continuous and $\calX$ is compact, $\Phi(\calX)$ is also compact. Hence, since we are in finite dimension, $\conv(\Phi(\calX))$ is compact \citep[see e.g.\ Proposition 2.3 in][]{gallier_notes_2008}, which means that $\conv(\Phi(\calX)) \subseteq \calK \subseteq \ovl{\conv(\Phi(\calX))} = \conv(\Phi(\calX))$.

\textbf{Step 2: Bounding the number of discrete points.}
By definition, $\Vlin$ is the $\bbC$-linear span of $\Phi(\calX)$. Using Step 1, we conclude $\calK \subseteq \Vlin$. Hence, $\calK$ is contained in the space $\Vlin$ with $\dim_{\bbR} \Vlin \leq 2n$. By Carathéodory's theorem \citep[see e.g.\ Theorem 2.2 in][]{gallier_notes_2008}, the matrix $\Sigma_Q \in \calK$ is hence representable as a convex combination of $2n+1$ points:
\begin{IEEEeqnarray*}{+rCl+x*}
\Sigma_Q = \sum_{k=1}^{2n+1} \lambda^{(k)} \Phi(x^{(k)}),
\end{IEEEeqnarray*}
with $\lambda^{(k)} \geq 0, \sum_k \lambda^{(k)} = 1$.
By setting $\tilde Q \equalDef \sum_{k=1}^{2n+1} \lambda^{(k)} \delta_{x^{(k)}}$, we obtain $\Sigma_{\tilde Q} = \Sigma_Q$.

\textbf{Step 3: Determining $z$.} Choose an arbitrary index $k^* \in \{1, \hdots, 2n+1\}$ such that $\lambda^{(k^*)} \geq (2n+1)^{-1}$, which always exists. For such an index, we set $z \equalDef x^{(k^*)}$.

\textbf{Step 4: Lower-bounding the approximate log-partition function.} By \Cref{lemma:lopt}, we conclude
\begin{IEEEeqnarray*}{+rCl+x*}
\Lopt_g(\calU([0, 1])) & \geq & L_g(\tilde Q) = \log\left(\sum_{k=1}^{2n+1} \lambda^{(k)} \exp(g(x^{(k)}))\right) \geq \log\left(\lambda^{(k^*)} \exp(g(x^{(k^*)}))\right) \\
& \geq & \log \left((2n+1)^{-1} \exp(\beta f(x^{(k^*)}) - \|g - \beta f\|_\infty)\right) \\
& = & \beta d - \log(2n+1) - \|g - \beta f\|_\infty~.
\end{IEEEeqnarray*}

\textbf{Step 5: Upper-bounding the true log-partition function.} We have
\begin{IEEEeqnarray*}{+rCl+x*}
L_{\beta f}(\calU([0, 1])) & = & \log Z_{\beta f} \stackrel{\text{\Cref{lemma:cos_integral}}}{\leq} \beta d - \log(\beta^{d/2})~.
\end{IEEEeqnarray*}

\textbf{Step 6: Putting it together.} The previous two steps yield the desired bound
\begin{IEEEeqnarray*}{+rCl+x*}
|\Lopt_g(\calU([0, 1])) - L_{\beta f}(\calU([0, 1]))| & \geq & \Lopt_g(\calU([0, 1])) - L_{\beta f}(\calU([0, 1])) \\
& \geq & \log\left(\frac{\beta^{d/2}}{2n+1}\right) - \|g - \beta f\|_\infty~. & \qedhere
\end{IEEEeqnarray*}
\end{proof}

\ifnotinthesis{\vskip 0.2in}
\ifnotinthesis{\bibliography{2022_sampling}}

\end{appendixenv}

\end{document}